\def\comments{1}

\documentclass[letterpaper,11pt]{article}
\usepackage[utf8]{inputenc}
\usepackage{comment}
\usepackage{preamble-jon}
\allowdisplaybreaks

\newcommand{\nope}[1]{}
\usepackage{subcaption}
\usepackage{titling}

\newcommand{\llnorm}[1]{\left\lVert#1\right\rVert_2}

\newcommand{\abs}[1]{\left| #1 \right|}

\renewcommand{\epsilon}{\varepsilon}
\newcommand{\ball}[2]{\mathit{B}_{#2}\left( #1 \right)}

\newtheorem{condition}[thm]{Condition}

\newcommand{\sub}{\subseteq}

\renewcommand{\l}{\left}
\renewcommand{\r}{\right}

\renewcommand{\t}{\tilde}

\newcommand{\tr}{\mathrm{tr}}

\newcommand{\Sym}{\mathsf{Sym}}

\newcommand{\id}{\mathbb{I}}
\newcommand{\Lap}{\mathrm{Lap}}

\newcommand{\mw}{w_{\min}}
\newcommand{\SD}{\ensuremath{\mathrm{d_{TV}}}}
\newcommand{\dH}{\ensuremath{\mathrm{d_{H}}}}

\newcommand{\Pub}{\mathrm{Pub}}
\newcommand{\Priv}{\mathrm{Priv}}

\newcommand{\True}{\textsc{True}}
\newcommand{\False}{\textsc{False}}

\newcommand{\Preconditioner}{\mathrm{PubPreconditioner}}
\newcommand{\SC}{\mathrm{SC}}
\newcommand{\PDPGE}{\mathrm{PubDPGaussianEstimator}}
\newcommand{\LowDimPartitioner}{\mathrm{DPLowDimPartitioner}}
\newcommand{\LowDimPartitionerPublic}{\mathrm{LowDimPartitioner}}
\newcommand{\PCount}{\mathrm{PCount}}
\newcommand{\PrivPCA}{\mathrm{PrivPCA}}
\newcommand{\DPGE}{\mathrm{DPGaussianEstimator}}
\newcommand{\DPHC}{\mathrm{DPHardClustering}}
\newcommand{\DPEC}{\mathrm{EasyClustering}}
\newcommand{\DPHE}{\mathrm{DPHardEstimator}}
\newcommand{\DPEE}{\mathrm{DPEasyEstimator}}
\newcommand{\PrivParams}{\mathrm{PrivParams}}
\newcommand{\QPub}{\mathrm{Q_{Pub}}}
\newcommand{\QPriv}{\mathrm{Q_{Priv}}}

\usepackage[style=alphabetic,backend=bibtex,maxalphanames=10,maxbibnames=20,maxcitenames=10,giveninits=false,doi=false,url=true,backref=true]{biblatex}
\newcommand*{\citet}[1]{\AtNextCite{\AtEachCitekey{\defcounter{maxnames}{2}}}\textcite{#1}}

\newcommand*{\citep}[1]{\citep{#1}}

\usepackage{hyperref}

\title{Private Estimation with Public Data\thanks{Authors are listed in alphabetical order.}
}
\iftrue
\author{Alex Bie\thanks{\texttt{yabie@uwaterloo.ca}. Cheriton School of Computer Science, University of Waterloo. Supported by an NSERC Discovery Grant, a David R. Cheriton Graduate Scholarship, and a Vector Scholarship in Artificial Intelligence.}
    \and
    Gautam Kamath\thanks{\texttt{g@csail.mit.edu}. Cheriton School of Computer Science, University of Waterloo. Supported by an NSERC Discovery Grant, an unrestricted gift from Google, an unrestricted gift from Apple, and a University of Waterloo startup grant.}
    \and
    Vikrant Singhal
    \thanks{\texttt{vikrant.singhal@uwaterloo.ca}. Cheriton School of Computer Science, University of Waterloo. Supported by an NSERC Discovery Grant.}}
\date{}
\fi

\begin{document}
\maketitle
\thispagestyle{empty}



\begin{abstract}
    We initiate the study of differentially private (DP) estimation with access to a small amount of public data.
    For private estimation of $d$-dimensional Gaussians, we assume that the public data comes from a Gaussian that may have vanishing similarity in total variation distance with the underlying Gaussian of the private data. We show that under the constraints of pure or concentrated DP, $d+1$ public data
    samples are sufficient to remove any dependence on the range
    parameters of the private data distribution from the private sample complexity,
    which is known to be otherwise necessary without public
    data. For separated Gaussian mixtures, we assume that the underlying public and private distributions are the same, and we consider two settings: (1) when given a dimension-independent amount of public data, the private
    sample complexity can be improved polynomially in
    terms of the number of mixture components, and any dependence
    on the range parameters of the distribution can be removed
    in the approximate DP case; (2) when given an amount of public data linear in the dimension, the private
    sample complexity can be made independent of
    range parameters even under concentrated DP, and additional improvements can be made to the overall sample complexity.
\end{abstract}

\newpage

\tableofcontents

\newpage

\section{Introduction}
Differential privacy (DP)~\cite{DworkMNS06} guarantees the privacy of \emph{every} point in a dataset.
This is a strong requirement, and often gives rise to qualitatively new requirements when performing private data analysis. 
For instance, for the problem of private mean estimation (under pure or concentrated DP), the analyst must specify a guess for the unknown parameter, and needs more data to get an accurate estimate depending on how good their guess is.
This cost can be prohibitive in cases where the data domain may be unfamiliar. 

Fortunately, in many cases, it is natural to assume that there exists an additional \emph{public} dataset. 
This public dataset may vary in both size and quality. 
For example, one could imagine that a fraction of users opt out of privacy considerations, giving a small set of public in-distribution data.  
Alternatively, it is common to pretrain models on large amounts of public data from the web, which may be orders of magnitude larger than the private data but significantly out of distribution.
In a variety of such settings, this public data can yield dramatic theoretical and empirical improvements to utility in private data analysis (see discussion in Section~\ref{sec:related}).
We seek to answer the following question:

\begin{quote}
    \emph{How can one take advantage of public data for private estimation?}
\end{quote} 

We initiate the study of differentially private statistical estimation with a supplementary public dataset. 
In particular, our goal is to understand when a small amount of public data can significantly reduce the cost of private estimation.  

\subsection{Results}
We investigate private estimation with public data through the canonical problems of learning Gaussians and mixtures of Gaussians.

\paragraph{Gaussians.} In the case of pure and concentrated differential privacy, exactly $d+1$ public points are sufficient to remove any dependence on ``range parameters'' in the private sample complexity, even when the underlying Gaussian distribution of the public data bears \emph{almost no similarity} to that of the private data. We start with following result for the scenario when the two distributions are the same.

\begin{thm}\label{thm:gaussian-informal}
For all $\alpha, \beta, \epsilon > 0$ and $d\geq1$ there exists a computationally efficient $\frac {\eps^2} 2$-zCDP (and a computationally inefficient $(\eps, 0)$-DP) algorithm $\cM$ that takes $n$ private samples and $d+1$ public samples from a $d$-dimensional Gaussian $D$, and is private with respect to the private samples, such that if
    $$n = O\l(\frac {d^2} {\alpha^2} + \frac {d^2} {\alpha\eps}\r) \cdot \polylog\l(d, \frac 1 \beta, \frac 1 \alpha\r),$$
then $\cM$ estimates the parameters of $D$ to within total variation error $\alpha$ with probability at least $1-\beta$.

\end{thm}

Without public data, the analyst would be required to specify a bound $R$ on the $\ell_2$ norm of the mean and a bound $K$ on the condition number of the covariance, and the private sample complexity would scale logarithmically in terms of these parameters (which we denote as \emph{range parameters}). Thus, this cost could be arbitrarily large, or even infinite, given poor \emph{a priori} knowledge of the problem. Note that these requirements are not due to suboptimality of current private algorithms. Packing lower bounds for estimation under pure and concentrated DP show that dependence on range parameters is an inherent cost of privacy. However, this dependence can be entirely eliminated at the price of $d+1$ public data points, which is less public data than the non-private $\Theta(\tfrac {d^2} {\alpha^2})$ sample complexity of the problem. For estimating Gaussians with identity or known covariance (that is, mean estimation) it turns out that just \emph{one} public sample suffices to remove all dependence on range parameters.

The next natural question to ask is: what if our public data does not come from the same distribution as our private data?
Indeed, in practical settings, there may be significant domain shift between data that is publicly available and the private data of interest. 
Our next result shows that we can relax this assumption: it suffices that our public data comes from another Gaussian that is at most $\gamma < 1$ away in total variation distance from the underlying private data distribution.

\begin{thm} \label{thm:robust-gaussian-informal}
For all $\alpha, \beta, \epsilon > 0$, $d\geq1$, and $0 \leq \gamma < 1$, there exists a computationally efficient $\frac {\eps^2} 2$-zCDP (and a computationally inefficient $(\eps, 0)$-DP) algorithm $\cM$ that takes $n$ private samples from a $d$-dimensional Gaussian $D$, and $d+1$ public samples from a $d$-dimensional Gaussian $\wt D$ with $\SD(D, \wt D) \leq \gamma$, and is private with respect to the private samples, such that if
    $$n = O\l(\frac {d^2} {\alpha^2} + \frac {d^2} {\alpha\eps} \r) \cdot \polylog\l(d, \frac 1 \beta, \frac 1 \alpha, \frac 1 {1-\gamma}\r),$$
then $\cM$ estimates the parameters of $D$ to within total variation error $\alpha$ with probability at least $1-\beta$.
\end{thm}

The upper bound on the total variation gap $\gamma$ results in the sample complexity picking up an extra $\polylog(\tfrac 1 {1-\gamma})$ factor. 
For instance, even when the public and private distributions have total variation distance $1 - 1/d^{100}$ (that is, they are \emph{almost entirely dissimilar}), we incur only an extra $\polylog(d)$ factor in the sample complexity.
In particular, note that the public dataset may be simultaneously quite small and significantly out of distribution with respect to the private data.

\paragraph{Mixtures of Gaussians.} 
We focus on estimation of mixtures of $k$ \emph{separated} Gaussians in $d$ dimensions.
Specifically, we assume that, given mixture components $\cN(\mu_i,\Sigma_i)$ with
respective mixing weights $w_i$ for $i \in [k]$ (with $w_i \geq \mw$),
then for each $i \neq j$,
\begin{align}
    \llnorm{\mu_i - \mu_j} \geq
        \wt{\Omega}\left(\sqrt{k} + \frac{1}{\sqrt{w_i}} +
        \frac{1}{\sqrt{w_j}}\right)\cdot
        \max\left\{\sqrt{\|\Sigma_i\|_2},\sqrt{\|\Sigma_j\|_2}\right\}.
        \label{eq:gmm-separation}
\end{align}

Private estimation of such mixtures of Gaussians (without public data) was previously studied by~\cite{KamathSSU19}.
Their algorithms proceed in two stages: first, a clustering step to isolate the samples from each component, and an estimation step, which estimates the parameters of each individual Gaussian using the isolated samples from each component.
Our algorithms follow a similar structure, but employ the public data at appropriate points in the procedure to enable substantial savings.

We provide two flavours of results for learning mixtures of Gaussians.
First, we investigate the setting when we have a very small, dimension-independent amount of public data.

\begin{thm}\label{thm:gmm-hard-approx-informal}
    For all $\alpha,\beta,\eps,\delta>0$, there exists an
    $(\eps,\delta)$-DP algorithm $\cM$ that takes $n$ private samples
    and $m$ public samples from a Gaussian mixture $D$
    satisfying \eqref{eq:gmm-separation}, such that if
    \begin{align*}
        n = O\left(\frac{d^2}{\mw\alpha^2} +
            \frac{d^2}{\mw\alpha\eps}
            +\frac{d^{1.5}k^{2.5}}{\eps}\right)
            \cdot \polylog\left(d,k,\frac{1}{\beta},
            \frac{1}{\alpha},\frac{1}{\eps},
            \frac{1}{\delta}\right) ~~~\text{and}~~~
        m = O\left(\frac{\log(k/\beta)}{\mw}\right),
    \end{align*}
    then $\cM$ estimates the parameters of $D$ to within total variation
    error $\alpha$ with probability at least $1-\beta$.
\end{thm}

\begin{thm}\label{thm:gmm-hard-zcdp-informal}
    For all $\alpha,\beta,\rho>0$, there exists an
    $\rho$-zCDP algorithm $\cM$ that takes $n$ private samples
    and $m$ public samples from a Gaussian mixture $D$
    satisfying \eqref{eq:gmm-separation}, such that if
    for each
    $i \in [k]$, $\|\mu_i\|_2 \leq R$ and $\id \preceq \Sigma_i \preceq K\id$,
    and
    \begin{align*}
        n = O\left(\frac{d^2}{\mw\alpha^2} +
            \frac{d^2}{\mw\alpha\sqrt{\rho}} +
            \frac{d^{1.5}k^{2.5}}{\sqrt{\rho}}\right)
            \cdot\polylog\left(d,k,K,R,\frac{1}{\beta},
            \frac{1}{\alpha},\frac{1}{\rho}\right) ~~~\text{and}~~~
        m = O\left(\frac{\log(k/\beta)}{\mw}\right),
    \end{align*}
    then $\cM$ estimates the parameters of $D$ to within total variation error $\alpha$ with probability at least $1-\beta$.
\end{thm}

Our results differ from the prior work of~\cite{KamathSSU19} in the following ways:
(1) we have results for both approximate DP and zCDP, as opposed
to just approximate DP in \cite{KamathSSU19};
(2) the private sample complexity due to the clustering
step does not involve any dependence on the range parameters;
and (3) the sample complexity due to the clustering step is now
$O(d^{1.5}k^{2.5})$, as opposed to $O(d^{1.5}k^{9.06})$, which
is a significant improvement, and shows up only due to applications
of private PCA under better sensitivity guarantees. As far as
the private sample complexity in the component estimation process
is concerned, in the case of approximate DP, we have expressions
independent of distribution parameters owing to the recent Gaussian
learner by \cite{AshtianiL21}.

Finally, we investigate Gaussian mixture estimation with an amount of public data which is linear in the dimension. 
\begin{thm}\label{thm:gmm-easy-approx-informal}
    For all $\alpha,\beta,\eps,\delta>0$, there exists an
    $(\eps,\delta)$-DP algorithm $\cM$ that takes $n$ private samples
    and $m$ public samples from a Gaussian mixture $D$
    satisfying \eqref{eq:gmm-separation}, such that if
    \begin{align*}
        n = O\left(\frac{d^2}{\mw\alpha^2} +
            \frac{d^2}{\mw\alpha\eps}\right)
            \cdot \polylog\left(d,k,\frac{1}{\beta},
            \frac{1}{\alpha},\frac{1}{\eps},
            \frac{1}{\delta}\right) ~~~\text{and}~~~
        m = O\left(\frac{d\log(k/\beta)}{\mw}\right),
    \end{align*}
    then $\cM$ estimates the parameters of $D$ to within
    total variation error $\alpha$ with probability at least $1-\beta$.
\end{thm}

\begin{thm}\label{thm:gmm-easy-zcdp-informal}
    For all $\alpha,\beta,\rho>0$, there exists an
    $\rho$-zCDP algorithm $\cM$ that takes $n$ private samples
    and $m$ public samples from a Gaussian mixture $D$
    satisfying \eqref{eq:gmm-separation}, such that if
    \begin{align*}
        n = O\left(\frac{d^2}{\mw\alpha^2} +
            \frac{d^2}{\mw\alpha\sqrt{\rho}}\right)
            \cdot\polylog\left(d,k,\frac{1}{\beta},
            \frac{1}{\alpha},\frac{1}{\rho}\right) ~~~\text{and}~~~
        m = O\left(\frac{d\log(k/\beta)}{\mw}\right),
    \end{align*}
    then $\cM$ estimates the parameters of $D$ to within total variation error $\alpha$ with probability at least $1-\beta$.
\end{thm}
Using $\wt{O}(d/\mw)$ public samples establishes all the advantages
that using $\wt{O}(1/\mw)$ public samples gives, but additionally
improves the private sample complexity even further. As we will
describe in the technical overview, the clustering process is done
entirely using the public data, and ends up removing the term due
to private PCA from the private sample complexity in both approximate
DP and zCDP regimes. Additionally, in the zCDP case, we do not have
any dependence on distribution range parameters even in the component
estimation process because we can use results from
Theorem~\ref{thm:gaussian-informal} this time.

\subsection{Technical Overview}

\subsubsection{Estimating Gaussians}
To start, we make the simple observation that for identity covariance Gaussians $\cN(\mu, \id)$, a single public sample is sufficient to remove the \emph{a priori} assumption that $\|\mu\|_2 \leq R$, which is required to obtain a finite sample complexity under pure and concentrated DP. Given a single public sample $\wt X$, Gaussian concentration implies that with probability at least $1- \beta$, the true mean lies within $\wt{O}(\sqrt{d})$ of $\wt X$. Hence, setting $R = \wt{O}(\sqrt{d})$ and starting with our coarse estimate $\wt X$, we can directly apply the $\tfrac{\eps^2}{2}$-zCDP ``clip-and-noise'' Gaussian mechanism to get a sample complexity of $\wt O (\tfrac{d}{\alpha^2} + \tfrac{d}{\eps\alpha})$, or the efficient $\eps$-DP algorithm from \cite{HopkinsKM22} to arrive at a $\wt O(\tfrac {d} {\alpha^2\eps})$ sample complexity with no dependence on $R$ in either case.

We apply the same strategy to estimate the parameters of a general $d$-dimensional Gaussian. First, we compute the sample mean and covariance of $d+1$ public samples. Then we shift our private samples with respect to the public sample mean, and then rescale them according to the public sample covariance. The distribution of the resulting transformed private samples satisfy, with probability at least $1-\beta$ over public samples, $\poly(d, \tfrac 1 \beta)$ range bounds on both the norm of the mean and the condition number of the covariance. This means we can apply either the concentrated DP Gaussian learner from \cite{KamathLSU19} or the pure DP Gaussian learner from \cite{BunKSW19} (followed by inverting the shift and scale) to arrive at a successful learner matching the private sample complexity up to logarithmic factors of these learners that require these \emph{a priori} range bounds, but now without any dependence on them.

Finally, we extend our results to the case where the public data comes from a Gaussian that is at most $\gamma$ far in TV distance from the private data distribution. We translate the $\gamma$-TV bound on the public and the private data distributions to a bound on the differences in their parameters. Taking this into account when applying our shift and rescaling with the public data yields $\poly(d, \tfrac 1 \beta, \tfrac 1 {1-\gamma})$ range bounds on the transformed private data distribution's parameters, which suffices to apply the same private learners as described earlier.

\subsubsection{Estimating Gaussian Mixtures}

 The general process of estimating the parameters of Gaussian
mixtures involves separating the mixture components by using
techniques like PCA, followed by their isolation, and finally
by individual component estimation. In other words, the components
are first projected onto a low-dimensional subspace of $\R^d$,
which separates the large mixture components from the
rest, given that the separation in the original space is enough.
The subspace is computed with respect to the available data,
that is, it is chosen to be the top $k$ subspace of the data
matrix, which preserves the original distances between the
individual components, but shrinks down each component so that
the intra-component distances are small but inter-component
distances are large still. This allows us to separate a group
of components from the rest. Repeating this process on each
group, we can narrow down to a single Gaussian in each group,
and estimate its parameters.

This has been the high-level idea in a lot of prior work, both
non-private \cite{VempalaW02, AchlioptasM05} and private
\cite{KamathSSU19}. In our setting, we utilise the public
data available in different settings at various points to
establish these tasks, while having a low cost in the sample
complexity of the private data. Given a lower bound on the
mixing weights ($\mw$), we consider two settings in terms
of the availability of the public data: (1) when there is
very little available -- $\wt{O}(1/\mw)$; and (2) when there is
a lot more available -- $\wt{O}(d/\mw)$. The key challenge, as
indicated, is accurate clustering of data.

\bigskip

\noindent \textbf{Public Data Sample Complexity: $\wt{O}(1/\mw)$}

\vspace{-11pt}

\paragraph{Clustering.}
This is the more challenging setting of the two because there
is not enough public data for techniques such as PCA to give accurate
results. We therefore have to rely on the private data for
these tasks, and use their private counterparts to get the desired
results.

\begin{enumerate}
    \item Superclustering to reduce sensitivity for Private PCA:
        The main goal, as in most private estimation tasks,
        is to reduce the amount of noise added for privacy in order
        to get better accuracy. This boils down to limiting the
        sensitivity of the target empirical
        function of the data. For this, we would like to make sure
        that the range of the data (to which, we scale the noise) is
        very tight -- in our context, it means that we would like
        to have a ball that contains a group of mixture components,
        which are very close to one another, having close to the
        smallest possible radius. We can limit the data to within this
        ball, and the sensitivity for the next step would be reduced.
        To get such a ball, we use the public data, and come up
        with a ``superclustering'' algorithm,
        which does exactly
        that. This algorithm essentially relies on the concentration
        properties of high-dimensional Gaussians -- it finds the
        radius of the largest Gaussian and a point (as the centre)
        within that Gaussian,
        and additively grows the ball until it stops finding more
        points. This gives a ball whose radius is an $O(k)$ approximation
        to the radius of the largest Gaussian. We use this ball to
        isolate the same components in the private data, and perform
        private PCA to separate the components within. Getting this
        tight ball is crucial to reduce the cost of private PCA in
        terms of the sample complexity compared to the prior work
        of \cite{KamathSSU19}.
    \item Private PCA: We use the well-known private PCA algorithm
        that has been previously
        analysed in \cite{DworkTTZ14, KamathSSU19} for this.
    \item Partitioning in low dimensions: The next task is to partition
        the components in the
        low-dimensional subspace in a way that there is no
        overlap. For this, we use an algorithm that is similar
        to the ``terrific ball'' algorithm of \cite{KamathSSU19}, and
        relies upon the accuracy guarantees of the private PCA step.
        It uses both public and private data, and looks for a terrific
        ball within them. If there is one, indicating that there is
        more than one component in that subset of the datasets, it
        partitions the datasets (and the components), and queues
        the partitions for further work. Otherwise, it just adds that
        subset of the private dataset to the set of discovered
        clusters.
    \item The whole process is repeated on all the remaining
        partitions.
\end{enumerate}

\paragraph{Component Estimation.}
The final task is to estimate the parameters
of the $k$ isolated components.
In the $(\eps,\delta)$-DP setting,
we simply apply the private Gaussian learner of
\cite{AshtianiL21} to each component, which enables us to
learn with no dependence in the sample complexity on the
range parameters of the Gaussian itself (among other available
choices, were the $(\eps,\delta)$-DP learners from
\cite{KamathMSSU21, KothariMV21, TsfadiaCKMS21}, but
this one had the joint best sample complexity with
\cite{TsfadiaCKMS21}). In the $\rho$-zCDP setting,
we use the Gaussian learner from \cite{KamathLSU19}, instead.
Note that in this setting, the private sample complexity
does depend on the range parameters of the Gaussians,
themselves.

\bigskip

\noindent \textbf{Public Data Sample Complexity: $\wt{O}(d/\mw)$}

\vspace{-11pt}

\paragraph{Clustering.}
In the second setting, clustering is easier in terms of
the tasks that need to be done privately -- it can be done
entirely by using the public dataset itself. This time, the
number of public samples is enough to be able to accurately
do PCA on the public data, hence, perform clustering using
just the public data itself. This improves the private
data sample complexity by removing the term due to private
PCA altogether.

\begin{enumerate}
    \item No supercustering involved: The first change in the
        clustering algorithm is eliminating the superclustering step.
    \item PCA: We essentially get the projection matrix for the top $k$
        subspace of the data using non-private PCA, and work within that
        subspace using the public data to further partition the two
        datasets (and the components).
    \item Partitioning in low dimensions: The algorithm for partitioning
        in the low-dimensional subspace
        is similar to the private partitioner used in the previous
        case, except that there is no need for privacy in this case.
    \item The whole process is repeated on all the remaining
        partitions.
\end{enumerate}

\paragraph{Component Estimation.}
Once the components are separated, as before, we individually
estimate them using the partitions of the private dataset privately.
In the $(\eps,\delta)$-DP setting,
we again use the learner from
\cite{AshtianiL21}, which ensures parameter-free
private data sample complexity in this process. In the
$\rho$-zCDP setting though, we use our own
zCDP learner from this text (Theorem~\ref{thm:gaussian-informal}),
which uses the public
data, as well, to ensure parameter-free private data
sample complexity for this process.

\subsection{Related Work}
\label{sec:related}
A large part of our work focuses on using public data to remove the dependence on range parameters in private estimation. 
Understanding these dependences without public data has been a topic of significant study.
\cite{KarwaV18} investigated private estimation of univariate Gaussians, showing that logarithmic dependences on the range parameters were both necessary and sufficient for pure differential privacy, but by using stability-based histograms~\cite{KorolovaKMN09,BunNS16}, they could be removed under approximate differential privacy.
Similar results have been shown in the multivariate setting: while logarithmic dependences are necessary and sufficient under pure or concentrated DP~\cite{KamathLSU19}, they can be removed under approximate DP~\cite{AdenAliAK21,KamathMSSU21,TsfadiaCKMS21,AshtianiL21,KothariMV21,LiuKO21}.
More broadly, a common theme is that the dependences which are intrinsic to pure or concentrated DP can be eliminated by relaxing to approximate DP.
Our results demonstrate that instead of relaxing the privacy definition for all the data, one can achieve the same goal by employing a much smaller amount of public data.

Our investigation fits more broadly into a line of work employing public data for private data analysis.
Works, both theoretical and empirical, investigate the role of public data in private query release, synthetic data generation, and prediction~\cite{JiE13, BeimelNS16, AlonBM19, NandiB20, BassilyCMNUW20, BassilyMN20, LiuVSUW21}. For distribution-free classification with public and private data, it was shown in \cite{AlonBM19} that, roughly speaking, if a hypothesis class is privately learnable with a small amount of public data ($o(1/\alpha)$), it is privately learnable with no public data. For density estimation, we show that the family of unbounded Gaussians is learnable with a small amount of public data ($\alpha$-independent), despite the fact that it is not learnable with only private data.

Additionally, a number of empirical works study the efficacy of public data in private machine learning, including methods such as public pre-training~\cite{AbadiCGMMTZ16, PapernotCSTE19, TramerB21, LuoWAF21,YuZCL21, LiTLH22, YuNBGIKKLMWYZ22}, using public data to compute useful statistics about the private gradients~\cite{ZhouWB21, YuZCL21, KairouzRRT21,AmidGMRSSSTT21}, or using unlabeled public data to train a student model~\cite{PapernotAEGT17, PapernotSMRTE18, BassilyTT18}.

Another setting with mixed privacy guarantees is when different users may require local versus central differential privacy~\cite{AventKZHL17}.
Mean estimation has also been studied in this setting~\cite{AventDK20}.

\section{Preliminaries}

\subsection{Notation}

We define some notation here to be used throughout the
paper.
\begin{itemize}
    \item Public data is denoted by $\wt X = (\wt{X}_1,\dots,\wt{X}_m)$, with each $\wt{X}_i \in \R^d$. Private data is denoted by $X = (X_1,\dots,X_n)$, with each $X_j \in \R^d$. 
    \item We denote a ball of radius $r>0$ centred around a point
        $c \in \R^d$ by $B_{r}(c)$.
    \item For a set $S$, we denote its power set by $\cP(S)$.
\end{itemize}

\subsection{Useful Concentration Inequalities}

We first state a multiplicative Chernoff bound.
\begin{lem}[Multiplicative Chernoff]\label{lem:chernoff-mult}
    Let $X_1,\dots,X_m$ be independent Bernoulli random
    variables, and let $X$ be
    their sum. If $p = \ex{}{X_i}$, then for $0 \leq \delta_1 \leq 1$
    and $\delta_2 \geq 0$,
    $$\pr{}{X \leq (1-\delta_1)pm} \leq e^{-\frac{\delta_1^2 pm}{2}}$$
    and
    $$\pr{}{X \leq (1+\delta_2)pm} \leq
        e^{-\frac{\delta_2^2pm}{2+\delta_2}}.$$
\end{lem}

Next, we state Bernstein's inequality.
\begin{lem}[Bernstein's Inequality]\label{lem:chernoff-add}
    Let $X_1,\dots,X_m$ be independent Bernoulli random variables.
    Let $p = \ex{}{X_i}$.
    Then for $m \geq \frac{5p}{2\epsilon^2}\ln(2/\beta)$ and
    $\eps \leq p/4$,
    $$\pr{}{\abs{\frac{1}{m}\sum{X_i}-p} \geq \epsilon}
        \leq 2e^{-\epsilon^2m/2(p+\epsilon)}
        \leq \beta.$$
\end{lem}

Now, we state the Hanson-Wright inequality
about quadratic forms.
\begin{lem}[Hanson-Wright inequality~\cite{HansonW71}]
  \label{lem:HW}
  Let $X \sim \cN(0,\mathbb{I}_{d \times d})$ and let $A$ be a $d \times d$ matrix.
  Then for all $t > 0$, the following two bounds hold:
  \[
    \pr{}{X^TAX - \tr(A) \geq 2 \|A\|_F \sqrt{t} + 2\|A\|_2t} \leq \exp(-t);
  \]
  \[
    \pr{}{X^TAX - \tr(A) \leq -2 \|A\|_F \sqrt{t} } \leq \exp(-t).
  \]
\end{lem}

We mention an inequality that bounds the tails of
a one-dimensional Gaussian $\cN(\mu,\sigma^2)$.
\begin{lem}[$1$-D Gaussian Concentration]\label{lem:gauss-conc-1d}
    Let $Z \sim \cN(\mu,\sigma^2)$. Then,
    $$\pr{}{|Z-\mu| \leq t\sigma} \leq 2e^{-\frac{t^2}{2}}.$$
\end{lem}

Next, we state a concentration inequality for $0$-mean
Laplace random variables.
\begin{lem}[Laplace Concentration]\label{lem:lap-conc}
    Let $Z \sim \Lap(t)$. Then
    $\pr{}{\abs{Z} > t\cdot\ln(1/\beta)} \leq \beta.$
\end{lem}

We now mention an anti-concentration
inequality for weighted $\chi^2$ distributions from
\cite{ZhangZ18}. We adjust the constants appropriately
in the following theorem as per the specifications in
the aforementioned article.
\begin{lem}[Theorem 6 from \cite{ZhangZ18}]
    \label{lem:chi-squared}
    Let $Z \sim \cN(0,1)$, $a > 0$, and $Y = aZ$. Then
    for $\tau \geq a$, we have the following.
    $$\pr{}{Y \geq a + \tau} \geq 0.06e^{-\frac{3\tau}{2a}}$$
\end{lem}

The following are standard concentration results
for the empirical mean and covariance of a set of
Gaussian vectors (see, e.g.,~\cite{DiakonikolasKKLMS16}).
\begin{lem}\label{lem:gaussian-sum-conc}
    Let $X_1, \dots, X_n$ be i.i.d.\ samples
    from $\cN(0, \mathbb{I}_{d \times d})$.
    Then we have that 
    \[
        \pr{}{\left\|\frac{1}{n} \sum_{i \in [n]} X_i \right\|_2 \geq t}
            \leq 4\exp(c_1 d - c_2 nt^2);
    \]
    \[
        \pr{}{\left\|\frac{1}{n} \sum_{i \in [n]} X_iX_i^T  - I \right\|_2 \geq t}
            \leq 4\exp(c_3 d - c_4 n\min(t,t^2)),
    \]
    where $c_1, c_2, c_3, c_4 >0$ are some absolute constants.
\end{lem}

\subsection{Gaussian Mixtures}

Here, we provide preliminaries for the problem of parameter
estimation of mixtures of Gaussians. Let $\Sym_d^+$ denote
set of all $d \times d$, symmetric, and positive semidefinite
matrices. Let $\cG(d) = \{ \cN(\mu,\Sigma) :
    \mu \in \R^{d}, \Sigma \in \Sym_d^+\}$
be the family of $d$-dimensional Gaussians. We can now define the
class $\cG(d,k)$ of mixtures of Gaussians as follows.
\begin{defn} [Gaussian Mixtures]
    The class of \emph{Gaussian $k$-mixtures in $\R^{d}$} is
    $$\cG(d,k) \coloneqq \left\{\sum\limits_{i=1}^{k}{w_i G_i}:
       G_1,\dots,G_k \in \cG(d), w_1,\dots,w_k > 0,
       \sum_{i=1}^{k} w_i = 1 \right\}.$$
    We can specify a Gaussian mixture
    by a set of $k$ tuples as:
    $\{(\mu_1,\Sigma_1,w_1),\dots,(\mu_k,\Sigma_k,w_k)\},$
    where each tuple represents the mean,
    covariance matrix, and mixing weight
    of one of its components. Additionally,
    for each $i$, we refer to
    $\sigma_i^2 = \llnorm{\Sigma_i}$
    as the maximum directional variance
    of component $i$.
\end{defn}

We assume that the mixing weight of each component is lower
bounded by $\mw$.
We impose a separation condition for mixtures of Gaussians
to be able to learn them.
\begin{defn}[Separated Mixtures]
  \label{def:sep}
    For $s > 0$, a Gaussian mixtures $\cD \in \cG(d,k)$
    is \emph{$s$-separated} if
    $$\forall 1 \leq i < j \leq k,~~~
    \llnorm{\mu_i - \mu_j} \geq
        \left(s + \frac{10}{\sqrt{w_i}} +
        \frac{10}{\sqrt{w_j}}\right)\cdot\max\{\sigma_i,\sigma_j\}.$$
    We denote the family of separated Gaussian mixtures by
    $\cG(d,k,s)$.
\end{defn}

Now, we define what it means to ``learn'' a Gaussian mixture
in our setting.
\begin{defn}[$(\alpha,\beta)$-Learning]
        \label{def:mixture-learning}
    Let $\cD \in \cG(d,k)$ be parameterized by
    $\{(\mu_1,\Sigma_1,w_1),\dots,(\mu_k,\Sigma_k,w_k)\}$.
    We say that an algorithm \emph{$(\alpha,\beta)$-learns}
    $\cD$, if on being given sample-access to $\cD$,
    it outputs with probablity at least $1-\beta$
    a distribution $\wh{\cD} \in \cG(d,k)$
    parameterised by $\{(\wh{\mu}_1,\wh{\Sigma}_1,\wh{w}_1),\dots,
    (\wh{\mu}_k,\wh{\Sigma}_k,\wh{w}_k)\}$,
    such that there exists a permutation
    $\pi:[k] \rightarrow [k]$, for which
    the following conditions hold.
    \begin{enumerate}
        \item For all $1 \leq i \leq k$,
            $\SD(\cN(\mu_i,\Sigma_i),
            \cN(\wh{\mu}_{\pi(i)},\wh{\Sigma}_{\pi(i)}))
            \leq O(\alpha)$.
        \item For all $1 \leq i \leq k$,
            $\abs{w_i - \wh{w}_{\pi(i)}} \leq
            O\left(\tfrac{\alpha}{k}\right)$.
    \end{enumerate}
    Note that the above two conditions together
    imply that $\SD(\cD,\wh{\cD}) \leq \alpha$.
\end{defn}

Finally, we define the ``median radius'' of a Gaussian.
\begin{defn}\label{def:median-radius}
    Let $G \coloneqq \cN(\mu,\Sigma)$ be a $d$-dimensional Gaussian,
    and $R > 0$. Then $R$ is called the median radius of $G$, if
    the $G$-measure of $\ball{\mu}{R}$ is exactly $1/2$.
\end{defn}

The following is an anti-concentration lemma from
\cite{AroraK01} about Gaussians. It lower bounds
the distance of a sample from a Gaussian from any
arbitrary point in space.
\begin{lem}[Lemma~6 from \cite{AroraK01}]
    \label{lem:gaussian-anti-conc}
    Let $G \coloneqq \cN(\mu,\Sigma)$ be a Gaussian
    in $\R^d$ with median radius $R$, and let $\sigma^2 = \|\Sigma\|$.
    Suppose $z \in \R^d$ is an arbitrary point, and $x \sim G$.
    Then for all $t \geq 1$, with probability at least $1 - 2e^{-t}$,
    $$\|x-z\|^2 \geq (\max\{R-t\sigma,0\})^2 + \|z-\mu\|^2 -
        2\sqrt{2t}\sigma\|z-\mu\|.$$
\end{lem}

\subsection{Robustness of PCA to Noise}

Principal Component Analysis (PCA) is one of the main tools
in learning mixtures of Gaussians. It is common to project
the entire data onto the top-$k$ principal directions (a subspace,
which would approximately contain all the means of the $k$
components). Ideally, PCA should eliminate directions that
are not so useful, whilst maintaining the distances among
the Gaussian components. This would allow us to cluster the
data in low dimensions easily based on the most useful directions,
while not having to worry about the rest of the directions
that do not give us useful information for this process.
For the purpose of doing PCA under differential privacy,
we would like to have results for PCA when there is noise
involved in the process of obtaining DP guarantees.

Let $X \in \R^{n \times d}$ be the dataset obtained from the
mixture of Gaussians in question. Suppose $A \in \R^{n \times d}$,
such that for each $i$, $A_i$ is the true mean of the Gaussian
from which $X_i$ has been sampled. Let $n_j$ denote the number
of points in $X$ belonging to component $j$. The following
lemma gives guarantees for PCA when we have a noisy approximation
to the top-$k$ subspace of the dataset.

\begin{lem}[Lemma~3.1 from \cite{KamathSSU19}]\label{lem:PCA_AM_style}
	Let $X\in \R^{n\times d}$ be a collection of
    $n$ datapoints from $k$ clusters each centered
    at $\mu_1, \mu_2,...,\mu_k$. Let $A\in\R^{n\times d}$
    be the corresponding matrix of (unknown) centers
    (for each $j$ we place the center $\mu_{c(j)}$
    with $c(j)$ denoting the clustering point $X_j$
    belongs to). Let $\Pi_{V_k}\in \R^{d\times d}$
    denote the $k$-PCA projection of $X$'s rows. Let
    $\Pi_U\in \R^{d\times d}$ be a projection such
    that for some bound $B\geq 0$ it holds that
    $\|X^T X - (X\Pi_U)^T(X\Pi_U)\|_2 \leq
    \|X^T X - (X\Pi_{V_k})^T(X\Pi_{V_k})\|_2 + B$.
    Denote $\bar{\mu_i}$ as the empirical mean of all
    points in cluster $i$ and denote $\hat{\mu_i}$ as
    the projection of the empirical mean
    $\hat{\mu_i} = \Pi_U \bar{\mu_i}$. Then
	\[ \|\bar{\mu_i} - \hat{\mu_i}\|_2 \leq
        \tfrac 1 {\sqrt {n_i}}\|X-A\|_2+ \sqrt{\tfrac B {n_i}}  \] 
\end{lem}

The next lemma bounds the singular values of a matrix
$X$ that is sampled from a mixture of Gaussians, but centred
around $A$.

\begin{lem}[Lemma~3.2 from \cite{KamathSSU19}]\label{lem:data-spectral}
    Let $X \in \R^{n \times d}$ be a sample
    from $\cD \in \cG(d,k)$,
    and let $A \in \R^{n \times d}$ be the matrix
    where each row $i$ is the (unknown) mean
    of the Gaussian from which $X_i$ was sampled.
    For each $i$, let $\sigma^2_i$ denote the
    maximum directional variance of component
    $i$, and $w_i$ denote its mixing weight.
    Define $\sigma^2 = \max\limits_{i}\{\sigma^2_i\}$
    and $\mw = \min\limits_{i}\{w_i\}$. If
    $$n \geq \frac{1}{\mw} \left(\xi_1 d +
        \xi_2\log\left(\frac{2k}{\beta}\right)\right),$$
    where $\xi_1,\xi_2$ are universal constants,
    then with probability at least $1-\beta$,
    $$\frac{\sqrt{n\mw}\sigma}{4} \leq \llnorm{X-A}
        \leq 4\sqrt{n\sum\limits_{i=1}^{k}{w_i\sigma^2_i}}.$$
\end{lem}

\subsection{Privacy Preliminaries}

We start with different definitions of differential privacy.

\begin{defn}[Differential Privacy (DP) \cite{DworkMNS06}]
    \label{def:dp}
    A randomized algorithm $M:\cX^n \rightarrow \cY$
    satisfies $(\eps,\delta)$-differential privacy
    ($(\eps,\delta)$-DP) if for every pair of
    neighboring datasets $X,X' \in \cX^n$
    (i.e., datasets that differ in exactly one entry),
    $$\forall Y \subseteq \cY~~~
        \pr{}{M(X) \in Y} \leq e^{\eps}\cdot
        \pr{}{M(X') \in Y} + \delta.$$
    When $\delta = 0$, we say that $M$ satisfies
    $\eps$-differential privacy or pure differential
    privacy.
\end{defn}

\begin{defn}[Concentrated Differential Privacy (zCDP)~\cite{BunS16}]
    A randomized algorithm $M: \cX^n \rightarrow \cY$
    satisfies \emph{$\rho$-zCDP} if for
    every pair of neighboring datasets $X, X' \in \cX^n$,
    $$\forall \alpha \in (1,\infty)~~~D_\alpha\left(M(X)||M(X')\right) \leq \rho\alpha,$$
    where $D_\alpha\left(M(X)||M(X')\right)$ is the
    $\alpha$-R\'enyi divergence between $M(X)$ and
    $M(X')$.\footnote{Given two probability distributions
    $P,Q$ over $\Omega$,
    $D_{\alpha}(P\|Q) = \frac{1}{\alpha - 1}
    \log\left( \sum_{x} P(x)^{\alpha} Q(x)^{1-\alpha}\right)$.}
\end{defn}

Note that $(\eps,0)$-DP implies $\frac{\eps^2}{2}$-zCDP, which implies $(\eps \sqrt{\log(1/\delta)}, \delta)$-DP for every $\delta > 0$~\cite{BunS16}. Now, we define the notion of ``private algorithms with public data'', algorithms that take public data samples and private data samples as input, and guarantee differential privacy with respect to the private data.
\begin{defn}[Private Algorithms with Public Data]\label{def:public-dp}
Let $\wt \cX$ be the domain of public data and $\cX$ be the domain of private data. A randomized algorithm $M: \wt{\cX}^m \times \cX^n \to \cY$ taking in public and private data satisfies $(\eps,\delta)$-DP (or $\rho$-zCDP) \emph{with respect to the private data} if for any public dataset $\wt{X} \in \wt{\cX}^m$, the resulting randomized algorithm $M(\wt{X}, \cdot): \cX^n \to \cY$ is $(\eps,\delta)$-DP (or $\rho$-zCDP, respectively).
\end{defn}

These definitions of DP are closed under post-processing, and can be composed with graceful degradation of the privacy parameters.
\begin{lem}[Post-Processing \cite{DworkMNS06,BunS16}]\label{lem:post-processing}
    If $M:\cX^n \rightarrow \cY$ is
    $(\eps,\delta)$-DP, and $P:\cY \rightarrow \cZ$
    is any randomized function, then the algorithm
    $P \circ M$ is $(\eps,\delta)$-DP.
    Similarly if $M$ is $\rho$-zCDP then the algorithm
    $P \circ M$ is $\rho$-zCDP.
\end{lem}

\begin{lem}[Composition of DP~\cite{DworkMNS06, DworkRV10, BunS16}]\label{lem:composition}
    If $M$ is an adaptive composition of differentially
    private algorithms $M_1,\dots,M_T$, then the following
    all hold:
    \begin{enumerate}
        \item If $M_1,\dots,M_T$ are
            $(\eps_1,\delta_1),\dots,(\eps_T,\delta_T)$-DP
            then $M$ is $(\eps,\delta)$-DP for
            $$\eps = \sum_t \eps_t~~~~\textrm{and}~~~~\delta = \sum_t \delta_t.$$
        \item If $M_1,\dots,M_T$ are
            $(\eps_0,\delta_1),\dots,(\eps_0,\delta_T)$-DP
            for some $\eps_0 \leq 1$, then for every $\delta_0 > 0$, $M$
            is $(\eps, \delta)$-DP for
            $$\eps = \eps_0 \sqrt{6 T \log(1/\delta_0)}~~~~
                \textrm{and}~~~~\delta = \delta_0 + \sum_t \delta_t$$
        \item If $M_1,\dots,M_T$ are $\rho_1,\dots,\rho_T$-zCDP
            then $M$ is $\rho$-zCDP for $\rho = \sum_t \rho_t$.
    \end{enumerate}
\end{lem}

\subsubsection{Known Differentially Private Mechanisms}

We state standard results on achieving differential
privacy via noise addition proportional to the
sensitivity~\cite{DworkMNS06}.

\begin{defn}[Sensitivity]
    Let $f : \cX^n \to \R^d$ be a function,
    its \emph{$\ell_1$-sensitivity} and
    \emph{$\ell_2$-sensitivity} are
    $$\Delta_{f,1} = \max_{X \sim X' \in \cX^n} \| f(X) - f(X') \|_1
    ~~~~\textrm{and}~~~~\Delta_{f,2} = \max_{X \sim X' \in \cX^n} \| f(X) - f(X') \|_2,$$
    respectively.
    Here, $X \sim X'$ denotes that $X$ and $X'$
    are neighboring datasets (i.e., those that
    differ in exactly one entry).
\end{defn}

For functions with bounded $\ell_1$-sensitivity,
we can achieve $\eps$-DP by adding noise from
a Laplace distribution proportional to
$\ell_1$-sensitivity. For functions taking values
in $\R^d$ for large $d$ it is more useful to add
noise from a Gaussian distribution proportional
to the $\ell_2$-sensitivity, to get $(\eps,\delta)$-DP
and $\rho$-zCDP.

\begin{lem}[Laplace Mechanism] \label{lem:laplacedp}
    Let $f : \cX^n \to \R^d$ be a function
    with $\ell_1$-sensitivity $\Delta_{f,1}$.
    Then the Laplace mechanism
    $$M(X) = f(X) + \Lap\left(\frac{\Delta_{f,1}}
        {\eps}\right)^{\otimes d}$$
    satisfies $\eps$-DP.
\end{lem}

\begin{lem}[Gaussian Mechanism] \label{lem:gaussiandp}
    Let $f : \cX^n \to \R^d$ be a function
    with $\ell_2$-sensitivity $\Delta_{f,2}$.
    Then the Gaussian mechanism
    $$M(X) = f(X) + \cN\left(0,\left(\frac{\Delta_{f,2}
        \sqrt{2\ln(2/\delta)}}{\eps}\right)^2 \cdot \id_{d \times d}\right)$$
    satisfies $(\eps,\delta)$-DP.
    Similarly, the Gaussian mechanism
    $$M_{f}(X) = f(X) +
        \cN\left(0, \left(\frac{\Delta_{f,2}}{\sqrt{2\rho}}\right)^2 \cdot \id_{d \times d}\right)$$
    satisfies $\rho$-zCDP.
\end{lem}

Next, we mention a very basic pure DP algorithm that computes
the cardinality of a dataset as a simple application of
Lemma~\ref{lem:laplacedp}.
\begin{lem}[$\PCount$]\label{lem:pcount}
    Let $X = (X_1,\dots,X_n)$ be a set of points from some
    data universe $\chi$. Then for all $\eps>0$ and $0<\beta<1$,
    there exists an $\eps$-DP
    mechanism ($\PCount: \chi^* \to \R$) that on input $X$
    outputs $n'$, such that with probability at least $1-\beta$,
    $\abs{n-n'} \leq \tfrac{\ln(1/\beta)}{\eps}$.
\end{lem}
\begin{proof}
    The algorithm is just the following. Step 1: sample
    $\gamma \sim \Lap\left(\tfrac{1}{\eps}\right)$.
    Step 2: output $\abs{X} + \gamma$.
    
    Since the sensitivity of $\abs{X}$ is $1$, by
    Lemma~\ref{lem:laplacedp}, $\PCount$ is $\eps$-DP.
    
    Next, by the guarantees of Lemma~\ref{lem:lap-conc},
    $\abs{\gamma} \leq \tfrac{\ln(1/\beta)}{\eps}$ with
    probability at least $1-\beta$.
\end{proof}

Now, we mention three results from prior work for learning
high-dimensional Gaussians under approximate DP, $z$CDP, and pure DP 
constraints, respectively.
\begin{lem}[$\DPGE$]\label{lem:gaussian-estimator}
    Let $\Sigma \in \R^{d \times d}$ be symmetric and
    positive-definite, and $\mu \in \R^d$.
    For all $0 < \alpha,\beta < 1$, if given $n_{GE}$ independent
    samples from $\cN(\mu,\Sigma)$, then the following algorithms
    output a symmetric and positive-definite
    $\wh{\Sigma} \in \R^{d \times d}$, and a vector $\wh{\mu} \in \R^d$,
    such that with probability at least $1-\beta$,
    $\SD(\cN(\mu,\Sigma),\cN(\wh{\mu},\wh{\Sigma})) \leq \alpha$.
    \begin{enumerate}
        \item For the $(\eps,\delta)$-DP Gaussian estimator from \cite{AshtianiL21},
        \begin{align*}
            n_{GE} = O\left(\frac{d^2 + \log\l(\frac 1 \beta\r)}{\alpha^2} +
        \frac{\l(d^2\sqrt{\log\l(\frac 1 \delta\r)} +
        d\log\l(\frac 1 \delta\r)\r)\cdot\polylog\left(d,\frac{1}{\alpha},\frac{1}{\beta},
        \frac{1}{\eps},
        \log\left(\frac{1}{\delta}\right)\right)}{\alpha\eps}\right).
        \end{align*}
        
        \item For the $\rho$-zCDP Gaussian estimator from \cite{KamathLSU19},
            for $\id \preceq \Sigma \preceq K\id$ and $\|\mu\| \leq R$
            (where $K \geq 1$ and $R>0$),
        \begin{align*}
            n_{GE} &= O\left(\frac{d^2 + \log\l(\frac 1 \beta\r)}{\alpha^2} +
                \frac{d^2\cdot\polylog\left(\frac{d}{\alpha\beta\rho}\right) +
                d\log\left(\frac{d\log(R)}{\alpha\beta\rho}\right)}
                {\alpha\sqrt{\rho}} \right.\\
                &\qquad \qquad \qquad \qquad \qquad \left. + \frac{d^{1.5}\sqrt{\log(K)}\cdot
                \polylog\left(\frac{d\log(K)}{\rho\beta}\right) +
                \sqrt{d\log\left(\frac{Rd}{\beta}\right)}}{\sqrt{\rho}}\right).
        \end{align*}
            
        \item For the (computationally inefficient) $\epsilon$-DP Gaussian estimator from \cite{BunKSW19}, for $\id \preceq \Sigma \preceq K\id$ and $\|\mu\| \leq R$ (where $K \geq 1$ and $R>0$),
        \begin{align*}
            n_{GE} = O\l(\frac {d^2 + \log \l(\frac 1 \beta\r)} {\alpha^2} + \frac {d\log \l(\frac {dR} {\alpha}\r) +d^2\log\l(\frac {dK} \alpha \r) + \log \l(\frac 1 \beta\r)} {\alpha\eps}\r).
        \end{align*}
    \end{enumerate}
\end{lem}

\section{Estimating Gaussians}
    \label{sec:gaussians}

We discuss our results for estimating multivariate Gaussians
with access to a limited amount of public data. We start with the case of estimating Gaussians when we have some public data coming from the same distribution as private data. After that, we demonstrate that this distributional assumption on the private and the public data can be relaxed -- the public data can come from another Gaussian that is at most $\gamma$ away in TV distance from the underlying (Gaussian) distribution of the private data.

\subsection{Same Public and Private Distributions}

\subsubsection{Unknown Mean and Identity Covariance with $1$ Public Sample}\label{subsec:gaussians-mean}

As a warm-up, we consider the case where the covariance is known to be identity, and the only thing left to do is estimating the unknown mean $\mu \in \R^d$. In this setting, we are given a single public sample $\wt X$, along with private samples $X_1,\dots,X_n$, where the $\wt X$ and the $X_j$ are drawn from a $d$-dimensional Gaussian $\cN(\mu, \id)$ independently. We observe that a single public sample is sufficient to get a ``good enough'' coarse estimate of $\mu$, allowing us to apply existing private algorithms for a finer estimate.
This is captured in the following Claim~\ref{clm:1sample}, which is an application of Lemma~\ref{lem:HW}, in the case where $A = \id$ and $t = \log(1/\beta)$.
\begin{clm}\label{clm:1sample}
Suppose we get a sample $\wt{X} \sim \cN(\mu, \id)$. Then with probability at least $1- \beta$, $\|\mu - \wt X\|_2 \leq \sqrt{d + 2\sqrt{d\log(\tfrac 1 \beta)} + 2\log (\tfrac 1 \beta)}$.
\end{clm}

We restate existing pure and concentrated DP algorithms for Gaussian mean estimation below.
\begin{lem}[Known Private Mean Estimators for Identity Covariance Gaussians]
    \label{lem:private-gaussian-mean}
    For all $\alpha,\beta,\eps,\rho>0$, there exist $\eps$-DP and
    $\rho$-zCDP algorithms that take $n>0$ samples from $\cN(\mu,\id)$
    over $\R^d$, where $\|\mu\|_2 \leq R$, and output estimate
    $\wh{\mu} \in \R^d$, such that with probability at least $1-\beta$,
    $\SD(\cN(\mu,\id),\cN(\wh{\mu},\id)) \leq \alpha$, as long as the
    following bounds on $n$ hold.
    \begin{enumerate}
        \item For the $\rho$-zCDP ``clip-and-noise'' algorithm via the Gaussian Mechanism,
            \begin{align*}
            n = \wt O \l(\frac {d + \log\l(\frac 1 \beta\r)}{\alpha^2} + \frac {\sqrt{d}(R+\sqrt{d})\log\l(\frac 1 \beta\r)} {\alpha\sqrt{\rho}} \r).
            \end{align*}
        \item For the (computationally inefficient) $\eps$-DP algorithm from \cite{BunKSW19},
            \begin{align*}
            n = O\l(\frac {d+\log\l(\frac 1 \beta\r)} {\alpha^2} + \frac {d\log\l(\frac {dR} \alpha\r) + \log\l(\frac 1 \beta\r)} {\alpha \eps} \r).
            \end{align*}
        \item For the $\eps$-DP algorithm from \cite{HopkinsKM22},
            \begin{align*}
            n = \wt O\l(\frac{d + \log\l(\frac 1 \beta\r)}{\alpha^2} +
            \frac {d + \log\l(\frac 1 \beta\r)}{\alpha^2\eps} + \frac {d\log R + \min\{d, \log R\}\cdot\log\l(\frac 1 \beta\r)} \eps \r).
            \end{align*}
\end{enumerate}
\end{lem}

In the above, the ``clip-and-noise'' mechanism simply clips all the
points in the dataset to within radius
$\lambda \coloneq R+\sqrt{d + 2\sqrt{d\log(\tfrac n \beta)} + 2\log (\tfrac n \beta)}$
around the origin, and outputs the noisy empirical mean of the
clipped points via the Gaussian mechanism. Note that the $\ell_2$
sensitivity of the empirical mean is now $\tfrac{2\lambda}{n}$, therefore,
the magnitude of the noise vector added to the empirical mean
is at most $\wt{O}\left(\tfrac{\sqrt{d}\lambda}{\sqrt{\rho}n}\right)$
with high probability, which is less than $\alpha$ when
$n \geq \wt{O}\left(\tfrac{\sqrt{d}\lambda}{\sqrt{\rho}\alpha}\right)$.
This is $\rho$-zCDP from Lemma~\ref{lem:gaussiandp}.

We now show how Claim~\ref{clm:1sample}, combined with the results above, yields $1$-public-sample, private mean estimation algorithms.
Notably, our private sample complexity no longer depends on a range bound $R$, allowing us to use a fixed sample size to estimate the family of all identity covariance Gaussians over $\R^d$.

\begin{thm}[Private Gaussian Mean Estimation with Public Data]\label{thm:gaussian-mean}
For all $\alpha,\beta,\eps,\rho>0$, there exist $\eps$-DP and
    $\rho$-zCDP algorithms that take $1$
    public sample $\wt{X}$ and $n>0$ private samples $X=(X_1,\dots,X_n)$ from $\cN(\mu,\id)$ over $\R^d$, are private
    with respect to the private samples $X$, and outputs estimate
    $\wh{\mu} \in \R^d$, such that with probability at least $1-\beta$,
    $\SD(\cN(\mu,\id),\cN(\wh{\mu},\id)) \leq \alpha$, as long as the
    following bounds on $n$ hold.
\begin{enumerate}
    \item For the $\rho$-zCDP ``clip-and-noise'' algorithm via the Gaussian Mechanism and $1$ public sample,
        $$n = \wt O\l(\frac {d+\log\l(\frac{1}{\beta}\r)}{\alpha^2} + \frac{d\log\l(\frac{1}{\beta}\r)}{\alpha\sqrt{\rho}} \r).$$
    \item For the (inefficient) $\eps$-DP algorithm via \cite{BunKSW19} and $1$ public sample,
            $$n = O\l(\frac {d+\log\l(\frac 1 \beta\r)}{\alpha^2} + \frac {d\log\l(\frac {d\log(1/\beta)}{\alpha}\r) + \log\l(\frac 1 \beta\r)} {\alpha\eps}\r).$$
     \item For the (efficient) $\eps$-DP algorithm via \cite{HopkinsKM22} and $1$ public sample,
         $$n = \wt O\l(\frac{d + \log\l(\frac{1}{\beta}\r)}{\alpha^2} +
         \frac {d + \log\l(\frac 1 \beta\r)}{\alpha^2\eps}
         \r).$$
\end{enumerate}
\end{thm}
\begin{proof}
    For concreteness, we look at (1), however an analogous argument suffices for (2) and (3). Say we draw our one public sample $\wt X$, and form a shifted private dataset $Y_1,\dots,Y_n$, each $Y_j \coloneq X_j -\wt X$. Define
    $\lambda \coloneq \sqrt{d + 2\sqrt{d\log(\tfrac 2 \beta)} + 2\log (\tfrac 2 \beta)}$
    and $\mu_Y \coloneq \ex{}{Y_1} = \mu - \wt{X}$.
    By Claim~\ref{clm:1sample}, we have that with probability
    $\geq 1- \beta/2$ over the sampling of $\wt X$, $\| \mu_Y\|_2 \leq \lambda$.
    Hence, we set $R = \lambda$ and target failure probability
    $\beta/2$, and run the zCDP mean estimation algorithm from
    Lemma~\ref{lem:private-gaussian-mean} on $Y_1,\dots,Y_n$.
    Suppose the output of the algorithm is $\wh{\mu}_Y$. Then
    we return $\wh{\mu} = \wh{\mu}_Y + \wt{X}$.

    By union bound, with probability $\geq 1-\beta$, we have that our private data distribution satisfies the range requirements for our private algorithm's guarantees to hold \emph{and} that our private algorithm succeeds. In this case, we have $\|\mu - \wh{\mu}\|_2 = \|(\mu -\wt{X}) - (\wh{\mu} - \wt{X})\|_2 = \|\mu_Y - \wh \mu_Y\|_2$, so $\SD(\cN(\wh\mu, \id), \cN(\mu, \id)) = \SD(\cN(\wh\mu_Y, \id), \cN(\mu_Y, \id)) \leq \alpha$. Plugging in $R = \lambda$ and target failure probability $\beta/2$ into the algorithm's sample complexity in Lemma~\ref{lem:private-gaussian-mean} gives us the desired sample complexity.
    
    Notice that $\wh{\mu}_Y$ is $\rho$-zCDP with respect to $Y_1,\dots,Y_n$, which implies that $\wh{\mu}$ is $\rho$-zCDP \emph{with respect to private data $X_1,\dots,X_n$}. To see why, note that for any fixed $\wt X$, the private algorithm from Lemma~\ref{lem:private-gaussian-mean} is robust to arbitrary change in $Y_j$, and therefore, to any change in $X_j$ because each $X_j$ maps to exactly one $Y_j$. By the post-processing guarantee of zCDP (Lemma~\ref{lem:post-processing}), $\wh{\mu}$ is also private with respect to $X$.
\end{proof}

We remark that the above 1-public-sample private mean estimators can also be used to estimate the mean of a Gaussian with arbitrary known covariance, by reducing to the identity covariance case via rescaling with the known covariance.

\subsubsection{Unknown Mean and Covariance with $d+1$ Public Samples}

Next, we study the case where the covariance is also unknown. In this setting, we are given $d+1$ public samples $\wt X_1,\dots, \wt X_{d+1}$ and $n$ private samples $X_1,\dots,X_n$, where all the $\wt X_i$ and $X_j$ are drawn from an unknown, $d$-dimensional Gaussian $\cN(\mu, \Sigma)$ independently. We extend the same high-level idea from the identity covariance case: we use public data to do coarse estimation, and then use the coarse estimate to transform the private data, reducing to the bounded case that can be solved using the existing private algorithms. Specifically, we use the $d+1$ public samples for a ``public data preconditioning'' step (Algorithm \ref{alg:pub-precond}).

\begin{algorithm}[ht]
\caption{Public Data Preconditioner $\Preconditioner_{\beta}(\wt{X})$}\label{alg:pub-precond}
\KwIn{Public samples $\wt{X} = (\wt X_1,\dots,\wt X_{d+1})$. Failure probability $\beta>0$.}
\KwOut{$\wh{\mu} \in \R^d$, $\wh{\Sigma} \in \R^{d\times d}$, $L \in \R$, $U \in \R$.}
\vspace{5pt}
\tcp{Compute the empirical mean and covariance of $\wt{X}$.}
$$
    \wh \mu \gets \frac 1 {d+1} \sum_{i=1}^{d+1} \wt X_i  ~~~\text{and}~~~ \wh \Sigma \gets \frac 1 d \sum_{i=1}^{d+1} (\wt X_i - \wh \mu) (\wt X_i - \wh \mu)^T
$$

\tcp{Compute $L$ and $U$.}
$$
    L \gets \frac d {4d + 4\sqrt{2d\log \l(\frac 3 \beta\r)} + 2\log \l(\frac 3 \beta\r)}
    ~~~\text{and}~~~ U \gets \frac {9d^2} {\beta^2}
$$

\Return $(\wh{\mu}, \wh{\Sigma}, L, U)$.
\vspace{5pt}
\end{algorithm}

The preconditioning parameters output by Algorithm \ref{alg:pub-precond} are used to recenter, then rescale our private data $X_1,\dots, X_n$. The transformed private data, which we denote by $Y_1,\dots,Y_n$, is then fed as input to an existing $\DPGE$ (from Lemma \ref{lem:gaussiandp}), which outputs estimates $\wh \mu_Y$ and $\wh \Sigma_Y$. We apply the inverse of the preconditioning transform to $\wh \mu_Y$ and $\wh \Sigma_Y$ to obtain our final estimates. This process is summarized in Algorithm \ref{alg:gaussian}.

\begin{algorithm}[ht]
    \caption{Public-Private Gaussian Estimator
    $\PDPGE_{\alpha,\beta,\PrivParams}(\wt{X}, X)$}\label{alg:gaussian}
\KwIn{Public samples $\wt{X} = (\wt X_1,\dots,\wt X_{d+1})$. Private samples $X=(X_1,\dots,X_n$). Error tolerance $\alpha>0$, failure probability $\beta>0$. Privacy parameters $\PrivParams \sub \R$.}
\KwOut{$\wh\mu_X \in \R^d$, $\wh \Sigma_X \in \R^{d\times d}$}
\vspace{5pt}
\tcp{Precondition the private data using the public dataset.}
$(\wh\mu, \wh\Sigma, L, U) \gets \Preconditioner_{\frac{\beta}{2}}(\wt{X})$.\\
\For{$j \in [n]$}{
    Set $Y_j \gets \frac {1}{\sqrt{L}}\wh{\Sigma}^{-1/2}(X_j - \wh{\mu})$.
}
\vspace{5pt}
\tcp{Set range parameters for private applications.}
$$
    R = \sqrt{\frac U L} \cdot \sqrt{\log\l(\frac 6 \beta\r)} ~~~\text{and}~~~ K = \frac U L
$$

\tcp{Estimate Gaussian using private data and private algorithm.}
Set $Y \gets (Y_1,\dots,Y_n)$ and $(\wh{\mu}_Y,\wh{\Sigma}_Y) \gets \DPGE_{\alpha,\frac{\beta}{2},\PrivParams,R,K}(Y)$.\\
Set
$$
    \wh{\mu}_X \gets \sqrt L \wh{\Sigma}^{1/2} \wh{\mu}_Y + \wh{\mu} ~~~\text{and}~~~ \wh{\Sigma}_X \gets L \wh{\Sigma}^{1/2} \wh{\Sigma}_Y \wh{\Sigma}^{1/2} 
$$

\Return $\wh{\mu}_X, \wh{\Sigma}_X$.
\vspace{5pt}
\end{algorithm}

The key step needed to establish the correctness of this algorithm is ensuring that, with high probability over the sampling of public data, the parameters of the true distribution underlying $Y_1,\dots, Y_n$ indeed satisfy tight range bounds that could help $\DPGE$ provide the desired success guarantee with parameter-free, private sample complexity. In other words, the mean of the transformed Gaussian lies in a known $\poly(d,\tfrac 1 \beta)$ ball, and the condition number of its covariance is in $\poly(d, \tfrac 1 \beta)$.

\begin{lem}[Public Data Preconditioning]\label{lem:pub-precond-guarantee}
    For all $\beta>0$, there exists an algorithm that takes $d+1$ independent samples $\wt{X}=(\wt{X}_1,\dots,\wt{X}_{d+1})$ from a Gaussian $\cN(\mu,\Sigma)$ over $\R^d$, and outputs $\wh{\mu} \in \R^d$, $\wh{\Sigma} \in \R^{d \times d}$, $L \in \R$, and $U \in \R$, such that letting
    $\Sigma_Y \coloneq \tfrac{1}{L} \wh \Sigma^{-1/2}\Sigma\wh\Sigma^{-1/2}$ and $\mu_Y \coloneq \tfrac{1}{\sqrt{L}} \wh \Sigma^{-1/2} (\mu - \wh \mu)$, then with probability at least $1-\beta$ over the sampling of the data,
    \begin{enumerate}
        \item $\id \preceq \Sigma_Y \preceq \tfrac U L \id$
        \item $\|\mu_Y\| \leq \sqrt {\tfrac U L} \cdot \sqrt{5\log\l(\tfrac 3 \beta\r)}$
    \end{enumerate}
    where $\tfrac U L = O\l(\tfrac{d^2\log \l(\tfrac 1 \beta\r)}{\beta^2}\r)$.
\end{lem}

The Lemma follows from results on the well-studied problem of bounding the singular values of a random Gaussian matrix. We employ the following facts stated in \cite{SankarST06}\footnote{\cite{SankarST06} attributes the bound on the smallest eigenvalue to Edelman \cite{Edelman88}, and the bound on the largest to Davidson and Szarek \cite{DavidsonS01}}.
\begin{fact}[Singular Values of Gaussian Matrices
    \cite{SankarST06}]\label{fact:gaussian-condition}
    Let $Z\in \R^{d\times d}$ be a matrix with each $Z_{ij} \sim N(0,1)$ independently. Denote by $\sigma_d(Z)$ the smallest singular value of $Z$, and by $\sigma_1(Z)$ its largest singular value. Then we have the following.
    \begin{enumerate}
        \item $\pr{}{|\sigma_d(Z)| \leq \tfrac \beta {\sqrt d}} \leq \beta$
        \item $\pr{}{|\sigma_1(Z)| \geq 2\sqrt d + \sqrt {2 \log\l(\tfrac 1 \beta\r)}} \leq \beta$
    \end{enumerate}
\end{fact}
We also use another fact about the distributions of the empirical mean and covariance of standard normal random variables.
\begin{fact}[Properties of Wishart Distribution\footnote{See Theorem 6 from \url{https://www.stat.pitt.edu/sungkyu/course/2221Fall13/lec2.pdf}.}]\label{fact:cochran}
    Let $Z_1,\dots,Z_{m+1}$, $W_1,\dots,W_m$ be chosen i.i.d. from $\cN(0,\id)$ over $\R^d$. Suppose $\wh{\mu} = \tfrac{1}{m+1}\sum\limits_{i=1}^{m+1}{Z_i}$. Then
    $$\frac{1}{m}\sum\limits_{i=1}^{m+1}{(Z_i-\wh{\mu})(Z_i-\wh{\mu})^T}
        \sim \frac{1}{m}\sum\limits_{i=1}^{m}{W_i W_i^T}
        \sim \frac{1}{m}\cW_{d}(m,\id).$$
    In other words, the two quantities are identically distributed according to the scaled, $d$-dimensional Wishart distribution with $m$ degrees of freedom.
\end{fact}

\begin{proof}[Proof of Lemma \ref{lem:pub-precond-guarantee}]
We prove the theorem by proving the utility of Algorithm~\ref{alg:pub-precond}. Note that the quantities $L,U$, as defined in Algorithm~\ref{alg:pub-precond}, satisfy $\tfrac U L = O\l(\tfrac {d^2\log ( \tfrac 1 \beta)}  {\beta^2}\r)$.

To prove part (1), where we wish to bound the eigenvalues of $\Sigma_Y$, we note using the properties of Loewner ordering, and the symmetric nature of $\wh \Sigma, \Sigma$ that
\begin{align*}
    \id \preceq \frac 1 L \widehat\Sigma^{-1/2}\Sigma \widehat\Sigma^{-1/2} \preceq \frac U L \id &\iff L\widehat\Sigma \preceq  \Sigma \preceq U\widehat\Sigma \\
        &\iff L \Sigma^{-1/2}\wh\Sigma\Sigma^{-1/2} \preceq \id \preceq U\Sigma^{-1/2}\wh\Sigma\Sigma^{-1/2}.
\end{align*}
Therefore, it is sufficient to prove the final inequality. Let $\wh{\Sigma}_Z \coloneq \Sigma^{-1/2}\wh{\Sigma}\Sigma^{-1/2}$. Then
$$\wh{\Sigma}_Z = \frac 1 {d} \sum_{i=1}^{d+1} (\Sigma^{-1/2} (\wt{X}_i - \wh\mu)) (\Sigma^{-1/2}(\wt{X}_i - \wh\mu))^T = \frac 1 {d} \sum_{i=1}^{d+1} (Z_i - \wh\mu_Z)(Z_i - \wh\mu_Z)^T$$
where for each $i \in [d+1]$, $Z_i \coloneq \Sigma^{-1/2}(X_i - \mu) \sim N(0, \id)$ independently and $\wh{\mu}_Z \coloneq \Sigma^{-1/2}(\wh{\mu}-\mu) = \tfrac 1 {d+1} \sum_{i=1}^{d+1} Z_i$. Therefore, it suffices to show $L\wh\Sigma_Z \preceq \id$ and $\id \preceq U \wh\Sigma_Z$. For $i \in [d]$, let $W_i \sim \cN(0,\id)$ independently. Then Fact~\ref{fact:cochran} says that $\wh{\Sigma}_Z$ is \emph{identically distributed} to $\frac 1 d\sum_{i=1}^d W_i W_i^T$. From here, we can apply the bounds from Fact~\ref{fact:gaussian-condition} by noting that $\lambda_d(\wh{\Sigma}_Z) \sim \tfrac{1}{d}\sigma_d\l(\begin{bmatrix} W_1,\dots,W_d \end{bmatrix}\r)^2$ and $\lambda_1(\wh{\Sigma}_Z) \sim \tfrac{1}{d}\sigma_1\l(\begin{bmatrix} W_1,\dots,W_d \end{bmatrix}\r)^2$.
\begin{enumerate}
    \item With probability $\geq 1- \tfrac \beta 3$, $\tfrac{1}{d} \sigma_d\l(\begin{bmatrix} W_1,\dots,W_d \end{bmatrix}\r)^2 > \tfrac {(\beta/3)^2} {d^2} \implies \lambda_d(\wh \Sigma_Z) > \tfrac {(\beta/3)^2} {d^2} \implies U \wh\Sigma_Z \succeq \id$.
    \item With probability $\geq 1- \tfrac \beta 3$, $\sigma_1(\wh\Sigma_Z) = \tfrac 1 d \sigma_1\l(\begin{bmatrix} W_1,\dots,W_d \end{bmatrix}\r)^2 < \tfrac{\l(2\sqrt d + \sqrt{2\log (3 / \beta)}\r)^2}{d} \implies \lambda_1(\wh\Sigma_Z) < \tfrac{\l(2\sqrt d + \sqrt{2\log (3 / \beta)}\r)^2}{d} \implies L\widehat\Sigma_Z \preceq \id$.
\end{enumerate}
Taking the union bound, (1) holds with probability at least $1-\tfrac{2\beta}{3}$.

Next, we prove the bound on $\mu_Y$ in (2). We have
$$
    \mu_Y = \frac 1 {\sqrt L} \widehat \Sigma^{-1/2}(\mu - \wh \mu) = \frac 1 {\sqrt L} \widehat \Sigma^{-1/2}\l(\mu - \frac 1 {d+1} \sum_{i=1}^{d+1}(\Sigma^{1/2}Z_i + \mu)\r) = -\frac 1 {\sqrt L}\widehat\Sigma^{-1/2}\Sigma^{1/2}\wh\mu_Z.
$$
Since $\wh \mu_Z$ is identically distributed to $\frac 1 {\sqrt{d+1}} Z_1$, applying Lemma \ref{lem:HW} with $t = \log(3/\beta)$ and $A = \id$ implies that with probability $\geq 1- \beta/3$, 
\begin{align*}
    \|\wh\mu_Z\|_2 \leq \sqrt{\frac {d + 2\sqrt{d\log(3/\beta)} + 2\log(3/\beta)} {d+1}}
\end{align*}
From (1), we know that $\left\|\tfrac{1}{L}\wh{\Sigma}^{-1/2}\Sigma\wh{\Sigma}^{-1/2}\right\|_2 \leq \tfrac{U}{L}$. Hence, $\left\|-\frac 1 {\sqrt L}\wh\Sigma^{-1/2}\Sigma^{1/2}\right\|_2 \leq \sqrt {U/L}$. This implies that
\begin{align*}
    \|\mu_Y\|_2 \leq \l\|-\frac 1 {\sqrt L} \widehat\Sigma^{-1/2}\Sigma^{1/2} \r\|_2\cdot \l\|\wh \mu_Z\r\|_2 \leq \sqrt {\frac U L} \cdot \sqrt{\frac {d + 2\sqrt{d\log(3/\beta)} + 2\log(3/\beta)} {d+1}}
\end{align*}
which implies (2). Applying the union bound again, we have the claim.
\end{proof}

Algorithm \ref{alg:gaussian}, indeed, is a $(d+1)$-public-sample, private algorithm satisfying privacy with respect to private data $X_1,\dots,X_n$, and the required accuracy guarantees. This follows from the guarantees of public data preconditioning (as stated in the aforementioned Lemma \ref{lem:pub-precond-guarantee}) combined with the guarantees of existing $\DPGE$ (stated in Lemma~\ref{lem:gaussian-estimator}). Again, we note that our sample complexity no longer depends on the \emph{a priori} bounds on the mean and the covariance of the unknown private data distribution.

\begin{thm}[Private Gaussian Estimation with Public Data]\label{thm:gaussian}
    For all $\alpha, \beta, \eps, \rho > 0$, there exist $\eps$-DP and $\rho$-zCDP algorithms that take $d+1$ public samples $\wt{X} = (\wt{X}_1,\dots,\wt{X}_{d+1})$ and $n$ private samples $X = (X_1,\dots,X_n)$ from an unknown, $d$-dimensional Gaussian $\cN(\mu,\Sigma)$, and are private with respect to the private samples, and return $\wh{\mu}_X \in \R^d$ and $\wh{\Sigma}_X \in \R^{d \times d}$, such that $\SD(\cN(\wh \mu_X, \wh \Sigma_X), \cN(\mu, \Sigma)) \leq \alpha$ with probability at least $1-\beta$, as long as the following bounds on $n$ hold.
    \begin{enumerate}
        \item For $\rho$-zCDP (with computational efficiency) via \cite{KamathLSU19} and $d+1$ public samples,
        \begin{align*}
                n = O\left(\frac{d^2 + \log\l(\frac 1 \beta\r)}{\alpha^2} +
                    \frac{d^2\cdot\polylog\left(\frac{d}{\alpha\beta\rho}\right)}{\alpha\sqrt{\rho}} \right).
        \end{align*}
        \item For $\eps$-DP (without computational efficiency) via \cite{BunKSW19} and $d+1$ public samples,
        \begin{align*}
            n = O\l(\frac {d^2 + \log \l(\frac 1 \beta\r)} {\alpha^2} + \frac {d^2\log \l(\frac d {\alpha\beta}\r)} {\alpha\eps}\r).
        \end{align*}
    \end{enumerate}
\end{thm}
\begin{proof}
    We prove the privacy and the utility guarantees for Algorithm~\ref{alg:gaussian}. In this proof, we will work in the $\rho$-zCDP regime (where, $\PrivParams = \{\rho\}$) using the Gaussian estimator of \cite{KamathLSU19} (the second algorithm in Lemma~\ref{lem:gaussian-estimator}), and the analogous guarantees in the $\eps$-DP case (where $\PrivParams = \{\eps\}$) would follow by the same argument (but by using the third algorithm from Lemma~\ref{lem:gaussian-estimator}, instead).
    
    We start by proving the utility guarantees first. Applying our public data preconditioner (Algorithm~\ref{alg:pub-precond} with public data $\wt{X}$ as input) on $X$ with target failure probability $\beta/2$ yields $Y_1,\dots,Y_n$, such that for each $j \in [n]$, $Y_j \sim \cN(\mu_Y,\Sigma_Y)$, where $\mu_Y$ and $\Sigma_Y$ are quantites as defined in Lemma~\ref{lem:pub-precond-guarantee}. Then with probability $\geq 1-\tfrac \beta 2$ over the sampling of public data, we have $\id \preceq \Sigma_Y \preceq \tfrac U L \id$  and $\|\mu_Y\|_2 \leq \sqrt{\tfrac U L}\cdot\sqrt{5\log(\tfrac 6 \beta)}$ (Lemma~\ref{lem:pub-precond-guarantee}). Hence, we can set $K = \tfrac U L = O(d^2\log(\tfrac 1 \beta)/\beta^2) $, and $R = \sqrt{\tfrac U L}\cdot\sqrt{5\log(\tfrac 6 \beta)} = O(d\log(\tfrac 1 \beta)/\beta)$, and run the $\rho$-zCDP Gaussian estimator on $Y_1,\dots,Y_n$ with target failure probability $\beta/2$. We obtain our private sample complexity, which is now independent of the range parameters of the underlying distribution, by plugging in these values into the private Gaussian estimator's sample complexity.

    Under these parameter settings and sample complexity, by Lemma~\ref{lem:gaussian-estimator} and the union bound, we have that with probability $\geq 1- \beta$, the algorithm succeeds in outputting $\wh \mu_Y$ and $\wh \Sigma_Y$, such that $\SD(\cN(\wh{\mu}_Y,\wh{\Sigma}_Y),\cN(\mu_Y,\Sigma_Y)) \leq \alpha$. We output the estimates $\wh \Sigma_X \coloneq L \wh \Sigma^{1/2} \wh \Sigma_Y \wh\Sigma^{1/2}$ and $\wh \mu_X \coloneq \sqrt L \wh \Sigma^{1/2} \wh\mu_Y + \wh \mu$. Denoting $A \coloneq \frac 1 L \wh \Sigma^{-1}$, by the properties of the Mahalanobis norm, $\|\cdot \|_\Sigma$:
    \begin{align*}
        \|\wh \Sigma_X - \Sigma\|_{\Sigma} &=  \|A^{1/2}\wh \Sigma_X A^{1/2} -     A^{1/2}\Sigma A^{1/2}\|_{A^{1/2}\Sigma A^{1/2}} = \| \wh \Sigma_Y - \Sigma_Y\|_{\Sigma_Y} \\
        \|\wh \mu_X - \mu\|_\Sigma &= \| A^{1/2}\wh \mu_X - A^{1/2} \mu \|_{A^{1/2}\Sigma A^{1/2}} = \| (\wh\mu_Y + A^{1/2}\wh\mu) - (\mu_Y + A^{1/2}\wh\mu)\|_{\Sigma_Y} = \|\wh \mu_Y - \mu_Y\|_{\Sigma_Y}
    \end{align*}
    which implies that $\SD(\cN(\wh \mu_X, \wh \Sigma_X), \cN(\mu, \Sigma)) = \SD (\cN(\wh \mu_Y, \wh \Sigma_Y), \cN(\mu_Y, \Sigma_Y)) \leq \alpha$.

    To argue about privacy, note that releasing $\wh \mu_X$ and $\wh \Sigma_X$ is $\rho$-zCDP with respect to $Y$, since it is post-processing (Lemma~\ref{lem:post-processing}) of the output ($\wh{\mu}_Y,\wh{\Sigma}_Y$) of a $\rho$-zCDP algorithm using public information. To argue about the $\rho$-zCDP of our algorithm \emph{with respect to the private dataset} $X$, note that for any fixed public dataset $\wt{X}$, the application of $\DPGE$ from Lemma~\ref{lem:gaussian-estimator} is robust to changing one sample $Y_j$ arbitrarily. Since each $X_j$ maps to exactly one $Y_j$, $\DPGE$ is robust against changing any $X_j$ arbitrarily, as well. This gives us the final privacy guarantee with respect to $X$.
\end{proof}

\subsection{Different Public Data and Private Data Distributions}

In this section, we give results for the scenario, where the public data does not come from the same distribution as that of the private data. Specifically, we show that in the case where our public data comes from another Gaussian within TV distance $\gamma$ of our private data distribution, a slight modification of Algorithm \ref{alg:pub-precond} (public data preconditioning) will also work to reduce our unbounded private estimation problem to a bounded one.

\subsubsection{Technical Lemmata}

We start by stating results about different distance metrics
for distributions. Let $\dH(\cdot,\cdot)$ denote the Hellinger
distance between two distributions, and let $\SD(\cdot,\cdot)$
denote their total variation (TV) distance. We start with a known
fact about the relation between $\dH(\cdot,\cdot)^2$ and
$\SD(\cdot,\cdot)$.
\begin{lem}[Hellinger Distance vs. TV Distance]
    \label{lem:hellinger-tv}
    Let $P,Q$ be distributions over $R^d$. Then
    $\dH(P,Q)^2 \leq \SD(P,Q)$.
\end{lem}

Next, we state the expression for Hellinger distance between
two univariate Gaussians.
\begin{lem}[Hellinger Distance for Gaussians]
    \label{lem:gaussians-hellinger}
    Let $G_1\equiv\cN(\mu_1,\sigma_1^2),G_2\equiv\cN(\mu_2,\sigma_2^2)$
    be Gaussians over $\R$. Then
    $$\dH(G_1,G_2)^2 = 1 -
        \sqrt{\frac{2\sigma_1\sigma_2}{\sigma_1^2+\sigma_2^2}}
        \cdot e^{-\frac{(\mu_1-\mu_2)^2}{4(\sigma_1^2+\sigma_2^2)}}.$$
\end{lem}

Now, we lower-bound the Hellinger distance between two univariate
Gaussians.
\begin{lem}[Hellinger Distance Lower Bounds]
    \label{lem:gaussians-hellinger-lb}
    Let $G_1\equiv\cN(\mu_1,\sigma_1^2),G_2\equiv\cN(\mu_2,\sigma_2^2)$
    be Gaussians over $\R$. Suppose $\sigma_{\max} = \max\{\sigma_1,\sigma_2\}$.
    Then $\dH(G_1,G_2)^2 \geq \dH(\cN(\mu_1,\sigma_{\max}^2),\cN(\mu_2,\sigma_{\max}^2))^2$
    and $\dH(G_1,G_2)^2 \geq \dH(\cN(0,\sigma_1^2),\cN(0,\sigma_2^2))^2$.
\end{lem}
\begin{proof}
    Using Lemma~\ref{lem:gaussians-hellinger}, we have the following.
    \begin{align*}
        \dH(G_1,G_2)^2 &= 1 -
                \sqrt{\frac{2\sigma_1\sigma_2}{\sigma_1^2+\sigma_2^2}}
                \cdot e^{-\frac{(\mu_1-\mu_2)^2}{4(\sigma_1^2+\sigma_2^2)}}\\
            &\geq 1 - e^{-\frac{(\mu_1-\mu_2)^2}{4(\sigma_1^2+\sigma_2^2)}}\\
            &\geq 1 - e^{-\frac{(\mu_1-\mu_2)^2}{8\sigma_{\max}^2}}\\
            &= \dH(\cN(\mu_1,\sigma_{\max}^2),\cN(\mu_2,\sigma_{\max}^2))^2
    \end{align*}
    In the above, the second line follows from AM-GM inequality,
    that is, $\tfrac{\sigma_1^2+\sigma_2^2}{2} \geq \sigma_1\sigma_2$.
    This proves the first part. Now, we prove the second part.
    \begin{align*}
        \dH(G_1,G_2)^2 &= 1 -
                \sqrt{\frac{2\sigma_1\sigma_2}{\sigma_1^2+\sigma_2^2}}
                \cdot e^{-\frac{(\mu_1-\mu_2)^2}{4(\sigma_1^2+\sigma_2^2)}}\\
            &\geq 1 -
                \sqrt{\frac{2\sigma_1\sigma_2}{\sigma_1^2+\sigma_2^2}}\\
            &= \dH(\cN(0,\sigma_1^2),\cN(0,\sigma_2^2))^2
    \end{align*}
    In the above, the second line follows from the fact that
    $e^{-x} \leq 1$ for $x \geq 0$. This completes our proof.
\end{proof}

The next lemma describes the relation between the parameters of two
univariate Gaussians when their TV distance is upper-bounded.
\begin{lem}[TV Distance and Gaussian Parameters]
    \label{lem:gaussians-tv-parameters}
    Let $G_1\equiv\cN(\mu_1,\sigma_1^2),G_2\equiv\cN(\mu_2,\sigma_2^2)$
    be Gaussians over $\R$. Suppose
    $\sigma_{\max} = \max\{\sigma_1,\sigma_2\}$ and
    $\sigma_{\min} = \min\{\sigma_1,\sigma_2\}$. If
    $\SD(G_1,G_2) \leq \alpha$, then
    $$\frac{(\mu_2-\mu_1)^2}{\sigma_{\max}^2} \leq \frac{8\alpha}{1-\alpha}
        ~~~\text{and}~~~
        \frac{\sigma_{\max}}{\sigma_{\min}} \leq \frac{2}{(1-\alpha)^2}.$$
\end{lem}
\begin{proof}
    We have the following from Lemmata~\ref{lem:hellinger-tv}
    and~\ref{lem:gaussians-hellinger-lb}.
    \begin{align*}
        \alpha &\geq \SD(G_1,G_2)\\
            &\geq \dH(\cN(\mu_1,\sigma_{\max}^2),\cN(\mu_2,\sigma_{\max}^2))^2\\
            &= 1 - e^{-\frac{(\mu_1-\mu_2)^2}{8\sigma_{\max}^2}}\\
        \iff e^{-\frac{(\mu_1-\mu_2)^2}{8\sigma_{\max}^2}} &\geq 1-\alpha\\
        \iff \frac{(\mu_2-\mu_1)^2}{\sigma_{\max}^2} &\leq
                8\ln\left(\frac{1}{1-\alpha}\right)\\
            &\leq \frac{8\alpha}{1-\alpha}
    \end{align*}
    The last line holds because
    $\tfrac{1}{1-\alpha} = 1 + \tfrac{\alpha}{1-\alpha}$, and
    $x \geq \ln(1+x)$ for all $x \in \R$.
    
    For the second part, we apply Lemmata~\ref{lem:hellinger-tv}
    and~\ref{lem:gaussians-hellinger-lb} again to get the following.
    \begin{align*}
        \alpha &\geq \SD(G_1,G_2)\\
            &\geq \dH(\cN(0,\sigma_1^2),\cN(0,\sigma_2^2))^2\\
            &= 1 - \sqrt{\frac{2\sigma_1\sigma_2}{\sigma_1^2 + \sigma_2^2}}\\
        \iff \frac{1}{(1-\alpha)^2} &\geq
                    \frac{\sigma_1^2+\sigma_2^2}{2\sigma_1\sigma_2}\\
        \iff \frac{2}{(1-\alpha)^2} &\geq \frac{\sigma_1}{\sigma_2}
                + \frac{\sigma_2}{\sigma_1}\\
            &= \frac{\sigma_{\max}}{\sigma_{\min}} +
                \frac{\sigma_{\min}}{\sigma_{\max}} \\
        \implies \frac {2} {(1-\alpha)^2} &\geq \frac {\sigma_{\max}} {\sigma_{\min}}
    \end{align*}
\end{proof}

Finally, we state a multivariate analogue of Lemma \ref{lem:gaussians-tv-parameters}.
\begin{lem}\label{lem:d-gaussians-tv-parameters}
Suppose $\SD(\cN(\mu, \Sigma), \cN(\wt \mu,\wt \Sigma)) \leq \gamma$ for some $\gamma < 1$. Then
\begin{enumerate}
    \item $\frac {(1-\gamma)^4}{4} \wt \Sigma\preceq \Sigma \preceq \frac{4}{(1-\gamma)^4} \wt \Sigma$
    \item $(\mu -\wt \mu)(\mu -\wt \mu)^T \preceq \frac {8\gamma} {1-\gamma} (\Sigma + \wt \Sigma) $
\end{enumerate}
\end{lem}
\begin{proof}
    Let $G_1 \equiv \cN(\mu,\Sigma)$ and $G_2 \equiv \cN(\wt{\mu},\wt{\Sigma})$. For any unit vector $v \in \R^d$, denote by $v^T G_i$ the distribution over $\R$ obtained by sampling $x \sim G_i$ and outputting $v^T x$. A data-processing inequality for the TV distance gives us that, for any unit vector $v \in \R^d$,
    \begin{align*}
        \SD (\cN(v^T\mu,v^T\Sigma v),\cN (v^T\wt\mu, v^T\wt\Sigma v)) = \SD(v^T G_1, v^T  G_2) \leq \SD(G_1, G_2) 
    \end{align*}
    where the first equality comes from the fact that the projection of a Gaussian is also Gaussian, with the above parameters.
    
    By the second univariate bound in Lemma \ref{lem:gaussians-tv-parameters}, we have that for every unit vector $v \in \R^d$,
    $$
        \frac{(1-\gamma)^4}{4} \leq \frac {v^T\Sigma v} {v^T \wt \Sigma v} \leq \frac {4}{(1-\gamma)^4}.
    $$
    Rearranging the above gives us (1).

    For (2), we have that for every unit vector $v \in \R^d$,
    \begin{align*}
        \frac {v^T(\wt\mu - \mu)(\wt\mu - \mu)^T v} {v^T(\Sigma + \wt \Sigma)v} = \frac {(v^T\wt\mu - v^T\mu)^2} {v^T\Sigma v + v^T \wt\Sigma v} \leq \frac {(v^T\wt\mu - v^T\mu)^2} {\max\{v^T\Sigma v, v^T \wt\Sigma v \}} \leq \frac {8\gamma} {1- \gamma}
    \end{align*}
    where the last inequality comes from applying the first univariate bound in Lemma \ref{lem:gaussians-tv-parameters}. Rearranging the above gives us (2).
\end{proof}

\subsubsection{Unknown Mean and Covariance with $d+1$ Public Samples}
Let $L,U$ be quantities, as defined in Algorithm~\ref{alg:pub-precond}, and let $L_{\gamma} = \tfrac{(1-\gamma)^4}{4}\cdot L$ and $U_{\gamma} = \tfrac{4}{(1-\gamma)^4}\cdot U$ for $0 \leq \gamma < 1$. The following is an analogue of Lemma~\ref{lem:pub-precond-guarantee} for the case, where we apply Algorithm~\ref{alg:pub-precond} (public data preconditioning) with public data that comes from a Gaussian that is at most $\gamma$-far in TV distance from the private data distribution.

\begin{lem}[$\gamma$-Far Public Data Preconditioning]\label{lem:robust-pub-precond-guarantee}
    For all $\beta>0$, there exists an algorithm that takes $d+1$ independent samples $\wt{X}=(\wt{X}_1,\dots,\wt{X}_{d+1})$ from a Gaussian $\cN(\wt{\mu},\wt{\Sigma})$ over $\R^d$, and outputs $\wh{\mu} \in \R^d$, $\wh{\Sigma} \in \R^{d \times d}$, $L_{\gamma} \in \R$, and $U_{\gamma} \in \R$, such that for a Gaussian $\cN(\mu,\Sigma)$ over $\R^d$ with $0 \leq \SD(\cN(\mu,\Sigma),\cN(\wt{\mu},\wt{\Sigma})) \leq \gamma < 1$, if
    $\Sigma_Y = \tfrac{1}{L_{\gamma}} \wh \Sigma^{-1/2}\Sigma\wh\Sigma^{-1/2}$ and $\mu_Y = \tfrac{1}{\sqrt{L_{\gamma}}} \wh \Sigma^{-1/2} (\mu - \wh \mu)$, then with probability at least $1-\beta$ over the sampling of the data,
    \begin{enumerate}
        \item $\id \preceq \Sigma_Y \preceq \tfrac {U_\gamma} {L_\gamma} \id$
        \item $\|\mu_Y\|_2 \leq  \sqrt{\tfrac {U_\gamma} {L_\gamma}} \cdot \l( \sqrt{\tfrac {10\gamma}{1-\gamma}} + \sqrt{ 5\log\l(\tfrac 3 \beta\r)}\r)$
    \end{enumerate}
    where $\tfrac{U_\gamma}{L_\gamma} = O\l(\frac {d^2\log \l( \tfrac 1 \beta\r)}  {\beta^2(1-\gamma)^8}\r)$.
\end{lem}
\begin{proof}
    We prove the lemma by proving the utility of a modified version of Algorithm~\ref{alg:pub-precond}, which returns $L_{\gamma}$ and $U_{\gamma}$, instead of $L$ and $U$, respectively. The result follows from tracing through the proof of Lemma~\ref{lem:pub-precond-guarantee}, and applying Lemma~\ref{lem:d-gaussians-tv-parameters} as necessary. We highlight the differences.

We start with (1). By the same chain of equivalences in the proof of Lemma \ref{lem:pub-precond-guarantee}, it suffices to show that $L_{\gamma}\Sigma^{-1/2}\wh\Sigma\Sigma^{-1/2} \preceq \id \preceq U_{\gamma}\Sigma^{-1/2}\wh\Sigma\Sigma^{-1/2}$. We have the following.
\begin{align*}
    \Sigma^{-1/2}\widehat\Sigma\Sigma^{-1/2} 
    &= \Sigma^{-1/2}\l(\frac{1}{d} \sum_{i=1}^{d+1}  ((\wt \Sigma^{1/2} Z_i + \wt\mu) - \wh\mu) ((\wt\Sigma^{1/2} Z_i + \wt\mu) - \wh\mu)^T\r)\Sigma^{-1/2} \\
    &= \Sigma^{-1/2}\l(\frac{1}{d} \sum_{i=1}^{d+1}  (\wt \Sigma^{1/2} (Z_i - \wh\mu_Z)) (\wt \Sigma^{1/2} (Z_i - \wh\mu_Z ))^T\r)\Sigma^{-1/2} \\
    &= \Sigma^{-1/2}\wt\Sigma^{1/2} \widehat\Sigma_Z \wt\Sigma^{1/2}\Sigma^{-1/2}
\end{align*}
In the above, $Z_i$, $\wh{\mu}_Z$, and $\wh{\Sigma}_Z$ are quantities, as defined in the proof of Lemma~\ref{lem:pub-precond-guarantee}. From the same proof, we know that with probability $\geq 1- \beta/3$, we have $U \wh\Sigma_Z \succeq \id$, which implies that
$$
     U\Sigma^{-1/2}\wt \Sigma^{1/2} \widehat \Sigma_Z \wt \Sigma^{1/2} \Sigma^{-1/2}\succeq \Sigma^{-1/2} \widetilde{\Sigma} \Sigma^{-1/2} \succeq \frac {(1-\gamma)^4} 4 \cdot \id
$$
where the last inequality follows from (1) in Lemma \ref{lem:d-gaussians-tv-parameters}. Recalling that we set $U_\gamma = \tfrac 4 {(1-\gamma)^4}\cdot U$, and rearranging gives us that $U_\gamma \Sigma^{-1/2}\wh\Sigma\Sigma^{-1/2} \succeq \id$, as desired. Similarly, with probability $\geq 1 - \beta/3$, $L \widehat \Sigma_Z \preceq \id$, which implies that
$$
L\Sigma^{-1/2}\wt \Sigma^{1/2} \widehat \Sigma_Z \wt \Sigma^{1/2} \Sigma^{-1/2}\preceq \Sigma^{-1/2} \widetilde{\Sigma} \Sigma^{-1/2} \preceq \frac  4 {(1-\gamma)^4} \cdot \id
$$
which again, after rearranging, allows us to conclude $L_\gamma \Sigma^{-1/2}\widehat\Sigma \Sigma^{-1/2} \preceq \id$.

It remains to verify that (2) holds. Write
$$
    \mu_Y = \frac 1 {\sqrt{L_\gamma}} \wh\Sigma^{-1/2}(\mu -\wh \mu) = \frac 1 {\sqrt{L_\gamma}}\wh\Sigma^{-1/2}(\mu - \wt \mu - \wt \Sigma^{1/2}\wh\mu_Z) = \frac 1 {\sqrt{L_\gamma}}\wh\Sigma^{-1/2}(\mu - \wt \mu)  -\frac 1 {\sqrt{L_\gamma}}\wh\Sigma^{-1/2}\wt\Sigma^{1/2}\wh\mu_Z.
$$
We bound the two terms separately. Note that the second term appears in the proof of Lemma~\ref{lem:pub-precond-guarantee}. By the same argument as in that proof, we claim that with probability $\geq 1-\tfrac{\beta}{3}$, $\|\wh \mu_Z\|_2 \leq \sqrt{\tfrac {d + 2\sqrt{d\log(3/\beta)} + 2\log(3/\beta)} {d+1}}$. When we indeed have $U \wh \Sigma_Z \succeq \id$ (that is, our events from (1) occur), we have $\left\|\wh{\Sigma}^{-1/2}\wt{\Sigma}^{1/2}\right\|_2 \leq \sqrt{U}$. Taking the union bound, with probability $\geq 1-\beta$,
$$
    \left\|- \frac 1 {\sqrt{L_\gamma}} \wh\Sigma^{-1/2} \wt \Sigma^{1/2}\wh\mu_Z\right\|_2 \leq \sqrt{\frac U {L_\gamma}} \cdot \sqrt{\frac {d + 2\sqrt{d\log(3/\beta)} + 2\log(3/\beta)} {d+1}}.
$$
Now, we argue that the first term is also bounded. First, we apply Lemma~\ref{lem:d-gaussians-tv-parameters} to get
$$
    (\mu - \t\mu)(\mu - \t\mu)^T \preceq \frac {8\gamma}{1-\gamma}(\Sigma + \wt\Sigma) \preceq \frac {8\gamma}{1-\gamma}\l(\frac 4 {(1-\gamma)^4}\wt\Sigma + \wt\Sigma\r) \preceq \frac {40\gamma} {(1-\gamma)^5} \wt \Sigma.
$$
Note that $U\widehat\Sigma_Z \succeq \id \implies \wt \Sigma \preceq U\widehat\Sigma$. Plugging this in above, and rearranging gives
$$
    \widehat\Sigma^{-1/2}(\mu - \t \mu)(\mu - \t \mu)^T\widehat \Sigma^{-1/2} \preceq \frac 4 {(1-\gamma)^4} \cdot U \cdot \frac {10\gamma}{1-\gamma} \cdot \id = U_\gamma \cdot \frac {10\gamma} {1-\gamma} \cdot \id.
$$
Thus, we have 
$$
    \left\|\frac 1 {\sqrt{L_\gamma}} \widehat \Sigma^{-1/2}(\mu - \t \mu)\right\|_2 \leq \frac 1 {\sqrt{L_\gamma}} \cdot \sqrt{\lambda_1(\widehat\Sigma^{-1/2}(\mu - \t \mu)(\mu - \t \mu)^T\widehat \Sigma^{-1/2})} \leq \sqrt {\frac {U_\gamma} {L_\gamma}} \cdot \sqrt{\frac{10\gamma}{1-\gamma}}.
$$
Combining the two terms gives us
$$
    \|\mu_Y\|_2 \leq \sqrt{\frac {U_\gamma} {L_\gamma}} \cdot \l( \sqrt{\frac {10\gamma} {1-\gamma}} + \sqrt{\frac {d + 2\sqrt{d\log(3/\beta)} + 2\log(3/\beta)} {d+1}}\r)
$$
which completes the proof.
\end{proof}

Lemma~\ref{lem:robust-pub-precond-guarantee}, combined with guarantees of the Gaussian estimators from Lemma~\ref{lem:gaussian-estimator}, allows us to conclude the following analogue to Theorem~\ref{thm:gaussian}.

\begin{thm}[$\gamma$-Far Public-Private Gaussian Estimation]\label{thm:robust-gaussian}
    For all $\alpha, \beta, \eps, \rho > 0$, there exist $\eps$-DP and $\rho$-zCDP algorithms that take $d+1$ public samples $\wt{X} = (\wt{X}_1,\dots,\wt{X}_{d+1})$ and $n$ private samples $X = (X_1,\dots,X_n)$ from unknown, $d$-dimensional Gaussians $\cN(\wt{\mu},\wt{\Sigma})$ and $\cN(\mu,\Sigma)$, respectively, such that $0 \leq \SD(\cN(\mu,\Sigma),\cN(\wt{\mu},\wt{\Sigma})) \leq \gamma < 1$, and are private with respect to the private samples, and return $\wh{\mu}_X \in \R^d$ and $\wh{\Sigma}_X \in \R^{d \times d}$, such that $\SD(\cN(\wh \mu_X, \wh \Sigma_X), \cN(\mu, \Sigma)) \leq \alpha$ with probability at least $1-\beta$, as long as the following bounds on $n$ hold.
    \begin{enumerate}
        \item For $\rho$-zCDP (with computational efficiency) via \cite{KamathLSU19} and $d+1$ public samples,
            \begin{align*}
                n = O\left(\frac{d^2 + \log\l(\frac 1 \beta\r)}{\alpha^2} +
                    \frac{d^2\cdot\polylog\left(\frac{d}{\alpha\beta\rho}\right) +
                    d\log\left(\frac{d}{\alpha\beta\rho(1-\gamma)}\right)}
                    {\alpha\sqrt{\rho}}  + \frac{d^{1.5}\cdot
                    \polylog\left(\frac{d}{\beta\rho(1-\gamma)}\right)}{\sqrt{\rho}}\right).
            \end{align*}
        \item For $\eps$-DP (without computational efficiency) via \cite{BunKSW19} and $d+1$ public samples,
            \begin{align*}
                n = O\l(\frac {d^2 + \log \l(\frac 1 \beta\r)} {\alpha^2} + \frac {d^2\log \l(\frac d {\alpha\beta(1-\gamma)}\r)} {\alpha\eps}\r).
            \end{align*}
    \end{enumerate}
\end{thm}
\begin{proof}
    The theorem follows from the privacy and the utility guarantees of a modified version of Algorithm~\ref{alg:gaussian}, which uses the modified version of Algorithm~\ref{alg:pub-precond} as outlined in Lemma~\ref{lem:robust-pub-precond-guarantee} (outputting $L_{\gamma}$ and $U_{\gamma}$ instead of $L$ and $U$). The proof remains the same as that of Theorem~\ref{thm:gaussian}.
\end{proof}

\section{Estimating Gaussian Mixtures}

In this section, we state our algorithms to learn mixtures
of well-separated Gaussians under differential privacy, when we
have trace amounts of public data available to us. We then provide
theoretical guarantees for privacy of our algorithms with
respect to the private data samples, along with its utility.
We analyse two cases here: when we have $\wt{O}(1/\mw)$ public samples
available to us, and when we have $\wt{O}(d/\mw)$ public samples
available.

The general scheme of our algorithms is similar to that of the non-private
algorithm for learning Mixtures of Gaussians by \cite{AchlioptasM05},
and the private algorithm (with no public data available) by
\cite{KamathSSU19}. Many algorithms for learning mixtures of
Gaussians follow a specific outline -- use PCA to project the
data onto a low-dimensional subspace, which would separate the
largest Gaussian from the rest; partition the dataset again, if
there is more than one Gaussian present, otherwise isolate the lone
Gaussian; estimate the parameters of that Gaussian; repeat the process
on the remaining points. Our algorithms are also spectral algorithms
that rely upon techniques, like PCA, but use their private counterparts
at various stages.

The difference in the aforementioned cases, where we have $O(1/\mw)$
and $O(d/\mw)$ public samples available, is that in the former, we
have very few public samples, but they are enough to be able to
isolate a group of nearby clusters (called, a ``supercluster''), which
we could then use for an application of private PCA. In the latter case,
we wouldn't need to apply private PCA at all because the number
of public samples is enough to be able to do non-private PCA
accurately.

\paragraph{Assumption.} We make an assumption about the shape of the covariances of
all Gaussians, which essentially says that the Gaussians are
not too flat or degenerate. Let $N$ be the total number of
points sampled from $D \in \cG(d,k,s)$. We formalise this as follows.
\begin{align}
    \forall i \in [k],~~~
        \|\Sigma_i\|_F\sqrt{\log(Nk/\beta)} \leq \frac{1}{8}\tr(\Sigma_i)~~~
        \text{and}~~~
        \|\Sigma_i\|_2\log^2(Nk/\beta) \leq \frac{1}{8}\tr(\Sigma_i)
        \label{eq:gaussian-not-flat}
\end{align}
Note that this implies that $d \geq 8\log^2(Nk/\beta)$
because $\tr(\Sigma_i) \leq d\|\Sigma_i\|_2$. We also assume
that $\beta < 1/2$.

\paragraph{Notation.} We define some notation before moving
on to the
technical content for Gaussian mixtures. We say that a Gaussian $\cN(\mu,\Sigma)$
in high dimensions satisfying Assumption (\ref{eq:gaussian-not-flat})
is \emph{contained} within $S \sub \R^d$, if
$\ball{\mu}{\sqrt{\frac{3}{2}\tr(\Sigma)}} \sub S$.
For a low-dimensional Gaussian $\cN(\mu',\Sigma')$ in $\ell < d$
dimensions, we say that a set $S \sub \R^d$ \emph{contains} the Gaussian
for fixed $0 < \beta < 1$ and $N \geq 1$,
if $\ball{\mu'}{\sqrt{\|\Sigma'\|_2}\sqrt{2\ell\ln(2N\ell/\beta)}} \sub S$. Also, given $D \in \cG(d,k,s)$, we say that
$B \sub \R^d$ is \emph{pure}, if for each Gaussian component $i$,
$B$ either contains the Gaussian $\cN(\mu_i,\Sigma_i)$, or
$\ball{\mu_i}{\sqrt{\frac{3}{2}\tr(\Sigma_i)}} \cap B = \emptyset$.
Similarly, given a set of points $T$ from $D$, we say that
$S \sub T$ is \emph{clean}, if for every $i \in [k]$,
$S$ has all points from component $i$ that lie in $T$,
or it has none of them. We sometimes say that for a clean
$S_1 \sub T$, $S_2$ is a clean subset of $S_1$ if
$S_2 \sub S_1$ and $S_2$ is clean with respect to $T$.

\subsection{Deterministic Regularity Conditions}

Here, we state a few results that would be useful in solving
the problems in the next two subsections.
The following condition bounds the number of points from each
component of a Gaussian mixture.
\begin{condition}[Sample Frequency]\label{cond:gmm-frequency}
    Suppose we have $n$ samples from a mixture of Gaussians
    $D = \{(\mu_1,\Sigma_1,w_1),\dots,(\mu_k,\Sigma_k,w_k)\}
    \in \cG(d,k,s)$ satisfying Assumption (\ref{eq:gaussian-not-flat}).
    Then for each $i \in [k]$, the number of samples from
    component $i$ lies in
    $\left[\tfrac{nw_i}{2},\tfrac{3nw_i}{2}\right]$.
    Furthermore, if $w_i \geq \tfrac{4\alpha}{9k}$, then
    the number of points from component $i$ lies in
    $\left[n(w_i-\tfrac{\alpha}{9k}),n(w_i+\tfrac{\alpha}{9k})\right]$.
\end{condition}

The next condition bounds the distance of any point sampled
from a Gaussian mixture from the mean of its respective
component.
\begin{condition}[Intra-Gaussian Distances From Mean]
    \label{cond:intra-gaussian-mean}
    Suppose we have $n$ samples from a mixture of Gaussians
    $D = \{(\mu_1,\Sigma_1,w_1),\dots,(\mu_k,\Sigma_k,w_k)\}
    \in \cG(d,k,s)$ satisfying Assumption (\ref{eq:gaussian-not-flat}).
    Then for each $i \in [k]$, if $x$ is one of the samples
    from component $i$, then
    $\tfrac{3}{4}\tr(\Sigma_i) \leq \|x-\mu_i\|^2 \leq \tfrac{3}{2}\tr(\Sigma_i)$.
\end{condition}

This condition bounds the distance between any two points sampled
from the same component of a Gaussian mixture.
\begin{condition}[Intra-Gaussian Distances Between Points]
    \label{cond:intra-gaussian-points}
    Suppose we have $n$ samples from $D = \{(\mu_1,\Sigma_1,w_1),\dots,(\mu_k,\Sigma_k,w_k)\}
    \in \cG(d,k,s)$ satisfying Assumption (\ref{eq:gaussian-not-flat}).
    Then for each $i \in [k]$, if $x,y$ are two of the samples
    from component $i$, then
    $\tfrac{3}{2}\tr(\Sigma_i) \leq \|x-y\|^2 \leq 3\tr(\Sigma_i)$.
\end{condition}

This final condition quantifies the minimum distance between
any two points from different components of a Gaussian mixture.
\begin{condition}[Inter-Gaussian Distances]\label{cond:inter-gaussian}
    Suppose we have $n$ samples from a mixture of Gaussians
    $D = \{(\mu_1,\Sigma_1,w_1),\dots,(\mu_k,\Sigma_k,w_k)\}
    \in \cG(d,k,s)$ satisfying Assumption (\ref{eq:gaussian-not-flat}),
    where $s \geq \Omega(\sqrt{\ln(n/\beta)})$.
    Then for any $i,j \in [k]$ with $i \neq j$, if $x$ is one of
    the samples from component $i$, and $y$ is one of the samples
    from component $j$, then
    $\|x-y\| \geq \tfrac{\sqrt{\max\{\tr(\Sigma_i),\tr(\Sigma_j)\}}}{4}$.
\end{condition}

We state another condition that says that in low dimensions,
too, the distance between the mean of a Gaussian from all its
points is bounded, and it is true for all the components of the
mixture.
\begin{condition}[Intra-Gaussian Distances from Mean in Low Dimensions]
    \label{cond:intra-gaussian-mean-low}
    Suppose we have $n$ samples from $D =
    \{(\mu_1,\Sigma_1,w_1),\dots,(\mu_k,\Sigma_k,w_k)\}
    \in \cG(\ell,k)$. For each $i \in [k]$, let
    $\sigma_i^2 = \|\Sigma_i\|_2$. Then for any $i \in [k]$
    and a fixed $0 < \beta < 1$,
    if $x$ is a datapoint sampled from $\cN(\mu_i,\Sigma_i)$,
    then $\|\mu_i - x\| \leq \sigma_i\sqrt{2\ell\ln(2n\ell/\beta)}$.
\end{condition}

The final condition about low-dimensional Gaussian mixtures
bounds the distance between two points from the same component.
\begin{condition}[Intra-Gaussian Distances Between Points in Low Dimensions]
    \label{cond:intra-gaussian-points-low}
    Suppose we have $n$ samples from $D =
    \{(\mu_1,\Sigma_1,w_1),\dots,(\mu_k,\Sigma_k,w_k)\}
    \in \cG(\ell,k)$. For each $i \in [k]$, let
    $\sigma_i^2 = \|\Sigma_i\|_2$. Then for any $i \in [k]$
    and a fixed $0 < \beta < 1$,
    if $x,y$ are a datapoints sampled from $\cN(\mu_i,\Sigma_i)$,
    then $\|x - y\| \leq 2\sigma_i\sqrt{2\ell\ln(2n\ell/\beta)}$.
\end{condition}

Now, we prove that the above conditions hold with high probability.

\begin{lem}\label{lem:gmm-frequqncy}
    Suppose we have $n$ samples from $D \in \cG(d,k,s)$ satisfying
    Assumption (\ref{eq:gaussian-not-flat}). If
    $$n \geq \max\left\{\frac{12}{\mw}\ln(2k/\beta),
        \frac{405k^2}{2\alpha^2}\ln(2k/\beta)\right\},$$
    then Condition~\ref{cond:gmm-frequency} holds with probability
    at least $1-\beta$.
\end{lem}
\begin{proof}
    We simply use Lemma~\ref{lem:chernoff-mult} $k$ times with
    $p \geq \mw$, and Lemma~\ref{lem:chernoff-add} $k$ times,
    and take the union bound.
\end{proof}

\begin{lem}\label{lem:intra-gaussian-mean}
    Suppose we have $n$ samples from $D \in \cG(d,k,s)$ satisfying
    Assumption (\ref{eq:gaussian-not-flat}). Then Condition~\ref{cond:intra-gaussian-mean} holds with probability
    at least $1-\beta$.
\end{lem}
\begin{proof}
    From the Hanson-Wright inequality (Lemma~\ref{lem:HW}),
    $\forall i, \forall x\sim \cN(\mu_i, \Sigma_i)$, we have that
    \[\tr(\Sigma_i) - 2\|\Sigma_i\|_F\sqrt{\log(n/\beta)}
        \leq \|x-\mu_i\|_2^2 \leq \tr(\Sigma_i) +
        2\|\Sigma_i\|_F\sqrt{\log(n/\beta)} + 2\|\Sigma_i\|_2\log(n/\beta).\]
    assumption (\ref{eq:gaussian-not-flat}), combined with the
    above, gives us the required result.
\end{proof}

\begin{lem}\label{lem:intra-gaussian-points}
    Suppose we have $n$ samples from $D \in \cG(d,k,s)$ satisfying
    Assumption (\ref{eq:gaussian-not-flat}). Then
    Condition~\ref{cond:intra-gaussian-points} holds with probability
    at least $1-\beta$.
\end{lem}
\begin{proof}
    For every $i$ and for any $x,y \sim \cN(\mu_i,\Sigma_i)$,
    we have that $x-y\sim \cN(0,2\Sigma_i)$. Using
    Assumption (\ref{eq:gaussian-not-flat}) and Lemma~\ref{lem:HW} again,
    we have the result.
\end{proof}

The following is a quantification of the median radius of a
Gaussian in terms of the trace of its covariance.
\begin{lem}\label{lem:median-radius}
    Suppose we have $n$ samples from $D = \{(\mu_1,\Sigma_1,w_1),\dots,(\mu_k,\Sigma_k,w_k)\} \in \cG(d,k,s)$ satisfying
    Assumption (\ref{eq:gaussian-not-flat}). Then for each $i \in [k]$,
    the median radius $R_i$ of component $i$ lies in
    $\left(\sqrt{\tfrac{3}{4}\tr(\Sigma_i)},\sqrt{\tfrac{3}{2}\tr(\Sigma_i)}\right)$.
\end{lem}
\begin{proof}
    For each $i \in [k]$, let $a_i = \sqrt{\tfrac{3}{4}\tr(\Sigma_i)}$
    and $b_i = \sqrt{\tfrac{3}{2}\tr(\Sigma_i)}$.
    We know from the proof of Lemma~\ref{lem:intra-gaussian-mean}
    that for a given $i \in [k]$, $1-\tfrac{\beta}{k}$
    of the probability mass of $G_i \coloneqq \cN(\mu_i,\Sigma_i)$ lies
    in $\ball{\mu_i}{b_i} \setminus \ball{\mu_i}{a_i}$.
    Let $\beta_1 \leq \beta/k$ be the probability mass of $G_i$
    in $\ball{\mu_i}{a_i}$, and let $\beta_2 = \beta/k - \beta_1$.
    We know that $\beta_1 < 1/2$, therefore, $R_i > a_i$.
    Since the mass of $G_i$ outside $\ball{\mu_i}{b_i}$ is
    $\beta_2 < 1/2$, it must be the case that $R_i < b_i$.
    Hence, the claim.
\end{proof}

\begin{lem}\label{lem:inter-gaussian}
    Suppose we have $n$ samples from $D \in \cG(d,k,s)$ satisfying
    Assumption (\ref{eq:gaussian-not-flat}). If
    $s \geq \Omega(\sqrt{\ln(n/\beta)})$,
    then Condition~\ref{cond:inter-gaussian} holds with probability
    at least $1-2\beta$.
\end{lem}
\begin{proof}
    For this, we will use Lemmata~\ref{lem:gaussian-anti-conc} and
    \ref{lem:median-radius}. We know from
    Lemma~\ref{lem:gaussian-anti-conc} that with probability
    at least $1-\beta/m$, for any given $i \in [k]$, and $x$ sampled
    from $G_i \coloneqq \cN(\mu_i,\Sigma_i)$ in the dataset sampled
    from $D$,
    $$\|x-z\|^2 \geq (\max\{R_i-\ln(\beta/2m)\sigma_i,0\})^2 +
        \|z-\mu_i\|^2 - 2\sqrt{2\ln(\beta/2m)}\sigma_i\|z-\mu_i\|.$$
    For any $j \in [k]$ with $j \neq i$, let $z = \mu_j$. Suppose
    $R_i$ and $R_j$ are the median radii of $G_i$ and $G_j$,
    respectively. WLOG, let's assume that $R_i \geq R_j$. Then we
    have the following.
    \begin{align*}
        \|x-\mu_j\|^2 &\geq (\max\{R_i-\ln(2m/\beta)\sigma_i,0\})^2 +
                \|\mu_j-\mu_i\|^2 -
                2\sqrt{2\ln(2m/\beta)}\sigma_i\|\mu_j-\mu_i\|\\
            &\geq (\max\{R_i-\ln(2m/\beta)\sigma_i,0\})^2 +
                \frac{\|\mu_j-\mu_i\|^2}{2}
                \tag{Separation condition and
                Assumption (\ref{eq:gaussian-not-flat}).}
    \end{align*}
    Now, let $y \sim G_j$. Due to our separation condition,
    we know that,
    $$2\sqrt{2\ln(2m/\beta)}\sigma_j \leq
        \frac{\|\mu_j-\mu_i\|}{2\sqrt{2}} \leq \frac{\|x-\mu_j\|}{2}.$$
    We apply Lemma~\ref{lem:gaussian-anti-conc} again with $z = x$.
    \begin{align*}
        \|x - y\|^2 &\geq (\max\{R_j-\ln(2m/\beta)\sigma_j,0\})^2 +
                \|x-\mu_j\|^2 - 2\sqrt{2\ln(2m/\beta)}\sigma_j\|x-\mu_j\|\\
            &\geq (\max\{R_j-\ln(2m/\beta)\sigma_j,0\})^2 +
                \frac{\|x-\mu_j\|^2}{2}
                \tag{Separation condition.}\\
            &\geq (\max\{R_j-\ln(2m/\beta)\sigma_j,0\})^2 +
                \frac{(\max\{R_i-\ln(2m/\beta)\sigma_i,0\})^2}{4} +
                \frac{\|\mu_j-\mu_i\|^2}{8}\\
            &\geq \frac{\tr(\Sigma_j)}{4} + \frac{\tr(\Sigma_i)}{16}
                + \frac{\|\mu_j-\mu_i\|^2}{8}
                \tag{Assumption (\ref{eq:gaussian-not-flat}).}\\
            &\geq \frac{\tr(\Sigma_i)}{16}
    \end{align*}
    This proves the lemma.
\end{proof}

\begin{lem}\label{lem:intra-gaussian-mean-low}
    Suppose we have $n$ samples from $D =
    \{(\mu_1,\Sigma_1,w_1),\dots,(\mu_k,\Sigma_k,w_k)\}
    \in \cG(\ell,k)$. Then for a fixed $0 < \beta < 1$,
    Condition~\ref{cond:intra-gaussian-mean-low} holds
    with probability at least $1-\beta$.
\end{lem}
\begin{proof}
    Fix $i \in [k]$, $x$ be an arbitrary sample from component $i$,
    and an eigenvector of $\Sigma_i$
    $v$. In that direction, using Lemma~\ref{lem:gauss-conc-1d},
    we know that with probability at least $1-\tfrac{\beta}{n\ell}$,
    $\|(x-\mu_i)vv^T\| \leq \sigma_i\sqrt{2\ln(2n\ell/\beta)}$.
    Applying the union bound in all $\ell$ directions of $\Sigma_i$,
    we have that with probability at least $1-\tfrac{\beta}{n}$,
    $\|x-\mu_i\| \leq \sigma_i\sqrt{2\ell\ln(2n\ell/\beta)}$. Applying
    the union bound again over all $n$ points, we get the required
    result.
\end{proof}

\begin{lem}\label{lem:intra-gaussian-points-low}
    Suppose we have $n$ samples from $D =
    \{(\mu_1,\Sigma_1,w_1),\dots,(\mu_k,\Sigma_k,w_k)\}
    \in \cG(\ell,k)$. Then for a fixed $0 < \beta < 1$,
    Condition~\ref{cond:intra-gaussian-points-low} holds
    with probability at least $1-\beta$.
\end{lem}
\begin{proof}
    The proof just involves application of triangle inequality
    for any pairs of points from the same component, along with
    Lemma~\ref{lem:intra-gaussian-mean-low}.
\end{proof}

\subsection{\texorpdfstring{$\wt{O}(1/\mw)$}{} Public Samples}

In this subsection, we assume that we have $\wt{O}(1/\mw)$ public
samples available. Here, the number of public samples is not
enough to be able to do PCA accurately, so we, instead, try to
use them to isolate the superclusters in the data, and perform
private actions, like private PCA and private estimation algorithms,
on the private data based on that information.

We start with a superclustering algorithm for public data.
It first computes a constant-factor approximation to the diameter
($\approx r$) of the largest Gaussian in the public dataset by finding
the largest minimum pairwise distance among all points. It then tries
to find a supercluster around the point ($x$) that has the largest
minimum distance to all other points. The idea is that $\ball{x}{r}$
will have a lot of points, but if $\ball{x}{r+r}$ doesn't have any
more points, and $\ball{x}{r+2r}$ doesn't have any more points
either, then it must be the case that with high probability,
no other Gaussian would have a point in $\ball{x}{2r}$. This
follows from the concentration properties of high-dimensional
Gaussians, and the fact that $r$ approximates the diameter of the
largest Gaussian. If we do encounter more points in either
$\ball{x}{2r}$ or $\ball{x}{3r}$, we expand the ball by $O(r)$
and check for these conditions again. We show that the ball
returned by the algorithm is of size at most $O(kr)$, and that
it contains the Gaussian corresponding to the point $x$.

\begin{algorithm}[!ht]
\caption{Superclustering on Public Data
    $\SC_{k, m}(\wt{X})$}\label{alg:supercluster}
\KwIn{Samples $\wt{X}_1,\dots,\wt{X}_{m'} \in \R^d$.
    Parameters $\beta, k, m > 0$.}
\KwOut{Centre $c \in \R^d$ and radius $R \in \R$.}
\vspace{5pt}

Let $r \gets 16\cdot\max\limits_{x\in\wt{X}}\min\limits_{y\in\wt{X}}\|x-y\|$.\\
Let $c \gets \argmax\limits_{x\in\wt{X}}\min\limits_{y\in\wt{X}}\|x-y\|$.
\vspace{5pt}

Set $pure \gets \False$.\\
Let $R \gets r$.\\
\For{$i \gets 1,\dots,k$}{
    Let $m_i \gets \abs{\ball{c}{R} \cap \wt{X}}$.\\
    \If{$\abs{\ball{c}{R+r} \cap \wt{X}} = m_i$}{
        \If{$pure = \True$}{
            \Return $\ball{c}{R}$
        }
        \ElseIf{$\abs{\ball{c}{R+2r} \cap \wt{X}} = m_i$}{
            \Return $(c,R+r)$
        }
        \Else{
            $pure = \False$\\
            $R \gets R + 3r$
        }
    }
    \Else{
        \If{$\abs{\ball{c}{R+2r} \cap \wt{X}} = \abs{\ball{c}{R+r} \cap \wt{X}}$}{
            $pure = \True$
        }
        \Else{
            $pure = \False$
        }
        $R \gets R + 2r$
    }
}
\vspace{5pt}

\Return $(c,R)$.
\vspace{5pt}
\end{algorithm}

We now state the main result about Algorithm~\ref{alg:supercluster}.

\begin{thm}\label{thm:supercluster}
    There exists an algorithm, which if given a clean subset of
    $$m \geq O\left(\frac{\ln(k/\beta)}{\mw}\right)$$
    points from $D \in \cG(d,k)$, where $D$ satisfies
    Assumption (\ref{eq:gaussian-not-flat}), it outputs
    $c \in \R^d$ and $R \in \R$, such that the following
    holds, given Conditions~\ref{cond:gmm-frequency} to
    \ref{cond:inter-gaussian} hold.
    \begin{enumerate}
        \item $\ball{c}{R}$ is pure.
        \item Let $\Sigma'$ be the covariance of the largest
            Gaussian (covariance with maximum trace) contained in
            $\ball{c}{R}$. Then $R \in O(k\sqrt{\tr(\Sigma')})$.
    \end{enumerate}
\end{thm}
\begin{proof}
    Let $\cN(\mu,\Sigma)$ be the largest Gaussian that has
    points in $\wt{X}$. First, we claim that
    $r \in \Theta(\sqrt{\tr(\Sigma)})$. We know from
    Lemma~\ref{lem:inter-gaussian}, that the minimum distance
    between $x \sim \cN(\mu_i,\Sigma_i)$ and $y \sim \cN(\mu_j,\Sigma_j)$
    for $i \neq j$ is lower bounded by
    $\tfrac{\sqrt{\max\{\tr(\Sigma_i),\tr(\Sigma_j)\}}}{4}$.
    We also know from Lemma~\ref{lem:intra-gaussian-points}
    that for any $i \in [k]$, and any $x,z \sim \cN(\mu_i,\Sigma_i)$,
    $\|x-z\| \geq \sqrt{\tfrac{3}{2}\tr(\Sigma_i)}$. Let $y \in \wt{X}$,
    such that $y$ is sampled from component $j \neq i$ and
    is closest to $x$, that is, $y$ is the closest point to $x$
    that has not been sampled from component $i$ to which $x$
    belongs. This means that distance between $x$ and its nearest
    point is at least $\min\left\{\sqrt{\tfrac{3}{2}\tr(\Sigma_i)},
    \tfrac{\sqrt{\max\{\tr(\Sigma_i),\tr(\Sigma_j)\}}}{4}\right\}
    \geq \min\left\{\sqrt{\tfrac{3}{2}\tr(\Sigma_i)},
    \tfrac{\sqrt{\tr(\Sigma_i)}}{4}\right\}$. We also know from
    Lemma~\ref{lem:intra-gaussian-points} that for any
    $x,z \sim \cN(\mu_i,\Sigma_i)$, $\|x-z\| \leq \sqrt{3\tr(\Sigma_i)}$.
    This means that the point in $\wt{X}$ that is closest to
    $x$ is at most $\sqrt{3\tr(\Sigma_i)}$ far from $x$.
    This shows that the distance between $x$ and the point
    in $\wt{X}$ that is closest to $x$ lies in
    $\left[\tfrac{\sqrt{\tr(\Sigma_i)}}{4},\sqrt{3\tr(\Sigma_i)}\right]$.
    This is true for any point in $\wt{X}$. Therefore, the largest
    minimum distance has to lie in the interval
    $\left[\tfrac{\sqrt{\tr(\Sigma)}}{4},\sqrt{3\tr(\Sigma)}\right]$.
    This shows that $r \in \Theta(\sqrt{\tr(\Sigma)})$.
    
    We use this to show that $\ball{c}{R}$ is pure. We prove
    the following claims for that.
    \begin{clm}
        If at the end of any iteration, the algorithm doesn't exit
        the loop, but calls any of the return steps within the loop
        instead, then the returned ball is pure.
    \end{clm}
    \begin{proof}
        Suppose it happens in iteration $i$.
        Consider the case where the return statement was called
        in the \textbf{elif} block.
        If $\ball{c}{R}$ were intersecting with a component,
        then $\ball{c}{R+r}$ would contain it
        (Lemmata~\ref{lem:intra-gaussian-mean} and
        \ref{lem:intra-gaussian-points}) because $r$ is large enough.
        Now, we need to show that $\ball{c}{R+r}$ doesn't intersect
        with any new component. By construction, it must mean that
        $\abs{\ball{c}{R} \cap \wt{X}} = m_i$, but
        $\abs{\ball{c}{R+r} \cap \wt{X}} = m_i$
        and $\abs{\ball{c}{R+2r} \cap \wt{X}} = m_i$. If $\ball{c}{R+r}$
        intersects with another component that it doesn't fully contain,
        then it must be
        the case that the component would have at least one more
        point in $\ball{c}{R+2r} \setminus \ball{c}{R+r}$ and would
        be contained in that region (by Lemma~\ref{lem:intra-gaussian-mean})
        because $r$ is large enough. Since that doesn't happen, there
        cannot be any component that is partially contained in
        $\ball{c}{R+r}$. Therefore, $\ball{c}{R+r}$ is pure.
        
        If return were called because $pure$ were true, then
        it must have been the case that in iteration $i-1$, before
        the final update to $R$, $pure$ must have been set to $\True$
        because $\ball{c}{R+r}$ found new points, but $\ball{c}{R+2r}$
        couldn't. In all other case, $pure$ is set to $\False$.
        Since in iteration $i$, $\ball{c}{R+r}$ couldn't
        find new points either, by the same reasoning as above,
        $\ball{c}{R}$ must be pure.
    \end{proof}
    
    \begin{clm}
        In iteration $i$, either the algorithm
        exits with a pure ball containing at least $i-1$ components,
        or at the end of iteration $i$,
        $\ball{c}{R}$ contains at least $i$ components.
    \end{clm}
    \begin{proof}
        We show this by induction on $i$, as defined in
        Algorithm~\ref{alg:supercluster}.\\
        \textbf{Base Case:} At the beginning of the first iteration,
        $R = r$. But we know from Lemmata~\ref{lem:intra-gaussian-points}
        and \ref{lem:intra-gaussian-mean} that $r$ is big enough
        that the mean of the Gaussian from which $c$ has been sampled,
        along with all other points from the same Gaussian would
        be contained in $\ball{c}{r}$ because the mean of the Gaussian
        is at most $\sqrt{\tfrac{3}{2}\tr(\Sigma)}$ away from $c$ and
        all other points from the same component are at most
        $\sqrt{3\tr(\Sigma)}$ away from $c$. Now, if the algorithm
        does not exit after the end of the iteration, our base case
        holds because we have at least one component in the ball at end
        of the iteration. If the algorithm exited, it must be the case
        that $\ball{c}{2r}$ is pure (from the above claim) and has at
        least one component.\\
        \textbf{Inductive Step:} Suppose for $i \geq 1$, assume that
        at the end
        of iteration $i$, $\ball{c}{R}$ contains at least $i$ Gaussian
        components, or that it exits with at least $i-1$ components.
        If the algorithm exits in iteration $i$, we are done.
        So, we assume that it doesn't. In iteration $i+1$, if the algorithm
        exits and returns a ball, then it is pure (by the claim above and
        the inductive hypothesis), and has at least $i$ components. So,
        we're done.
        Otherwise, either $\ball{c}{R+r}$ or $\ball{c}{R+2r}$ finds
        new points. If $\ball{c}{R+r}$ finds new points, it is fully
        containing a component not previously contained within
        $\ball{c}{R}$,
        or it is intersecting with new components. In the former case,
        a new component is added by the end of the iteration, and
        our claim holds. This holds in the latter case, too, by
        Lemmata~\ref{lem:intra-gaussian-mean} and
        \ref{lem:intra-gaussian-points} because $\ball{c}{R+2r}$
        would contain that new component. So, we have contained
        at least $i+1$ components at the end of iteration $i+1$.
        If $\ball{c}{R+r}$ has no new points, but $\ball{c}{R+2r}$
        does, then these must be points from a new component.
        Therefore, $\ball{c}{R+3r}$ would completely contain the
        new component by Lemmata~\ref{lem:intra-gaussian-mean} and
        \ref{lem:intra-gaussian-points}. So, we have $i+1$ components
        again by the end of the iteration.
    \end{proof}
    
    From the above argument, we have that either the algorithm
    exited in the loop, and returned a pure ball, or it had
    $k$ components at the end of iteration $k$, in which case,
    the returned ball is again pure. This gives us the first
    part of the theorem.
    
    Now, we prove the second part. Let $c \in \wt{X}$
    be as defined in Algorithm~\ref{alg:supercluster}. Let $\Sigma$
    be the covariance of the largest Gaussian $G$ in $\wt{X}$.
    Suppose $c$ was sampled from component $G' \coloneqq \cN(\mu',\Sigma')$.
    If $\tr(\Sigma) = \tr(\Sigma')$, then by the proof above,
    $r \in \Theta(\sqrt{\tr(\Sigma)})$. Suppose
    $\tr(\Sigma') < \tr(\Sigma)$. Then in the worst case,
    the nearest point to $c$ is $\sqrt{3\tr(\Sigma')}$ away
    from $c$. It is possible in this case that all points
    from the largest Gaussian had nearest points
    $\tfrac{\sqrt{\tr(\Sigma)}}{4}$ away from them, but
    $\sqrt{3\tr(\Sigma')} \geq \tfrac{\sqrt{\tr(\Sigma)}}{4}$.
    In this case $r$ would still be $16\sqrt{3}$-factor
    approximation of $\sqrt{\tr(\Sigma')}$. We know that
    the returned ball $\ball{c}{R}$ contains the component
    of $c$. Also, the loop of the algorithm runs at most $k$
    times, where in each iteration, $R$ can only increase
    additively by $3r$. Therefore, the final radius can be
    at most $(3k+1)r \in O(k\sqrt{\tr(\Sigma')}$.
    This proves the second part of the theorem.
\end{proof}

Now, we restate a folklore private algorithm for obtaining
a projection matrix of the top-$k$ subspace of a dataset
via PCA, along with its main result. Here, we just talk about
its utility, but focus on the privacy later when we instantiate
it for different forms of differential privacy.
Given a dataset, it simply truncates
all points to within a given range, computes the empirical
covariance matrix, then adds random noise to each entry of
the matrix that is sampled from a Gaussian distribution
scaled according to the appropriate privacy parameters
depending on the notion of privacy. In the algorithm,
$f_{PCA}(\PrivParams)$ denotes an appropriate function
of the privacy parameters, which we would set according
to the type of differential privacy guarantee we require.

\begin{algorithm}[!ht]
\caption{DP PCA
    $\PrivPCA_{\PrivParams,\ell}(X, r)$}\label{alg:private-PCA}
\KwIn{Private samples $X=(X_1,\dots,X_{n'}) \in \R^{n' \times d}$.
    Radius $r \in \R$.
    Parameters $\PrivParams \subset \R, \ell > 0$.}
\KwOut{Projection matrix $\Pi \in \R^{d \times d}$.}
\vspace{5pt}

Let $f_{PCA}$ be an appropriate function of the elements of
    $\PrivParams$ to ensure DP.\\
Set $\sigma_P \gets 2r^2f_{PCA}(\PrivParams)$.\\
Truncate all points of $X$ to within $\ball{0}{r}$ to get dataset $Y$.\\
Let $E \in \R^{d \times d}$, such that for all
    $i,j \in [d]$ with $i \leq j$, $E_{i,j} \sim \cN(0,\sigma_P^2)$
    and $E_{j,i} \gets E_{i,j}$.\\
Let $\Pi$ be the projection matrix of the top-$\ell$ subspace of $Z \gets Y^TY+E$.
\vspace{5pt}

\Return $\Pi$.
\vspace{5pt}
\end{algorithm}

\begin{thm}[Theorem 9 from \cite{DworkTTZ14}]\label{thm:private-pca}
    Let $X,Y,Z,\ell,c,r,\sigma_P$ be quantities as defined in
    Algorithm~\ref{alg:private-PCA}, $C_{\ell}$ be the
    best rank-$\ell$ approximation to $Y^T Y$, and $\wh{C}_{\ell}$
    be the rank-$\ell$ approximation to $Z$. Then with probability
    at least $1-\beta$,
    $$\|Y^TY - \wh{C}_{\ell}\|_2 \leq \|Y^TY - C_{\ell}\|_2
        + O\left(\sigma_P\sqrt{d} + \sigma_P\sqrt{\ln(1/\beta)}\right).$$
\end{thm}

We get the following corollary using Theorem~\ref{thm:private-pca}
and Lemmata~\ref{lem:PCA_AM_style} and~\ref{lem:data-spectral}.
\begin{cor}\label{coro:private-pca}
    Let $X,Y,Z,\ell,c,r,\sigma_P,\Pi$ be quantities as defined in
    Algorithm~\ref{alg:private-PCA}, such that $X$ is a clean
    subset of $n$ samples from a mixture of $k$ Gaussians $D$
    satisfying Assumption (\ref{eq:gaussian-not-flat}), and
    $X$ contains points from $k'$ components. Suppose $\mw$ is
    the minimum mixing weight of a components with respect to
    $D$. Let $\sigma_{\max}^2$ be the largest directional variance
    of any component that has points in $X$, and $\Sigma$ be the
    covariance of the Gaussian with largest trace that has points
    in $X$. If $r \in O(k\sqrt{\tr(\Sigma)})$, all points in $X$
    lie within $\ball{0}{r}$, and
    $$n \geq O\left(\frac{d}{\mw}
        + d^{1.5}k^2f_{PCA}(\PrivParams)
        + \frac{\ln(k/\beta)}{\mw}\right),$$
    then with probability at least $1-O(\beta)$, for each component
    $(\mu',\Sigma',w')$ that has points in $X$,
    $$\|\mu' - \mu'\Pi\| \leq O\left(\frac{\sigma_{\max}}{\sqrt{w'}}\right).$$
\end{cor}
\begin{proof}
    By construction of our algorithm and our assumption on $r$, we
    know that $O\left(\sigma_P\sqrt{d} + \sigma_P\sqrt{\ln(1/\beta)}\right)
    \in f_{PCA}(\PrivParams)\cdot
    O\left(d^{1.5}k^2 + dk^2\sqrt{\ln(1/\beta)}\right) \in
    f_{PCA}(\PrivParams)\cdot
    O\left(d^{1.5}k^2\right)$ (since $D$ satisfies
    Assumption (\ref{eq:gaussian-not-flat})).
    Using Lemma~\ref{lem:gaussian-sum-conc} and our bound on $n$,
    we know that for each component $(\mu',\Sigma',w')$, the empirical
    mean of all the points from that Gaussian in $X$ would be at most
    $\tfrac{\sqrt{\|\Sigma'\|_2}}{\sqrt{w'}}$ away from $\mu'$. By the same
    reasoning,
    the empirical mean of the same projected points would be at most
    $\tfrac{\sqrt{\|\Sigma'\|_2}}{\sqrt{w'}}$ away from $\mu'\Pi$ because the
    projection
    of a Gaussian is still a Gaussian.
    Lemmata~\ref{lem:PCA_AM_style} and~\ref{lem:data-spectral}, and
    Theorem~\ref{thm:private-pca} together guarantee that the distance
    between the empirical mean of those points and that of the projected
    points of that component would be at most
    $O\left(\tfrac{\sigma_{\max}}{\sqrt{w'}}\right)$. Applying
    the triangle inequality gives us the result.
\end{proof}

Now, we mention a result about the sensitivity of $Y^TY$ in
Algorithm~\ref{alg:private-PCA}, which would be used in later
subsections.
\begin{lem}\label{lem:pca-sensitivity}
    In Algorithm~\ref{alg:private-PCA}, the $\ell_2$ sensitivity
    of $Y^TY$ is $2r^2$.
\end{lem}
\begin{proof}
    Since all points in $X$ are truncated to within $\ball{0}{r}$,
    for all $i$, $\|X_i\| \leq r$. Therefore, for a neighbouring
    dataset $X'$, and corresponding truncated dataset $Y'$,
    $\|Y^TY - Y'^T Y'\|_F \leq 2r^2$.
\end{proof}

Next, we describe an algorithm for partitioning the data
after the PCA step has been performed.
For this, we assume that the data is in low-dimensions
($\ell$-dimensional), where Condition~\ref{eq:gaussian-not-flat}
may not hold, but there is at least one cluster that is well-separated
from the other points, or there is just one cluster in the whole
dataset. We provide a private partitioner for private and public data
for this regime, which in the first
case, returns a partition of the datasets that contains
their clean subsets, and in the second
case, it returns the same cluster itself. For this, we define
the following queries (akin to those in \cite{KamathSSU19}) that
would be used for the public and the private data (respectively)
in the algorithm.
\begin{align}
    \QPub(X,c,r,t) &\coloneqq \left(\abs{X \cap \ball{c}{r}} \geq
        t\right) \wedge\nonumber\\
        &\qquad \left(\abs{X \cap (\ball{c}{11r} \setminus \ball{c}{r})} = 0\right)
        \wedge\nonumber\\
        &\qquad \left(\abs{X \cap (\R^d \setminus \ball{c}{11r})} \geq t\right)
        \label{eq:low-dim-query-public}\\
    \QPriv(X,c,r,t,\PrivParams) &\coloneqq
        \left(\PCount_{\PrivParams}(X \cap \ball{c}{r}) \geq
        t\right) \wedge\nonumber\\
        &\qquad \left(\PCount_{\PrivParams}(X \cap (\ball{c}{5r} \setminus \ball{c}{r})) < \frac{t}{320}\right)
        \wedge\nonumber\\
        &\qquad \left(\PCount_{\PrivParams}(X \cap (\R^d \setminus \ball{c}{5r})) \geq t\right)
        \label{eq:low-dim-query-private}
\end{align}
Just like in \cite{KamathSSU19}, we define the notion of a
\emph{terrific ball}. We say a ball $\ball{c}{r}$ is
$(\gamma,t)$-\emph{terrific} with respect to some dataset $X$
for $\gamma>1,t>0$, if (1) $\ball{c}{r}$ contains at least $t$
points from $X$, (2) $\ball{c}{\gamma r} \setminus \ball{c}{r}$
contains at most $\tfrac{t}{80}$ points from $X$, and (3)
$\R^d \setminus \ball{c}{\gamma r}$ contains at least $t$ points
from $X$. We sometimes omit the parameter $t$ when the context
is clear.

\begin{algorithm}[h]
\caption{DP Partitioner
    $\LowDimPartitioner_{n,m,\mw,\PrivParams}(\wt{Y}, \wt{Z}, r_{\max}, r_{\min})$}\label{alg:low-dim-partitioner}
\KwIn{Private Samples $\wt{Z}_1,\dots,\wt{Z}_{n'} \in \R^d$.
    Public Samples $\wt{Y}_1,\dots,\wt{Y}_{m'} \in \R^d$.
    Parameters $n \geq n', m \geq m', \mw > 0$.
    Privacy Parameters $\PrivParams$.}
\KwOut{A tuple of centre $c \in \R^d$, radius $R \in \R$, or $\bot$.}
\vspace{5pt}

\For{$i \gets 0,\dots,\log(r_{\max}/r_{\min})$}{
    $r_i \gets \tfrac{r_{\max}}{2^i}$\\
    \For{$j \gets 1,\dots,m'$}{
        $c_j \gets \wt{Y}_j$\\
        \If{$\QPub(\wt{Y},c_j,r_i,\tfrac{m\mw}{2}) = \True$}{
            \If{$\QPriv(\wt{Z},c_j,2r_i,\tfrac{n\mw}{4},\PrivParams) = \True$}{
                \Return $(c = c_j, R = 2r_i)$.
            }
            \Else{
                \Return $\bot$.
            }
        }
    }
}
\vspace{5pt}

\Return $\bot$.
\vspace{5pt}
\end{algorithm}

Now, we prove the main theorem about the utility of
Algorithm~\ref{alg:low-dim-partitioner}. In the following,
$f_{\PCount}(\PrivParams)$ denotes an appropriate function
of the privacy parameters based on the accuracy guarantees
of $\PCount$ (Lemma~\ref{lem:pcount})
depending on the notion of DP being used. We prove the privacy
guarantees in later sections.

\begin{thm}\label{thm:low-dim-partitioner}
    Let $Y$ be a clean subset of a set of public samples
    from $D \in \cG(d,k)$, $Z$ be a clean subset of
    private samples from $D$, $\Pi$ be a projection
    matrix to $\ell$ dimensions, $\wt{Y} = Y\Pi$, $\wt{Z} = Z\Pi$
    be the input to
    Algorithm~\ref{alg:low-dim-partitioner}, $r_{\max}, r_{\min} > 0$,
    such that for some $j \in [m']$,
    $\wt{Y} \cap \ball{\wt{Y}_i}{r_{\max}} = \wt{Y}$ and
    $\wt{Z} \cap \ball{\wt{Y}_i}{r_{\max}} = \wt{Z}$
    and $\ball{\wt{Y}_i}{r_{\max}}$ is pure with
    respect to the components of $D$ projected on to the subspace
    of $\Pi$ having points in $\wt{Y}$ (and $\wt{Z}$).
    Suppose $\sigma_{\max}^2$ is the largest directional
    variance among the Gaussians that have points in $Y$.
    Let that Gaussian be $\cN(\mu,\Sigma)$. Suppose,
    $$n \geq O\left(\frac{d\ln(k/\beta)}{\mw} +
        \frac{\ln(1/\beta)}{\mw \cdot f_{\PCount}(\PrivParams)}\right)$$
    and
    $$m \geq O\left(\frac{\ln(k/\beta)}{\mw}\right).$$
    Then we have the following with probability at least
    $1-4\beta$.
    \begin{enumerate}
    \item Suppose there are at least two Gaussians that have
        points in $Y$, and the same Gaussians have points in $Z$.
        For any other Gaussian $\cN(\mu',\Sigma')$
        that has points in $Y$ and $Z$, suppose
        for $N = m+n$, $\|\mu\Pi - \mu'\Pi\| \geq
        \Omega(\sigma_{\max}\sqrt{\ell\ln(N\ell/\beta)})$.
        Let $r_{\min} \leq \sigma_{\max}\sqrt{2\ell\ln(2N\ell/\beta)}$.
        Then the ball returned by Algorithm~\ref{alg:low-dim-partitioner}
        is pure with respect to the components of $D$ projected by $\Pi$,
        which have points in $\wt{Y}$,
        such that $\R^d \setminus \ball{c}{R}$ only contains
        components from $D$ projected by $\Pi$ that have points in
        $\wt{Y}$ (alternatively, $\wt{Z}$), and so
        does $\ball{c}{R}$.
    \item Suppose there is only one Gaussian that has points in
        $Y$ and $Z$. Then the algorithm returns $\bot$.
    \end{enumerate}
\end{thm}
\begin{proof}
    We first prove the first part of the theorem.
    In this case, we know that there exists a ball of some
    radius $r > 0$ and some centre $c \in \wt{Y}$, such that
    $\ball{c}{r}$ has the required properties. We prove the
    existence of such a ball first. Let $c \in \wt{Y}$
    be a point, such that the corresponding Gaussian to $c$
    in $d$ dimensions is $\cN(\mu,\Sigma)$ (with
    $\|\Sigma\| = \sigma_{\max}^2$). For any Gaussian
    $\cN(\mu',\Sigma')$ that has points in $Y$, in the subspace
    of $\Pi$, $\|\mu\Pi - \mu'\Pi\| \geq
    103\sigma_{\max}\sqrt{\ell\ln(N\ell/\beta)}$. Now, we
    know that for any Gaussian $\cN(\mu'',\Sigma'')$,
    in this subspace, for any $y \sim \cN(\mu'',\Sigma'')$
    that lies in $Y$ and any $z \sim \cN(\mu'',\Sigma'')$
    that lies in $Z$,
    $\|y\Pi - \mu''\Pi\|,\|z\Pi - \mu''\Pi\| \leq
    \sqrt{\|\Pi\Sigma''\Pi\|_2}\sqrt{2\ell\ln(2N\ell/\beta)}$
    from Lemma~\ref{lem:intra-gaussian-mean-low}.
    The triangle inequality implies that for any
    $y,z \sim \cN(\mu,\Sigma)$ with $y \in Y$ and $z \in Z$,
    and $a,b \sim \cN(\mu',\Sigma')$ with $a \in Y$ and $b \in Z$,
    $\|y\Pi-a\Pi\|,\|z\Pi-b\Pi\| \geq
    100\sigma_{\max}^2\sqrt{\ell\ln(2N\ell/\beta)}$.
    Given that $\ball{c}{2\sigma_{\max}\sqrt{2\ln(2N\ell/\beta)}}$
    contains all points in $Y\Pi$ and $Z\Pi$, whose
    Gaussian in high dimensions is $\cN(\mu,\Sigma)$,
    and contains the Gaussian in low dimensions,
    this ball satisfies the required properties.
    
    Now, let's say the algorithm returns a ball
    $B = \ball{c_j}{2r_i}$.
    We know that in this case, because the radius $r_i$
    decreases when $i$ increases,
    $r_i \geq 2\sigma_{\max}\sqrt{2\ell\ln(2N\ell/\beta)}$.
    Because $r_i$ is bigger than the largest possible diameter
    of any Gaussian component in the $\ell$-dimensional subspace,
    $\ball{c_j}{r_i}$ must be, at least, containing the component of
    $c_j$ in that subspace. We know that the shell
    $S = \ball{c_j}{11r_i} \setminus \ball{c_j}{r_i}$ has no
    points from any component in $\wt{Y}$ (by construction).
    For any component inside $\ball{c_j}{r_i}$ (either partially
    or completely), the distance between the mean of any
    such component from
    the mean of any other component that is completely outside
    $\ball{c_j}{r_i}$ has to be at least
    $10r_i - 2\times\tfrac{r_i}{2} = 9r_i$
    (because $r_i$ is at least twice the maximum distance
    from the mean of a Gaussian to any of its points in that
    subspace). If there a component, whose mean lies in the shell
    $T = \ball{c_j}{\frac{21r_i}{2}} \setminus \ball{c_j}{\frac{3r_i}{2}}$,
    it would be completely contained inside $S$, and we would
    see points of $\wt{Y}$ in $S$, which would
    be a contradiction. Therefore, the means of all components
    that lie completely outside $\ball{c_j}{r_i}$ must be within
    $\R^d \setminus \ball{c_j}{\frac{21r_i}{2}}$, which means
    that they would be entirely contained within
    $\R^d \setminus \ball{c_j}{10r_i}$ (because of
    Lemma~\ref{lem:intra-gaussian-mean-low}). Therefore,
    $B$ cannot have any component that lies completely outside
    $\ball{c_j}{r_i}$. By similar reasoning, the mean of any
    component that partially intersects with $\ball{c_j}{r_i}$
    has to be within $\ball{c_j}{\frac{3r_i}{2}}$. Therefore,
    by Lemma~\ref{lem:intra-gaussian-mean-low} again, that
    component would be contained within $\ball{c_j}{2r_i}$.
    Therefore, the ball $B$ is pure. By the construction
    of the query $\QPriv$, our sample complexity (of $n$), and
    the guarantees of $\PCount$, with probability
    at least $1-3\beta$, we know that the noise in each call
    to $\PCount$ cannot be more than $\tfrac{n\mw}{1280}$
    (Lemma~\ref{lem:pcount}).
    We also know that the number of points from each component
    is at least $\tfrac{n\mw}{2}$ (from Lemma~\ref{lem:gmm-frequqncy}).
    Since $\ball{c_j}{2r_i}$ is pure, it must have at least
    $\tfrac{n\mw}{2}$ points in it, so the noisy answer via
    $\PCount$ would be at least $\tfrac{n\mw}{4}$. Similarly,
    since there is no Gaussian component even partially contained
    within $\ball{c_j}{10r_i} \setminus \ball{c_j}{2r_i}$,
    the noisy answer will be less than $\tfrac{n\mw}{1280}$.
    Finally, since there is at least one component in
    $\R^d \setminus \ball{c_j}{10r_i}$, the noisy answer of
    $\PCount$ would be over $\tfrac{n\mw}{4}$. So, $\QPriv$
    would return $\True$. This proves the first part of the
    theorem.
    
    Now, we prove the second part of the theorem. It is sufficient
    to show that there exists no $5$-terrific ball in the private
    dataset $\wt{Z}$ when all the data is coming from a single
    Gaussian. First, we claim that if there is a $5$-terrific ball
    in a dataset $X$, then there exists a $3$-terrific ball, too.
    The argument is simple. Suppose $\ball{c}{r}$ is a $5$-terrific
    ball, then it is easy to see that $\ball{c}{r'}$ for $r'=\tfrac{5r}{3}$
    is a $3$-terrific ball. This is because $\ball{c}{r'}$ would contain
    more points that $\ball{c}{r}$, hence, $\ball{c}{3r'} \setminus \ball{c}{r'}$
    would contain fewer points than $\ball{c}{5r} \setminus \ball{c}{r}$.
    Since $\R^d \setminus \ball{c}{5r} = \R^d \setminus \ball{c}{3r'}$,
    the final constraint would hold, as well. Hence, the claim.
    So now, it is sufficient to show that there exists no $3$-terrific
    ball with respect to $\wt{Z}$.
    
    We prove this via contradiction. We know that due to our sample
    complexity (bound on $n$), the noise in each call to $\PCount$
    is at most $\tfrac{n\mw}{1280}$. Therefore, for a query
    $\QPriv(\wt{Z},c',r',\tfrac{n\mw}{4},\PrivParams)$ to return a
    ball $\ball{c'}{r'}$, $\ball{c'}{r'}$ and $\R^d \setminus \ball{c'}{5r'}$
    must contain at least $\tfrac{n\mw}{8}$ points from $\wt{Z}$,
    and $\ball{c'}{5r'} \setminus \ball{c'}{r'}$ must contain at most
    $\tfrac{n\mw}{640}$ points from $\wt{Z}$. This implies the existence
    of a $3$-terrific ball of radius $r$ with respect to $\wt{Z}$.
    This means that there exists a unit
    vector $v$ such that on projecting all the data on to $v$, we would
    have an three adjacent intervals along the vector $I_1,I_2,I_3$ of
    length $2r$ each along the direction of $v$, such that $I_1$ and $I_3$
    together contain at most $\tfrac{n\mw}{640}$ points, $I_2$ contains at
    least $\tfrac{n\mw}{8}$ points, and the rest of the line contains
    at least $\tfrac{n\mw}{8}$ points. Suppose the mixing weight of the
    Gaussian in question is $w_i$. Lemma~\ref{lem:gmm-frequqncy} shows
    that the number of points in $\wt{Z}$ is $Cnw_i$, where
    $\tfrac{1}{2} \leq C \leq \tfrac{3}{2}$. A straightforward application
    of Chernoff bound (Lemma~\ref{lem:chernoff-mult}) and the union bound
    over all $2^{O(d)}$ unit vectors in a cover of the unit sphere in
    $\R^d$, along with our sample complexity, implies that the probability
    mass contained in $I_2$ and the part of line excluding $I_1$, $I_2$,
    and $I_3$ (which we call, $I_4$) is at least $\tfrac{\mw}{8Cw_i}$ each,
    and the mass
    contained in $I_1 \cup I_3$ is strictly less than $\tfrac{\mw}{160Cw_i}$.
    We show that this is impossible when the underlying distribution
    along $v$ is a (one-dimensional) Gaussian.
    
    Suppose the variance along $v$ of the said Gaussian is $\sigma^2$.
    WLOG, we will assume that its mean is $0$. Note that for the said
    condition about the probability mass to be true, $I_2$ must contain
    the mean. This is because of the symmetry about the mean and unimodality
    of one-dimensional Gaussians -- if $I_2$ does not contain the mean,
    then either $I_1$ or $I_3$ would be closer to the mean, and would
    contain more mass than $I_2$ because all three of these intervals
    are of the same length. Next, we show that the total mass contained
    within $I_1$ and $I_3$ is minimised when the mean is at the centre
    of $I_2$. After that, it would be sufficient to show that in this
    configuration, the probability mass contained within $I_1$ and $I_3$
    together is at least $\tfrac{\mw}{160Cw_i}$ when the mass contained
    in $I_2$ and $I_4$ each is at least $\tfrac{\mw}{8Cw_i}$. Let
    $0 \leq r_1 \leq r$ and $r_2 = 2r - r_1$, such that one end point
    of $I_2$ is $r_1$ away from the mean, and the other end point is
    $r_2$ away from the mean. Then the probability mass
    contained within $I_1 \cup I_3$ is
    $$\frac{1}{\sqrt{2\pi\sigma}}\int\limits_{r_1}^{r_1+2r}
        {e^{-\frac{t^2}{2\sigma^2}}dt} +
        \frac{1}{\sqrt{2\pi\sigma}}\int\limits_{r_2}^{r_2+2r}
        {e^{-\frac{t^2}{2\sigma^2}}dt} =
        \frac{1}{\sqrt{2\pi}}\int\limits_{r_1/\sigma}^{(r_1+2r)/\sigma}
        {e^{-\frac{t^2}{2}}dt} +
        \frac{1}{\sqrt{2\pi}}\int\limits_{(2r-r_1)/\sigma}^{(4r-r_1)/\sigma}
        {e^{-\frac{t^2}{2}}dt}.$$
    Define $f(r_1) = \frac{1}{\sqrt{2\pi}}\int\limits_{r_1/\sigma}^{(r_1+2r)/\sigma}
        {e^{-\frac{t^2}{2}}dt} +
        \frac{1}{\sqrt{2\pi}}\int\limits_{(2r-r_1)/\sigma}^{(4r-r_1)/\sigma}
        {e^{-\frac{t^2}{2}}dt}$.
    Then we have the following.
    \begin{align*}
        \frac{df(r_1)}{dr_1} = \frac{1}{\sqrt{2\pi}}\left[
                \frac{1}{\sigma}\cdot e^{-\frac{(r_1+2r)^2}{2\sigma^2}} -
                \frac{1}{\sigma}\cdot e^{-\frac{r_1^2}{2\sigma^2}} -
                \frac{1}{\sigma}\cdot e^{-\frac{(4r-r_1)^2}{2\sigma^2}} +
                \frac{1}{\sigma}\cdot e^{-\frac{(2r-r_1)^2}{2\sigma^2}}\right]
    \end{align*}
    Within the interval $[0,r]$, setting the above to $0$, we have
    $r_1 = r$. It can be checked that for all $r_1 \in [0,r]$,
    $\tfrac{d^2f(r_1)}{dr_1^2} > 0$. This means that within $[0,r]$,
    $f$ is minimised at $r_1 = r$. Therefore, for the rest of the
    proof, we would concentrate on lower bounding the quantity
    $q(r) = \frac{2}{\sqrt{2\pi}}\int\limits_{r/\sigma}^{3r/\sigma}
    {e^{-\frac{t^2}{2}}dt}$. We define
    $p(r) = \frac{2}{\sqrt{2\pi}}\int\limits_{0}^{r/\sigma}
    {e^{-\frac{t^2}{2}}dt}$ to be the mass of $I_2$ in this regime, and
    $\lambda(r) = \frac{2}{\sqrt{2\pi}}\int\limits_{3r/\sigma}^{\infty}
    {e^{-\frac{t^2}{2}}dt}$ to be the mass of $I_4$. We consider three
    cases.
    
    \noindent \paragraph{Case 1: $r^2 = \sigma^2 + \tau$,
    $0 \leq \tau \leq \sigma^2$.} Noting that $e^{-x + 0.25} \geq e^{-x^2}$
    for $x \in [0,1]$ and
    $\tfrac{1}{\sqrt{2}} \leq \tfrac{r}{\sqrt{2}\sigma} \leq 1$,
    we have the following.
    \begin{align*}
        \frac{\mw}{8Cw_i} \leq
        p(r) = \frac{2}{\sqrt{\pi}}\int\limits_{0}^{r/\sqrt{2}\sigma}
                {e^{-s^2}ds} \leq
            \frac{2e^{\frac{1}{4}}}{\sqrt{\pi}}\int\limits_{0}^{r/\sqrt{2}\sigma}
                {e^{-s}ds} =
            \frac{2e^{\frac{1}{4}}}{\sqrt{\pi}}\cdot
                \left(1 - e^{-\frac{r}{\sqrt{2}\sigma}}\right)
    \end{align*}
    Now, we compute $q(r)$.
    \begin{align*}
        q(r) &= \frac{2}{\sqrt{\pi}}\int\limits_{r/\sqrt{2}\sigma}^
                {3r/\sqrt{2}\sigma}{e^{-s^2}ds}\\
            &= \frac{2}{\sqrt{\pi}}\left[\int\limits_{r/\sqrt{2}\sigma}^
                {1}{e^{-s^2}ds} +
                \int\limits_{1}^{3/\sqrt{2}}{e^{-s^2}ds} +
                \int\limits_{3/\sqrt{2}}^{3r/\sqrt{2}\sigma}{e^{-s^2}ds}
                \right]\\
            &\geq \frac{2}{\sqrt{\pi}}\left[\int\limits_{r/\sqrt{2}\sigma}^
                {1}{e^{-s}ds} +
                \int\limits_{1}^{3/\sqrt{2}}{e^{-s^2}ds} +
                \int\limits_{3/\sqrt{2}}^{3r/\sqrt{2}\sigma}
                {e^{-\frac{3rs}{\sqrt{2}\sigma}}ds}
                \right]\\
            &= \frac{2}{\sqrt{\pi}}\left[
                e^{-\frac{r}{\sqrt{2}\sigma}} - \frac{1}{e} +
                \int\limits_{1}^{3/\sqrt{2}}{e^{-s^2}ds} +
                \frac{\sqrt{2}\sigma}{3r}\cdot e^{-\frac{9r}{2\sigma}} -
                \frac{\sqrt{2}\sigma}{3r}\cdot e^{-\frac{9r^2}{2\sigma^2}}
                \right]\\
            &\geq \frac{2}{\sqrt{pi}}\left[
                \int\limits_{1}^{3/\sqrt{2}}{e^{-s^2}ds} - \frac{1}{e} +
                e^{-\frac{r}{\sqrt{2}\sigma}} +
                \frac{1}{3}\cdot e^{-\frac{9r}{2\sigma}} -
                \frac{1}{3}\cdot e^{-\frac{9r^2}{2\sigma^2}}
                \right]\\
            &> \frac{2e^{\frac{1}{4}}}{20\sqrt{\pi}}\cdot
                \left(1 - e^{-\frac{r}{\sqrt{2}\sigma}}\right)\\
            &\geq \frac{\mw}{160Cw_i}
    \end{align*}
    We have a contradiction. Thus, in this case, the probability
    mass in $I_1 \cup I_3$ has to be at least $\tfrac{\mw}{160Cw_i}$.
    
    \paragraph{Case 2: $r^2 = \sigma^2 + \tau$, $\tau \geq \sigma^2$.}
    We have the following upper bound on $\lambda(r)$ using
    Lemma~\ref{lem:gauss-conc-1d}.
    $$\lambda(r) = \frac{2}{\sqrt{2\pi}}\int\limits_{3r/\sigma}^{\infty}
        {e^{-\frac{t^2}{2}}dt} \leq 2e^{-\frac{9r^2}{2\sigma^2}}$$
    We know that $\lambda(r) \geq \tfrac{\mw}{8Cw_i}$. This
    implies that $\tfrac{r^2}{\sigma^2} \leq
    \tfrac{2}{9}\ln\left(\tfrac{16Cw_i}{\mw}\right)$.
    This is also equivalent to $\tfrac{\mw}{16Cw_i} \leq \tfrac{1}{e^9}$.
    Lemma~\ref{lem:chi-squared} gives us the following lower bound
    on $1-p(r)$.
    $$1 - p(r) = \frac{2}{\sqrt{2\pi}}\int\limits_{r/\sigma}^{\infty}
        {e^{-\frac{t^2}{2}}dt}
        \geq 0.06e^{-\frac{3r^2}{2\sigma^2} + \frac{3}{2}}
        = (0.06)e^{\frac{3}{2}}e^{-\frac{3r^2}{2\sigma^2}}$$
    This gives us the following lower bound on $q(r)$.
    \begin{align*}
        q(r) &= \frac{2}{\sqrt{2\pi}}\int\limits_{r/\sigma}^{\infty}
                {e^{-\frac{t^2}{2}}dt} -
                \frac{2}{\sqrt{2\pi}}\int\limits_{3r/\sigma}^{\infty}
                {e^{-\frac{t^2}{2}}dt}\\
            &\geq (0.06)e^{\frac{3}{2}}e^{-\frac{3r^2}{2}} -
                2e^{-\frac{9r^2}{2\sigma^2}}\\
            &\geq (0.03)e^{\frac{3}{2}}e^{-\frac{3r^2}{2}}
                \tag{For $r^2 \geq 2\sigma^2$.}\\
            &\geq (0.03)e^{\frac{3}{2}}e^{-\frac{3}{2}\cdot\frac{2}{9}
                \ln\left(\frac{16Cw_i}{\mw}\right)}\\
            &= (0.03)e^{\frac{3}{2}}\cdot
                \left(\frac{\mw}{16Cw_i}\right)^{\frac{1}{3}}\\
            &\geq \frac{\mw}{160Cw_i}
                \tag{In the regime $\tfrac{\mw}{16Cw_i} \leq \tfrac{1}{e^9}$.}
    \end{align*}
    This gives us a contradiction for the second case, as well.
    
    \paragraph{Case 3: $r^2 = \sigma^2 - \tau$, $0 \leq \tau < \sigma^2$.}
    This final case has two sub-cases to analyse: $9r^2 \leq 2\sigma^2$ and
    $9r^2 > 2\sigma^2$. In the first instance, we have the following
    bounds on $p(r)$.
    $$\frac{\mw}{8Cw_i} \leq p(r) =
        \frac{2}{\sqrt{2\pi}}\int\limits_{0}^{r/\sigma}
        {e^{-\frac{t^2}{2}}dt} \leq \frac{2r}{\sqrt{2\pi}\sigma}$$
    Next, we lower bound $p(3r)$.
    $$p(3r) = \frac{2}{\sqrt{2\pi}}\int\limits_{0}^{3r/\sigma}
        {e^{-\frac{t^2}{2}}dt} \geq \frac{2}{\sqrt{2\pi}}\cdot\frac{3r}{\sigma}
        \cdot e^{-\frac{9r^2}{2\sigma^2}} \geq \frac{6r}{e\sqrt{2\pi}\sigma}$$
    This gives us the following lower bound on $q(r)$.
    \begin{align*}
        q(r) &= \frac{2}{\sqrt{2\pi}}\int\limits_{0}^{3r/\sigma}
                {e^{-\frac{t^2}{2}}dt} -
                \frac{2}{\sqrt{2\pi}}\int\limits_{0}^{r/\sigma}
                {e^{-\frac{t^2}{2}}dt}\\
            &\geq \frac{6r}{e\sqrt{2\pi}\sigma} - \frac{2r}{\sqrt{2\pi}\sigma}\\
            &> \frac{r}{5\sqrt{2\pi}\sigma}\\
            &> \frac{\mw}{160Cw_i}
    \end{align*}
    This shows that this situation is impossible. So, we move on
    to the final instance, that is, when $\tfrac{2\sigma^2}{9} < r^2 < \sigma^2$.
    Just as we did in Case 1, we bound $p(r)$. We get the following.
    $$\frac{\mw}{8Cw_i} \leq p(r) \leq
        \frac{2e^{\frac{1}{4}}}{\sqrt{\pi}}\cdot
        \left(1 - e^{-\frac{r}{\sqrt{2}\sigma}}\right)$$
    We lower bound $p(3r)$ again.
    \begin{align*}
        p(3r) &= \frac{2}{\sqrt{\pi}}\int\limits_{0}^{1}{e^{-s^2}ds} +
                \frac{2}{\sqrt{\pi}}\int\limits_{1}^{3r/\sqrt{2}\sigma}
                {e^{-s^2}ds}\\
            &\geq \frac{2}{\sqrt{\pi}}\int\limits_{0}^{1}{e^{-s^2}ds} +
                \frac{2}{\sqrt{\pi}}\int\limits_{0}^{3r/\sqrt{2}\sigma}
                {e^{-\frac{3rs}{\sqrt{2}\sigma}}ds}\\
            &= \frac{2}{\sqrt{\pi}}\int\limits_{0}^{1}{e^{-s^2}ds} +
                \frac{2}{\sqrt{\pi}}\cdot\frac{\sqrt{2}\sigma}{3r}
                \left[e^{-\frac{3r}{\sqrt{2}\sigma}} -
                e^{-\frac{9r^2}{2\sigma^2}}\right]\\
            &\geq \frac{2}{\sqrt{\pi}}\int\limits_{0}^{1}{e^{-s^2}ds} +
                \frac{2\sqrt{2}}{3\sqrt{\pi}}
                \left[e^{-\frac{3r}{\sqrt{2}\sigma}} -
                e^{-\frac{9r^2}{2\sigma^2}}\right]
    \end{align*}
    Now, we lower bound $q(r)$.
    \begin{align*}
        q(r) &= \frac{2}{\sqrt{2\pi}}\int\limits_{0}^{3r/\sigma}
                {e^{-\frac{t^2}{2}}dt} -
                \frac{2}{\sqrt{2\pi}}\int\limits_{0}^{r/\sigma}
                {e^{-\frac{t^2}{2}}dt}\\
            &\geq \frac{2}{\sqrt{\pi}}\int\limits_{0}^{1}{e^{-s^2}ds} +
                \frac{2\sqrt{2}}{3\sqrt{\pi}}
                \left[e^{-\frac{3r}{\sqrt{2}\sigma}} -
                e^{-\frac{9r^2}{2\sigma^2}}\right] -
                \frac{2e^{\frac{1}{4}}}{\sqrt{\pi}}\cdot
                \left(1 - e^{-\frac{r}{\sqrt{2}\sigma}}\right)\\
            &\geq \frac{1}{20} \cdot
                \frac{2e^{\frac{1}{4}}}{\sqrt{\pi}}\cdot
                \left(1 - e^{-\frac{r}{\sqrt{2}\sigma}}\right)
                \tag{For $\tfrac{2\sigma^2}{9} \leq r^2 \leq \sigma^2$.}\\
            &\geq \frac{\mw}{160Cw_i}
    \end{align*}
    This gives us the final contradiction, hence, completing the
    proof for Case 3.
    
    Therefore, there cannot be such intervals on the line
    of $v$. This means that there cannot be any such direction,
    where such intervals exist, which means that there cannot
    be a $3$-terrific ball with respect to $\wt{Z}$, implying
    that the algorithm could not have found a $5$-terrific ball
    with respect to $\wt{Z}$. By the construction of the algorithm,
    the output would be $\bot$. This completes the proof of
    the theorem.
\end{proof}

We now state the main private algorithm for clustering
points from a Gaussian mixture (Algorithm~\ref{alg:gmm-clustering-hard})
that utilises the public data available to it, along with
a theorem describing its utility guarantees.
In the theorem, $f_{PCA}'$ and $f_{\PCount}'$ denote
appropriate functions of the privacy parameters
in terms of the previously mentioned functions $f_{PCA}$ and $f_{\PCount}$,
respectively. Note that the output of the algorithm itself is not
private, but the intermediate steps involved are. Since
its output will not be released to public, and will just
be utilised by our main private estimation algorithm, there
will be no violation of privacy in the end.
This algorithm is a general private framework, which could
be instantiated for approximate DP and zCDP by choosing
the right parameters and privacy primitives, for example,
it would use approximate DP PCA algorithm for the former, and
zCDP PCA algorithm for the latter. It first uses
Algorithm~\ref{alg:supercluster} on the public data to
find a supercluster that tightly bounds a group of clusters.
Since the private data is sampled from the same distribution
as the public data, it uses that supercluster to isolate the
points in the private dataset corresponding to the components
contained within the supercluster. Next, it uses private PCA
to project the data on to the top-$k$ subspace, and then
works within that subspace to
either determine if the data is coming from a single Gaussian,
or to partition the dataset into clean subsets in case it is
coming from at least two Gaussians. In the former case, it
adds the set of private points in the original $d$-dimensional
space to the output (which contains sets of points). In the latter
case, it just proceeds to work on the partitions independently.

\begin{algorithm}[!ht]
\caption{DP GMM Hard Clustering
    $\DPHC_{\PrivParams, \beta, k, \mw}(\wt{X}, X)$}\label{alg:gmm-clustering-hard}
\KwIn{Private samples $X=(X_1,\dots,X_n) \in \R^{n \times d}$.
    Public samples $\wt{X} = (\wt{X}_1,\dots,\wt{X}_{m})
        \in \R^{m \times d}$.
    Parameters $\PrivParams \subset \R, \beta, k, \mw > 0$.}
\KwOut{Set $C \subset \cP(X)$.}
\vspace{5pt}

Let $Y \gets \wt{X}$ and $Z \gets X$.\\
Let $C \gets \emptyset$.\\
Let $Q_{\Priv}$ and $Q_{\Pub}$ be queues of sets of points.\\
Add set $Z$ to $Q_{\Priv}$, and set $Y$ to $Q_{\Pub}$.\\
$\PrivParams'$ be the set of privacy parameters modified
    as per composition based on $\PrivParams$.\\
Set $count \gets 0$ and $i \gets 1$.
\vspace{5pt}

\While{$count < k$ and $i \leq 2k$}{
    Pop $Q_{\Priv}$ to get $Z$, and $Q_{\Pub}$ to get $Y$.\\
    \vspace{5pt}
    \tcp{Run superclustering algorithm on the public dataset.}
    Set $(c_i,R_i) \gets \SC_{k,m}(Y)$.\\
    \vspace{5pt}
    \tcp{Partition both datasets on the basis of the supercluster.}
    $Y^i \gets Y \cap \ball{c_i}{R_i}$ and
        $Z^i \gets Z \cap \ball{c_i}{R_i}$.\\
    Add $Z \setminus Z^i$ to $Q_{\Priv}$, and $Y \setminus Y^i$ to $Q_{\Pub}$.\\
    \vspace{5pt}
    \tcp{Private PCA: project points of both datasets onto
        returned subspace.}
    Let $M^i \in \R^{|Z^i| \times d}$, such that each row $M_j^i \gets c_i$.\\
    $\Pi_i \gets \PrivPCA_{\PrivParams',k}(Z^i-M^i,R_i)$.\\
    $Y' \gets (Y^i-M^i)\Pi_i$ and $Z' \gets (Z^i-M^i)\Pi_i$.\\
    \vspace{5pt}
    \tcp{If there is only one Gaussian, add it to $C$, otherwise further partition
        the datasets.}
    $B \gets \LowDimPartitioner_{n,m,\mw,\PrivParams'}
        (Z', Y', R_i, \tfrac{R_i}{\sqrt{d}})$.\\
    \If{$B = \bot$}{
        $C \gets C \cup \{Z^i\}$.\\
        $count \gets count + 1$.
    }
    \Else{
        $B$ is an ordered pair $(c_i',r_i')$.\\
        $S' \gets \ball{c'}{r_i'} \cap Y'$ and
            $T' \gets \ball{c'}{r_i'} \cap Z'$.\\
        Let $S$ be points in $Y^i$ corresponding to $S'$,
            and $T$ be points in $Z^i$ corresponding to $T'$.\\
        Add $S$ to $Q_{\Pub}$ and $T$ to $Q_{\Priv}$.\\
        Add $Y^i \setminus S$ to $Q_{\Pub}$ and $Z^i \setminus T$ to $Q_{\Priv}$.
    }
    \vspace{5pt}
    $i \gets i + 1$
}
\vspace{5pt}

\Return $C$.
\vspace{5pt}
\end{algorithm}

\begin{thm}\label{thm:gmm-clustering-hard}
    There exists an algorithm (Algorithm~\ref{alg:gmm-clustering-hard})
    that takes $n$ private samples $X$, and $m$ public samples from
    $D \in \cG(d,k,s)$ satisfying Assumption (\ref{eq:gaussian-not-flat}),
    and outputs a partition $C$ of $X$, such that
    given Conditions \ref{cond:gmm-frequency} to \ref{cond:intra-gaussian-points-low}
    hold, and
    $$s \geq \Omega(\sqrt{k\ln((n+m)k/\beta)}),$$
    $$n \geq O\left(\frac{d\ln(k/\beta)}{\mw} + d^{1.5}k^{2.5}f_{PCA}'(\PrivParams)
        + \frac{\sqrt{k}\ln(k/\beta)}{\mw \cdot f_{\PCount}'(\PrivParams)}\right),$$
    and
    $$m \geq O\left(\frac{\ln(k/\beta)}{\mw}\right),$$
    then with probability at least $1-O(\beta)$, $\abs{C} = k$ and
    each $S \in C$ is clean and non-empty.
\end{thm}
\begin{proof}
    It is enough to show that $\forall i \geq 0$, at the end of the
    $i$-th iteration, (1) a new Gaussian component is isolated in $C$
    and $count$ equals the number of isolated components in $C$,
    (2) or the private dataset $X$ is further partitioned into non-empty,
    clean subsets, and the public dataset $\wt{X}$ is further partitioned
    into clean subsets, such that $\abs{Q_{\Pub}} = \abs{Q_{\Priv}}$ and
    for each $j \in [\abs{Q_{\Pub}}]$, $Q_{\Pub}[j]$ and $Q_{\Priv}[j]$
    have points from the same Gaussian components. We can prove this via
    induction on $i$.
    
    \noindent \textbf{Base Case:} The end of the $0$-th iteration
    essentially means the actual start of the loop. In this case,
    we know that $Q_{\Pub}$ and $Q_{\Priv}$ are clean subsets of
    $\wt{X}$ and $X$, respectively. Therefore, the claim trivially
    holds in this case.
    
    \noindent \textbf{Inductive Step:} We assume for all iterations,
    up to and including some $i \geq 0$, the claim holds. Then we show
    that it holds for iteration $i+1$, as well. By the inductive hypothesis,
    we know that $Z$ and $Y$ are clean subsets, and contain points from
    the same components. When the superclustering
    algorithm (Algorithm~\ref{alg:supercluster}) is called, it finds
    a pure ball that will either contain the whole dataset $Y$ (hence, $Z$
    because the ball is pure), or it partitions $Y$
    into clean subsets (Theorem~\ref{thm:supercluster}). Since the
    components from which the points
    in $Y$ and $Z$ come from are the same, $\ball{c_i}{R_i}$ would
    contain points from the same components from $Z$ as in $Y$ (because
    the deterministic regularity conditions hold). So, adding
    $Z \setminus Z^i$ to $Q_{\Priv}$ and $Y \setminus Y^i$ to $Q_{\Pub}$
    ensures that the partition of $Z$ added to the back of $Q_{\Priv}$
    is clean and contains the points from the same components
    as those that have points in $Y \setminus Y^i$ (which itself
    gets added to the back of $Q_{\Pub}$), which preserves the
    ordering of the set of components in the two queues.
    
    If there is just one component that has points in $Y^i$
    (hence, in $Z^i$), then after projecting the data points
    on to the subspace of $\Pi_i$, and feeding the projected datasets
    to Algorithm~\ref{alg:low-dim-partitioner}, we will get $\bot$
    (by the guarantees from Theorem~\ref{thm:low-dim-partitioner})
    with probability at least $1-\tfrac{\beta}{2k}$. Because $Z^i$
    is a clean subset, $C$ now contains one more clean subset of $X$,
    and we have isolated a new component, and updated the value
    of $count$ to reflect the change.
    
    Suppose there are at least two components that have points in
    $Y^i$ (hence, in $Z^i$). Then we know from Theorem~\ref{thm:supercluster},
    that if $\Sigma$ is the covariance having the largest trace among all
    components that have points in $Y^i$ (hence, in $Z^i$), then
    $R_i \in O(k\sqrt{\tr(\Sigma)})$. Let the largest directional variance
    among the Gaussians that have points in $Z^i$ (hence, in $Y^i$) be
    $\sigma_{\max}^2$, its mean be $\mu$, and its mixing weight be $w$.
    From the guarantees in Corollary~\ref{coro:private-pca}, we know
    that $\|\mu - \mu\Pi_i\| \leq O\left(\tfrac{\sigma_{\max}}{\sqrt{w}}\right)$
    with probability at least $1-\tfrac{\beta}{2k}$.
    Because of the separation condition, applying the triangle inequality,
    we know that the distance between $\mu\Pi_i$ and the projected mean
    of any other Gaussian having points in $Z^i\Pi$ is at least
    $\Omega(\sigma_{\max}\sqrt{k\ln((n+m)k)/\beta})$. Then from the
    guarantees of Theorem~\ref{thm:low-dim-partitioner}, we have
    that a ball is returned, and it is pure with respect to the
    components of the mixture projected by $\Pi_i$ that have points
    in $Z^i\Pi$. Therefore, subsets $S'$ and $T'$ are clean, and
    so are $S$ and $T$. This implies that $Y^i \setminus S$ and
    $Z^i \setminus T$ are clean, too. By adding them in order to
    $Q_{\Pub}$ and $Q_{\Priv}$ respectively, we partition the
    datasets $\wt{X}$ and $X$ further into clean subsets, and
    preserve the ordering of the components in the respective queues.
    This proves the claim.
    
    Note that before the PCA step, we recentre the points in both
    $Y^i$ and $Z^i$ for the ease of analysis later when we provide
    guarantees for privacy. It does not affect correctness of any
    step of the algorithm because it is just recentering of data.
    
    Now, we just have to show that the number of iterations we
    allow is enough with high probability. Because we have $k$
    components in the mixture, we can only partition into clean
    subsets at most $k-1$ times. For each component, the loop
    will run at most one time. So, $2k-1$ iterations are enough
    to capture all the components, and $\abs{C} = k$.
    Taking the union bound over all possible $2k-1$ iterations,
    we have the required result.
\end{proof}

We finally state the private algorithm (Algorithm~\ref{alg:gmm-estimator-hard})
for estimating
the parameters of mixtures of Gaussians, along with a theorem
highlighting its utility guarantees. Just like the DP clustering
algorithm above, it is a private framework that could instantiated
for approximate DP and zCDP settings by using the corresponding
private algorithms. It calls Algorithm~\ref{alg:gmm-clustering-hard}
to cluster the points according to their respective components,
then calls private estimators to estimate the parameters of the
respective components. It finally uses $\PCount$
to estimate their mixing weights based on the private data.

\begin{algorithm}[!ht]
\caption{DP GMM Hard Estimator
    $\DPHE_{\PrivParams, \alpha, \beta, k, \mw}(\wt{X}, X)$}\label{alg:gmm-estimator-hard}
\KwIn{Private samples $X=(X_1,\dots,X_n) \in \R^{n \times d}$.
    Public samples $\wt{X} = (\wt{X}_1,\dots,\wt{X}_{m})
        \in \R^{m \times d}$.
    Parameters $\PrivParams \subset \R, \beta, k, \mw > 0$.}
\KwOut{Set $\wh{D} = \{(\wh{\mu}_1,\wh{\Sigma}_1,\wh{w}_1),
    \dots,(\wh{\mu}_k,\wh{\Sigma}_k,\wh{w}_k)\}$.}
\vspace{5pt}

\tcp{Clustering.}
Set $C \gets \DPHC_{\PrivParams, \beta, k, \mw}(X, \wt{X})$.
\vspace{5pt}

\tcp{Parameter estimation.}
$\wh{D} \gets \emptyset$.\\
\For{$i \gets 1, \dots, k$}{
    Set $(\wh{\mu}_i,\wh{\Sigma}_i) \gets
        \DPGE_{\PrivParams, \frac{\beta}{k}}(C_i)$.\\
    Set $\wh{w}_i \gets
        \max\left\{\tfrac{1}{n}\cdot\PCount_{\PrivParams}(\abs{C_i}),
        \tfrac{\alpha}{2k}\right\}$.\\
    $\wh{D} \gets \wh{D} \cup \{(\wh{\mu}_i,\wh{\Sigma}_i,\wh{w}_i)\}$.
}
\vspace{5pt}

\Return $\wh{D}$.
\vspace{5pt}
\end{algorithm}

\begin{thm}\label{thm:gmm-hard}
    For all $\alpha,\beta>0$ and sets of privacy parameters
    $\PrivParams$, there exists an
    algorithm (Algorithm~\ref{alg:gmm-estimator-hard}) that takes
    $n$ private samples
    and $m$ public samples from $D \in \cG(d,k,s)$, such that if
    $D=\{(\mu_1,\Sigma_1,w_1),\dots,(\mu_k,\Sigma_k,w_k)\}$
    satisfies Assumption (\ref{eq:gaussian-not-flat}), and
    \begin{align*}
        s & = \Omega(\sqrt{k\ln((n+m)k/\beta)})\\
        n &= O\left(\frac{n_{GE}\ln(k/\beta)}{\mw} +
        \frac{d\ln(k/\beta)}{\mw} + d^{1.5}k^{2.5}f_{PCA}'(\PrivParams) + \frac{\sqrt{k}\ln(k/\beta)}{\mw \cdot f_{\PCount}'(\PrivParams)}\right)\\
        m &= O\left(\frac{\ln(k/\beta)}{\mw}\right),
    \end{align*}
    where $n_{GE}$ is the sample complexity of privately
    learning a Gaussian using Lemma~\ref{lem:gaussian-estimator}
    according to the type of DP required, then it $(\alpha,\beta)$-learns $D$.
\end{thm}
\begin{proof}
    By our sample complexity bound and Theorem~\ref{thm:gmm-clustering-hard},
    with probability at least $1-\tfrac{\beta}{3}$, for each
    $i \in [k]$, $C_i$ is clean with respect to $X$.
    Therefore, each $C_i$ can be used to learn an independent
    component.
    
    Next, by the guarantees of Lemma~\ref{lem:gaussian-estimator},
    for each $i \in [k]$,
    $\SD(\cN(\wh{\mu}_i,\wh{\Sigma}_i),\cN(\mu_i,\Sigma_i)) \leq \alpha$
    with probability at least $1-\tfrac{\beta}{3k}$.
    
    Finally, using Lemmata~\ref{lem:gmm-frequqncy} and~\ref{lem:pcount},
    along with our sample complexity, we know that the error due to
    each call to $\PCount$ is at most $\tfrac{\alpha}{2k}$ with probability
    at least $1-\tfrac{\beta}{3k}$, and that the error due to sampling
    is at most $\tfrac{\alpha}{2k}$. We apply the triangle inequality
    to get the final error of $\tfrac{\alpha}{k}$ for each mixing weight.
    
    Applying the union bound over all failure events (including the failures
    of Conditions~\ref{cond:gmm-frequency} to~\ref{cond:intra-gaussian-points-low}),
    we have the desired result.
\end{proof}

\subsubsection{\texorpdfstring{$(\eps,\delta)$}{}-DP Algorithm}

Here, we describe our results under approximate DP constraints, and
instantiate Algorithms~\ref{alg:private-PCA},~\ref{alg:gmm-clustering-hard},
and~\ref{alg:gmm-estimator-hard} for this version of DP. We will
not restate the entire algorithms, but just describe how $\PrivParams$
and the sample complexity in these different algorithms change.
We now state the main theorem of our section.

\begin{thm}\label{thm:gmm-hard-approx}
    For all $\alpha,\beta,\eps,\delta>0$, there exists an
    $(O(\eps),O(\delta))$-DP algorithm $\cM$ that takes $n$ private samples
    and $m$ public samples from $D \in \cG(d,k,s)$, and is
    private with respect to the private samples, such that if
    $D=\{(\mu_1,\Sigma_1,w_1),\dots,(\mu_k,\Sigma_k,w_k)\}$
    satisfies Assumption (\ref{eq:gaussian-not-flat}), and
    \begin{align*}
    s &= \Omega(\sqrt{k\ln((n+m)k/\beta)}) \\
    n &= O\left(\frac{d^2\ln\left(\frac{k}{\beta}\right)}{\mw\alpha^2} +
    \frac{(d^2\log(k/\delta) +
    d\log^{1.5}(k/\delta))\cdot\polylog\left(d,\frac{k}{\beta},
    \frac{1}{\alpha},\frac{\sqrt{k}}{\eps},
    \ln\left(\frac{k}{\delta}\right)\right)} {\mw\alpha\eps}\right.  \\
    &\hspace{2.58cm}\left.+\frac{d^{1.5}k^{2.5}\log\left(\frac{k}{\delta}\right)}{\eps}\right) \\
    m &= O\left(\frac{\log(k/\beta)}{\mw}\right),
    \end{align*}
    then $\cM$ $(\alpha,\beta)$-learns $D$.
\end{thm}
\begin{proof}
    We mainly focus on the privacy guarantees here because the accuracy
    would follow from Theorem~\ref{thm:gmm-hard} after setting the
    parameters appropriately. Note that the privacy parameters in
    this case in the set $\PrivParams$ are $\eps,\delta$. As per the
    advanced composition guarantees of DP (Lemma~\ref{lem:composition}),
    we set the privacy parameters in $\PrivParams'$ in
    Algorithm~\ref{alg:gmm-clustering-hard}
    to be $\eps' = \tfrac{\eps}{\sqrt{12k\log(1/\delta)}}$ and
    $\delta' = \tfrac{\delta}{4k}$. Define $\eps_0 = \sqrt{k}\eps'$.
    
    Now, we show that the intermediate steps in the call to
    Algorithm~\ref{alg:gmm-clustering-hard} yield $(\eps,\delta)$-DP.
    We don't release the output of that step itself, but use it in
    subsequent steps. The clusters are formed on the basis of intersection
    of the privately formed balls with the private dataset, and are
    then used subsequently in other private algorithms. So, the final
    algorithm would be private, as well. Thus, it is enough to prove the
    privacy guarantees of those intermediate steps that yield the
    said balls.
    
    In Algorithm~\ref{alg:gmm-clustering-hard}, we start working on
    the private dataset $X$ inside the loop, specifically, when we
    use private PCA. In each iteration $i$, we ensure that each
    row of the input to Algorithm~\ref{alg:private-PCA} ($Z^i-M^i$)
    has norm at most $R_i$ by virtue of selecting points that lie
    in a ball of radius $R_i$, and by recentering the points appropriately.
    Then by Lemma~\ref{lem:pca-sensitivity}, the sensitivity of
    $Y^TY$ in Algorithm~\ref{alg:private-PCA} is at most $2R_i^2$.
    Therefore, adding Gaussian noise calibrated to $\eps',\delta'$ to $Y^TY$,
    that is, setting $f_{PCA}(\eps',\delta') =
    \tfrac{\sqrt{2\ln(2/\delta')}}{\eps'}$,
    implying that $\sigma_P = \tfrac{2R_i^2\sqrt{2\ln(2/\delta')}}{\eps'}$
    is enough to ensure $(\eps',\delta')$-DP in this step
    (Lemmata~\ref{lem:gaussiandp} and~\ref{lem:post-processing}).
    The next step in Algorithm~\ref{alg:gmm-clustering-hard} that
    works with the private data is at the call to
    Algorithm~\ref{alg:low-dim-partitioner}. In this algorithm, only
    one call to $\PCount$ is performed, and this is the only time
    it works with the private dataset. Therefore, by the guarantees
    of Lemma~\ref{lem:pcount}, we know that this step is $\eps'$-DP.
    The next steps are either partitioning the private data or
    adding the isolated private data to $C$. Since neither of those
    sets are being released to public, there is no privacy loss here.
    Therefore, each iteration is $(2\eps',\delta')$-DP (by
    Lemma~\ref{lem:composition}).
    Applying composition over all $2k$ iterations, we have that
    all operations together in the entire run of the loop in the
    algorithm are $(2\eps,\delta)$-DP.
    
    In Theorem~\ref{thm:gmm-clustering-hard}, setting
    $f_{PCA}'(\PrivParams) = f_{PCA}(\eps_0,\delta_0) \in
    O\left(\tfrac{\log(k/\delta)}{\eps}\right)$ give us the right
    sample complexity for accuracy for all calls to
    Algorithm~\ref{alg:private-PCA} (by Lemma~\ref{lem:composition}),
    since we already multiply the $\sqrt{k}$ factor in the numerator.
    Similarly, setting
    $f_{\PCount}'(\PrivParams) = f_{\PCount}(\eps_0) =
    \tfrac{\eps}{\sqrt{12\log(1/\delta)}}$
    in Theorem~\ref{thm:gmm-clustering-hard} gives us the right
    sample complexity for all calls to Algorithm~\ref{alg:low-dim-partitioner}.
    
    Now, each call to the approximate-DP Gaussian learner
    (Lemma~\ref{lem:gaussian-estimator}) is on a disjoint part
    of the private dataset $X$. Therefore, all the $k$ calls
    together are $(\eps,\delta)$-DP because changing one point
    in $X$ can change one point in only one of the clusters.
    By the same reasoning, all calls to $\PCount$ together are
    $(\eps,\delta)$-DP.
    
    As far as the accuracy goes, the first two terms in the
    sample complexity ensures that enough points go to each
    call to the Gaussian estimator from Lemma~\ref{lem:gaussian-estimator}
    and to each call to $\PCount$ because of Lemma~\ref{lem:gmm-frequqncy}.
    Therefore, our GMM estimator $(\alpha,\beta)$-learns the
    mixture.
\end{proof}

\subsubsection{\texorpdfstring{$\rho$}{}-zCDP Algorithm}

Here, we state our results under zCDP, and instantiate
Algorithms~\ref{alg:private-PCA},~\ref{alg:gmm-clustering-hard},
and~\ref{alg:gmm-estimator-hard} for this version of DP.

\begin{thm}\label{thm:gmm-hard-zcdp}
    For all $\alpha,\beta,\rho>0$, there exists an
    $O(\rho)$-zCDP algorithm $\cM$ that takes $n$ private samples
    and $m$ public samples from $D \in \cG(d,k,s)$, and is
    private with respect to the private samples, such that if
    $D=\{(\mu_1,\Sigma_1,w_1),\dots,(\mu_k,\Sigma_k,w_k)\}$
    satisfies Assumption (\ref{eq:gaussian-not-flat}), for each
    $i \in [k]$, $\|\mu_i\| \leq R$ and $\id \preceq \Sigma_i \preceq K\id$,
    and
    \begin{align*}
        s &= \Omega(\sqrt{k\ln((n+m)k/\beta)})\\
        n &= O\left(\frac{d^2\ln\left(\frac{k}{\beta}\right)}{\mw\alpha^2} +
            \frac{d^2\cdot\polylog\left(\frac{dk}{\alpha\beta\rho}\right) +
            d\log\left(\frac{dk\log(R)}{\alpha\beta\rho}\right)}
            {\mw\alpha\sqrt{\rho}} \right.\\
            &\hspace{2.58cm}\left.+ \frac{d^{1.5}\sqrt{k\log(K)}\cdot
            \polylog\left(\frac{dk\log(K)}{\rho\beta}\right) +
            \sqrt{dk\log\left(\frac{Rdk}{\beta}\right)} +
            d^{1.5}k^{2.5}}{\sqrt{\rho}}\right)\\
        m &= O\left(\frac{\log(k/\beta)}{\mw}\right),
    \end{align*}
    then $\cM$ $(\alpha,\beta)$-learns $D$.
\end{thm}
\begin{proof}
    Again, we focus on the privacy guarantees because the accuracy
    would follow from Theorem~\ref{thm:gmm-hard} after setting the
    parameters correctly. Note that the privacy parameter in
    this case in the set $\PrivParams$ is $\rho$. As per the
    advanced composition guarantees of DP (Lemma~\ref{lem:composition}),
    we set the privacy parameter in $\PrivParams'$ in
    Algorithm~\ref{alg:gmm-clustering-hard} to be $\rho' = \tfrac{\rho}{2k}$.
    
    Now, we show that the intermediate steps in the call to
    Algorithm~\ref{alg:gmm-clustering-hard} yield $O(\rho)$-zCDP.
    As before, it is enough to prove the
    privacy guarantees of those intermediate steps,
    since the output of the algorithm itself wouldn't
    be released, but would be used by Algorithm~\ref{alg:gmm-estimator-hard},
    instead.
    
    In Algorithm~\ref{alg:gmm-clustering-hard}, we start working on
    the private dataset $X$ inside the loop, specifically, when we
    use private PCA. In each iteration $i$, by the same argument as
    in the proof of Theorem~\ref{thm:gmm-hard-approx}, the sensitivity
    of $Y^TY$ in the call to Algorithm~\ref{alg:private-PCA} is $2R_i^2$.
    Therefore, adding Gaussian noise calibrated to $\rho'$ to $Y^TY$,
    that is, setting $f_{PCA}(\rho') = \tfrac{1}{\sqrt{2\rho'}}$,
    implying that $\sigma_P = \tfrac{2R_i^2}{\sqrt{2\rho'}}$
    is enough to ensure $\rho'$-zCDP in this step
    (Lemmata~\ref{lem:gaussiandp} and~\ref{lem:post-processing}).
    The next step in Algorithm~\ref{alg:gmm-clustering-hard} that
    works with the private data is at the call to
    Algorithm~\ref{alg:low-dim-partitioner}. In this algorithm, only
    one call to $\PCount$ is performed, and this is the only time
    it works with the private dataset. Therefore, by the guarantees
    of Lemma~\ref{lem:pcount}, we know that this step is $\rho'$-zCDP.
    Therefore, each iteration is $2\rho'$-zCDP (by
    Lemma~\ref{lem:composition}).
    Applying composition over all $2k$ iterations, we have that
    all operations together in the entire run of the loop in the
    algorithm are $2\rho$-zCDP.
    
    In Theorem~\ref{thm:gmm-clustering-hard}, setting
    $f_{PCA}'(\PrivParams) = f_{PCA}(\rho) \in
    O\left(\tfrac{1}{\sqrt{\rho}}\right)$ give us the right
    sample complexity for accuracy for all calls to
    Algorithm~\ref{alg:private-PCA} (by Lemma~\ref{lem:composition}),
    since we already multiply the $\sqrt{k}$ factor in the numerator.
    Similarly, setting
    $f_{\PCount}'(\PrivParams) = f_{\PCount}(\sqrt{2\rho}) = \sqrt{2\rho}$
    in Theorem~\ref{thm:gmm-clustering-hard} gives us the right
    sample complexity for all calls to Algorithm~\ref{alg:low-dim-partitioner}.
    
    Now, each call to the zCDP Gaussian learner
    (Lemma~\ref{lem:gaussian-estimator}) is on a disjoint part
    of the private dataset $X$. Therefore, all the $k$ calls
    together are $\rho$-zCDP because changing one point
    in $X$ can change one point in only one of the clusters.
    By the same reasoning, all calls to $\PCount$ together are
    $\rho$-zCDP.
    
    For the accuracy guarantees, the first two terms in the
    sample complexity ensures that enough points go to each
    call to the Gaussian estimator from Lemma~\ref{lem:gaussian-estimator}
    and to each call to $\PCount$ because of Lemma~\ref{lem:gmm-frequqncy}.
    Therefore, our GMM estimator $(\alpha,\beta)$-learns the
    mixture.
\end{proof}

\subsection{\texorpdfstring{$\wt{O}(d/\mw)$}{} Public Samples}

In this subsection, we assume that we have $\wt{O}(d/\mw)$ public
data samples available. In this case, the number of public samples
is enough to do PCA accurately. So, the approach would be much
simpler than before, since we don't need to find a supercluster
anymore. Also, because we have enough public samples, we could simply use
the public dataset itself to partition the data when in low dimensions.
Therefore, the entire clustering operations to partition the private
dataset could be done using the public
dataset itself, without having to touch the private dataset at all.
As in the previous subsection, we will first give a general
private algorithm for estimating the GMM, but
will instantiate it separately for approximate DP and zCDP.

We start with the low-dimensional partitioner first, which
just uses the public dataset, and the public query, as defined
in \ref{eq:low-dim-query-public}.

\begin{algorithm}[!ht]
\caption{Partitioning in Low Dimensions
    $\LowDimPartitionerPublic_{m,\mw}(\wt{Y}, r_{\max}, r_{\min})$}\label{alg:low-dim-partitioner-easy}
\KwIn{Public Samples $\wt{Y}_1,\dots,\wt{Y}_{m'} \in \R^d$.
    Parameters $m \geq m', \mw > 0$.}
\KwOut{A tuple of centre $c \in \R^d$, radius $R \in \R$, or $\bot$.}
\vspace{5pt}

\For{$i \gets 0,\dots,\log(r_{\max}/r_{\min})$}{
    $r_i \gets \tfrac{r_{\max}}{2^i}$\\
    \For{$j \gets 1,\dots,m'$}{
        $c_j \gets \wt{Y}_j$\\
        \If{$\QPub(\wt{Y},c_j,r_i,\tfrac{m\mw}{2}) = \True$}{
            \Return $(c = c_j, R = 2r_i)$.
        }
    }
}
\vspace{5pt}

\Return $\bot$.
\vspace{5pt}
\end{algorithm}

\begin{thm}\label{thm:low-dim-partitioner-easy}
    Let $Y$ be a clean subset of a set of public samples
    from $D \in \cG(d,k)$, $\Pi$ be a projection
    matrix to $\ell$ dimensions, $\wt{Y} = Y\Pi$
    be the input to
    Algorithm~\ref{alg:low-dim-partitioner-easy}, $r_{\max}, r_{\min} > 0$,
    such that for some $j \in [m']$,
    $\wt{Y} \cap \ball{\wt{Y}_i}{r_{\max}} = \wt{Y}$
    and $\ball{\wt{Y}_i}{r_{\max}}$ is pure with
    respect to the components of $D$ projected on to the subspace
    of $\Pi$ having points in $\wt{Y}$.
    Suppose $\sigma_{\max}^2$ is the largest directional
    variance among the Gaussians that have points in $Y$.
    Let that Gaussian be $\cN(\mu,\Sigma)$. Suppose,
    $$m \geq O\left(\frac{d\ln(k/\beta)}{\mw}\right).$$
    Then we have the following with probability at least
    $1-\beta$.
    \begin{enumerate}
    \item Suppose there are at least two Gaussians that have
        points in $Y$.
        For any other Gaussian $\cN(\mu',\Sigma')$
        that has points in $Y$, suppose
        for $N > m$, $\|\mu\Pi - \mu'\Pi\| \geq
        \Omega(\sigma_{\max}\sqrt{\ell\ln(N\ell/\beta)})$.
        Let $r_{\min} \leq \sigma_{\max}\sqrt{2\ell\ln(2N\ell/\beta)}$.
        Then the ball returned by Algorithm~\ref{alg:low-dim-partitioner-easy}
        is pure with respect to the components of $D$ projected by $\Pi$,
        which have points in $\wt{Y}$,
        such that $\R^d \setminus \ball{c}{R}$ only contains
        components from $D$ projected by $\Pi$ that have points in
        $\wt{Y}$, and so
        does $\ball{c}{R}$.
    \item Suppose there is only one Gaussian that has points in
        $Y$. Then the algorithm returns $\bot$.
    \end{enumerate}
\end{thm}
\begin{proof}
    The proof for the first part is the same as that for the first
    part of Theorem~\ref{thm:low-dim-partitioner}. So, we don't discuss
    that any further.
    
    For the next part, note that we could relax the query $\QPub$
    by modifying the second constraint to say
    $\abs{X \cap (\ball{c}{11r} \setminus \ball{c}{r})} < \tfrac{t}{320}$.
    If the answer to the original query is true, then the answer
    to the modified query would be true, as well. In other words,
    if the answer to the modified query is false, then the answer
    to the original query is false, too. Therefore, it is enough
    to show that with probability at least $1-\beta$, there cannot
    be an $11$-terrific ball in $\wt{Y}$. As argued in the proof of
    Theorem~\ref{thm:low-dim-partitioner}, it is enough to show that
    there exists no $3$-terrific ball in $\wt{Y}$. We already proved
    this in the proof of Theorem~\ref{thm:low-dim-partitioner}.
    Hence, we have the claim.
\end{proof}

Next, we provide a clustering algorithm to partition the private
dataset that just operates on the public dataset
(Algorithm~\ref{alg:gmm-clustering-easy}). The algorithm is
the same as Algorithm~\ref{alg:gmm-clustering-hard}, but uses
the public data solely to partition the private data. It uses
tools like PCA and Algorithm~\ref{alg:low-dim-partitioner-easy}
on public data.

\begin{algorithm}[!ht]
\caption{GMM Easy Clustering
    $\DPEC_{\beta, k, \mw}(\wt{X}, X)$}\label{alg:gmm-clustering-easy}
\KwIn{Private samples $X=(X_1,\dots,X_n) \in \R^{n \times d}$.
    Public samples $\wt{X} = (\wt{X}_1,\dots,\wt{X}_{m})
        \in \R^{m \times d}$.
    Parameters $\beta, k, \mw > 0$.}
\KwOut{Set $C \subset \cP(X)$.}
\vspace{5pt}

Let $Y \gets \wt{X}$ and $Z \gets X$.\\
Let $C \gets \emptyset$.\\
Let $Q_{\Priv}$ and $Q_{\Pub}$ be queues of sets of points.\\
Add set $Z$ to $Q_{\Priv}$, and set $Y$ to $Q_{\Pub}$.\\
Set $count \gets 0$ and $i \gets 1$.
\vspace{5pt}

\While{$count < k$ and $i \leq 2k$}{
    Pop $Q_{\Priv}$ to get $Z$, and $Q_{\Pub}$ to get $Y$.\\
    \vspace{5pt}
    \tcp{PCA: project points of both datasets onto
        returned subspace.}
    Let $\Pi_i$ be the top-$k$ subspace of $Y^TY$.\\
    $Y' \gets Y\Pi_i$ and $Z' \gets Z\Pi_i$.\\
    \vspace{5pt}
    \tcp{If there is only one Gaussian, add it to $C$, otherwise further partition
        the datasets.}
    Let $R_i \gets 4\max\limits_{y_1,y_2 \in Y}\{\|y_1-y_2\|\}$
        and $r_i \gets \tfrac{\sqrt{2k\ln(2(n+m)k/\beta)}}{4\sqrt{d}}
        \min\limits_{y_1,y_2 \in Y}\{\|y_1-y_2\|\}$.\\
    $B \gets \LowDimPartitionerPublic_{m,\mw}
        (Y', R_i, r_i)$.\\
    \If{$B = \bot$}{
        $C \gets C \cup \{Z^i\}$.\\
        $count \gets count + 1$.
    }
    \Else{
        $B$ is an ordered pair $(c_i',r_i')$.\\
        $S' \gets \ball{c'}{r_i'} \cap Y'$ and
            $T' \gets \ball{c'}{r_i'} \cap Z'$.\\
        Let $S$ be points in $Y^i$ corresponding to $S'$,
            and $T$ be points in $Z^i$ corresponding to $T'$.\\
        Add $S$ to $Q_{\Pub}$ and $T$ to $Q_{\Priv}$.\\
        Add $Y^i \setminus S$ to $Q_{\Pub}$ and $Z^i \setminus T$ to $Q_{\Priv}$.
    }
    \vspace{5pt}
    $i \gets i + 1$
}
\vspace{5pt}

\Return $C$.
\vspace{5pt}
\end{algorithm}

\begin{thm}\label{thm:gmm-clustering-easy}
    There exists an algorithm (Algorithm~\ref{alg:gmm-clustering-easy})
    that takes $n$ private samples $X$, and $m$ public samples from
    $D \in \cG(d,k,s)$ satisfying Assumption (\ref{eq:gaussian-not-flat}),
    and outputs a partition $C$ of $X$, such that
    given Conditions \ref{cond:gmm-frequency} to \ref{cond:intra-gaussian-points-low}
    hold, and
    $$s \geq \Omega(\sqrt{k\ln((n+m)k/\beta)}),$$
    and
    $$m \geq O\left(\frac{d\ln(k/\beta)}{\mw}\right),$$
    then with probability at least $1-O(\beta)$, $\abs{C} = k$ and
    each $S \in C$ is clean and non-empty.
\end{thm}
\begin{proof}
    As in the proof of Theorem~\ref{thm:gmm-clustering-hard},
    it is enough to show that $\forall i \geq 0$, at the end of the
    $i$-th iteration, (1) a new Gaussian component is isolated in $C$
    and $count$ equals the number of isolated components in $C$,
    (2) or the private dataset $X$ is further partitioned into non-empty,
    clean subsets, and the public dataset $\wt{X}$ is further partitioned
    into clean subsets, such that $\abs{Q_{\Pub}} = \abs{Q_{\Priv}}$ and
    for each $j \in [\abs{Q_{\Pub}}]$, $Q_{\Pub}[j]$ and $Q_{\Priv}[j]$
    have points from the same Gaussian components. We can prove this via
    induction on $i$.
    
    \noindent \textbf{Base Case:} The proof is exactly the same as in
    the proof of Theorem~\ref{thm:gmm-clustering-hard}.
    
    \noindent \textbf{Inductive Step:} We assume for all iterations,
    up to and including some $i \geq 0$, the claim holds. Then we show
    that it holds for iteration $i+1$, as well. By the inductive hypothesis,
    we know that $Z$ and $Y$ are clean subsets, and contain points from
    the same components.
    
    We first reason about the correctness of the PCA step.
    Using Lemma~\ref{lem:gaussian-sum-conc} and our bound on $m$,
    we know that for each component $(\mu',\Sigma',w')$, the empirical
    mean of all the points from that Gaussian in $\wt{X}$ would be at most
    $\tfrac{\sqrt{\|\Sigma'\|_2}}{\sqrt{w'}}$ away from $\mu'$. By the
    same reasoning,
    the empirical mean of the same projected points would be at most
    $\tfrac{\sqrt{\|\Sigma'\|_2}}{\sqrt{w'}}$ away from $\mu'\Pi$ because
    the projection of a Gaussian is still a Gaussian.
    Lemmata~\ref{lem:PCA_AM_style} (by setting $B=0$)
    and~\ref{lem:data-spectral} together guarantee that the distance
    between the empirical mean of those points and that of the projected
    points of that component would be at most
    $O\left(\tfrac{\sigma_{\max}}{\sqrt{w'}}\right)$. The triangle
    inequality implies that $\|\mu' - \mu'\Pi\| \leq
    O\left(\tfrac{\sigma_{\max}}{\sqrt{w'}}\right)$. This
    happens with probability at least $1-\tfrac{\beta}{2k}$.
    
    Now, if there is just one component that has points in $Y$
    (hence, in $Z$), then on feeding $Y'$
    to Algorithm~\ref{alg:low-dim-partitioner}, we will get $\bot$
    (by the guarantees from Theorem~\ref{thm:low-dim-partitioner})
    with probability at least $1-\tfrac{\beta}{2k}$. Because $Z$
    is a clean subset, $C$ now contains one more clean subset of $X$,
    and we have isolated a new component, and updated the value
    of $count$ to reflect the change.
    
    Suppose there are at least two components that have points in
    $Y$ (hence, in $Z$). Let the largest directional variance
    among the Gaussians that have points in $Y$ be
    $\sigma_{\max}^2$, its mean be $\mu$, and its mixing weight be $w$.
    From the guarantees of the PCA step, we know
    that $\|\mu - \mu\Pi_i\| \leq O\left(\tfrac{\sigma_{\max}}{\sqrt{w}}\right)$.
    Because of the separation condition, applying the triangle inequality,
    we know that the distance between $\mu\Pi_i$ and the projected mean
    of any other Gaussian having points in $Z\Pi$ is at least
    $\Omega(\sigma_{\max}\sqrt{k\ln((n+m)k)/\beta})$.
    Also, we know by the construction of the algorithm that for any
    $y \in Y$, $Y \subseteq \ball{y}{R_i}$, and from
    Lemmata~\ref{lem:intra-gaussian-mean}, \ref{lem:intra-gaussian-points}
    and~\ref{lem:inter-gaussian}
    that $\ball{y}{R_i}$ would contain all the components of $D$ that
    have points in $Y$. Also, by Lemmata~\ref{lem:intra-gaussian-points},
    and~\ref{lem:inter-gaussian}, $r_i \leq \sigma_{\max}\sqrt{2k\ln(2(n+m)k/\beta)}$.
    Then from the
    guarantees of Theorem~\ref{thm:low-dim-partitioner-easy}, we have
    that a ball is returned, and it is pure with respect to the
    components of the mixture projected by $\Pi_i$ that have points
    in $Y'$. Therefore, subsets $S'$ and $T'$ are clean, and
    so are $S$ and $T$. This implies that $Y \setminus S$ and
    $Z \setminus T$ are clean, too. By adding them in order to
    $Q_{\Pub}$ and $Q_{\Priv}$ respectively, we partition the
    datasets $\wt{X}$ and $X$ further into clean subsets, and
    preserve the ordering of the components in the respective queues.
    This proves the claim.
    
    As argued before in the proof Theorem~\ref{thm:gmm-clustering-hard},
    allowing up to $2k$ iterations is enough. Therefore, we get a clean
    partition of $X$.
\end{proof}

Finally, we provide the DP algorithm for learning GMM's
in this regime of public data. The algorithm is the same
as Algorithm~\ref{alg:gmm-estimator-hard}, except that it
calls Algorithm~\ref{alg:gmm-clustering-easy}, instead of
Algorithm~\ref{alg:gmm-clustering-hard}.

\begin{algorithm}[!ht]
\caption{DP GMM Easy Estimator
    $\DPEE_{\PrivParams, \alpha, \beta, k, \mw}(\wt{X}, X)$}\label{alg:gmm-estimator-easy}
\KwIn{Private samples $X=(X_1,\dots,X_n) \in \R^{n \times d}$.
    Public samples $\wt{X} = (\wt{X}_1,\dots,\wt{X}_{m})
        \in \R^{m \times d}$.
    Parameters $\PrivParams \subset \R, \beta, k, \mw > 0$.}
\KwOut{Set $\wh{D} = \{(\wh{\mu}_1,\wh{\Sigma}_1,\wh{w}_1),
    \dots,(\wh{\mu}_k,\wh{\Sigma}_k,\wh{w}_k)\}$.}
\vspace{5pt}

\tcp{Clustering.}
Set $C \gets \DPEC_{\beta, k, \mw}(X, \wt{X})$.
\vspace{5pt}

\tcp{Parameter estimation.}
$\wh{D} \gets \emptyset$.\\
\For{$i \gets 1, \dots, k$}{
    Set $(\wh{\mu}_i,\wh{\Sigma}_i) \gets
        \DPGE_{\PrivParams, \frac{\beta}{k}}(C_i)$.\\
    Set $\wh{w}_i \gets
        \max\left\{\tfrac{1}{n}\cdot\PCount_{\PrivParams}(\abs{C_i}),
        \tfrac{\alpha}{2k}\right\}$.\\
    $\wh{D} \gets \wh{D} \cup \{(\wh{\mu}_i,\wh{\Sigma}_i,\wh{w}_i)\}$.
}
\vspace{5pt}

\Return $\wh{D}$.
\vspace{5pt}
\end{algorithm}

\begin{thm}\label{thm:gmm-easy}
    For all $\alpha,\beta>0$ and sets of privacy parameters
    $\PrivParams$, there exists an
    algorithm (Algorithm~\ref{alg:gmm-estimator-easy}) that takes
    $n$ private samples
    and $m$ public samples from $D \in \cG(d,k,s)$, such that if
    $D=\{(\mu_1,\Sigma_1,w_1),\dots,(\mu_k,\Sigma_k,w_k)\}$
    satisfies Assumption (\ref{eq:gaussian-not-flat}), and
    \begin{align*}
        s &= \Omega(\sqrt{k\ln((n+m)k/\beta)})\\
        n &= O\left(\frac{n_{GE}\ln(k/\beta)}{\mw}\right)\\
        m &= O\left(\frac{d\ln(k/\beta)}{\mw}\right),
    \end{align*}
    where $n_{GE}$ is the sample complexity of privately
    learning a Gaussian using Lemma~\ref{lem:gaussian-estimator}
    and Theorem~\ref{thm:gaussian}
    according to the type of DP required, then it $(\alpha,\beta)$-learns $D$.
\end{thm}
\begin{proof}
    By Theorem~\ref{thm:gmm-clustering-easy}, we know that with
    probability at least $1-\tfrac{\beta}{3}$, $\abs{C} = k$,
    and for each $i \in [k]$, $C_i$ is a clean subset of $X$.
    
    Next, by the guarantees of Lemma~\ref{lem:gaussian-estimator}
    and Theorem~\ref{thm:gaussian}, for each $i \in [k]$, with
    probability at least $1-\tfrac{\beta}{3k}$,
    $\SD(\cN(\wh{\mu}_i,\wh{\Sigma}_i),\cN(\mu_i,\Sigma_i)) \leq \alpha$.
    
    Finally, by the same argument as in the proof of
    Theorem~\ref{thm:gmm-hard}, the mixing weights are estimated
    accurately, as well.
    
    Therefore, taking the union bound over all failure events, we
    have the required result.
\end{proof}

\subsubsection{\texorpdfstring{$(\eps,\delta)$}{}-DP Algorithm}

Here, we instantiate Algorithm~\ref{alg:gmm-estimator-easy}
for the case of approximate DP.

\begin{thm}\label{thm:gmm-easy-approx}
    For all $\alpha,\beta,\eps,\delta>0$, there exists an
    $(O(\eps),O(\delta))$-DP algorithm $\cM$ that takes $n$ private samples
    and $m$ public samples from $D \in \cG(d,k,s)$, and is
    private with respect to the private samples, such that if
    $D=\{(\mu_1,\Sigma_1,w_1),\dots,(\mu_k,\Sigma_k,w_k)\}$
    satisfies Assumption (\ref{eq:gaussian-not-flat}), and
    \begin{align*}
        s &= \Omega(\sqrt{k\ln((n+m)k/\beta)})\\
        n &= O\left(\frac{d^2\ln\left(\frac{k}{\beta}\right)}{\mw\alpha^2} +
            \frac{(d^2\log(k/\delta) +
            d\log^{1.5}(k/\delta))\cdot\polylog\left(d,\frac{k}{\beta},
            \frac{1}{\alpha},\frac{\sqrt{k}}{\eps},
            \ln\left(\frac{k}{\delta}\right)\right)}{\mw\alpha\eps}\right)\\
        m &= O\left(\frac{d\log(k/\beta)}{\mw}\right),
    \end{align*}
    then $\cM$ $(\alpha,\beta)$-learns $D$.
\end{thm}
\begin{proof}
    We again focus on the privacy guarantees here because the accuracy
    would follow from Theorem~\ref{thm:gmm-easy} after setting the
    parameters appropriately. Note that the privacy parameters in
    this case in the set $\PrivParams$ are $\eps,\delta$. We just
    have to show the privacy guarantees for the calls to the Gaussian
    estimator and $\PCount$ because all other steps involve computations
    on the public dataset.

    By the same argument as in the proof for Theorem~\ref{thm:gmm-hard-approx},
    all $k$ calls together to the approximate DP learner from
    Lemma~\ref{lem:gaussian-estimator} are $(\eps,\delta)$-DP,
    and all calls together to $\PCount$ are $(\eps,\delta)$-DP.
    Therefore, we get $(2\eps,2\delta)$-DP guarantee in the end.
    
    For the accuracy goal, the sample complexity ensures that enough points go to each
    call to the Gaussian estimator from Lemma~\ref{lem:gaussian-estimator}
    and to each call to $\PCount$ because of Lemma~\ref{lem:gmm-frequqncy}.
    Therefore, our GMM estimator $(\alpha,\beta)$-learns the
    mixture.
\end{proof}

\subsubsection{\texorpdfstring{$\rho$}{}-zCDP Algorithm}

Finally, we instantiate Algorithm~\ref{alg:gmm-estimator-easy}
for the case of zCDP. Note that we don't need any bounds on the
range parameters in this case because we are using the Gaussian
learner from Section~\ref{sec:gaussians}.

\begin{thm}\label{thm:gmm-easy-zcdp}
    For all $\alpha,\beta,\rho>0$, there exists an
    $O(\rho)$-zCDP algorithm $\cM$ that takes $n$ private samples
    and $m$ public samples from $D \in \cG(d,k,s)$, and is
    private with respect to the private samples, such that if
    $D=\{(\mu_1,\Sigma_1,w_1),\dots,(\mu_k,\Sigma_k,w_k)\}$
    satisfies Assumption (\ref{eq:gaussian-not-flat}),
    and
    \begin{align*}
        s &= \Omega(\sqrt{k\ln((n+m)k/\beta)})\\
        n &= O\left(\frac{d^2\ln\left(\frac{k}{\beta}\right)}{\mw\alpha^2} +
            \frac{d^2\cdot\polylog\left(\frac{dk}{\alpha\beta\rho}\right)}
            {\mw\alpha\sqrt{\rho}}\right)\\
        m &= O\left(\frac{d\log(k/\beta)}{\mw}\right),
    \end{align*}
    then $\cM$ $(\alpha,\beta)$-learns $D$.
\end{thm}
\begin{proof}
    We mainly prove the privacy guarantees because the accuracy
    would follow from Theorem~\ref{thm:gmm-easy} after setting the
    parameters correctly. Note that the privacy parameter in
    this case in the set $\PrivParams$ is $\rho$.
    
    Using the same argument as in the proof of
    Theorem~\ref{thm:gmm-hard-approx}, we have $\rho$-zCDP
    for all $k$ calls to the learner in Section~\ref{sec:gaussians}
    (by Theorem~\ref{thm:gaussian}), and $\rho$-zCDP for all calls
    together to $\PCount$. Therefore, we have $2\rho$-zCDP in the end.
    
    By the same argument as before, the sample complexity
    is enough to send enough points to each call to the learner
    in Section~\ref{sec:gaussians} and to $\PCount$. Therefore,
    our GMM estimator $(\alpha,\beta)$-learns the mixture.
\end{proof}

\printbibliography

\appendix

\section{A Proof-of-Concept Numerical Result}\label{sec:exp}

For Gaussian mean estimation, our approach is relatively simple to implement on top of an existing private algorithm. We offer some proof-of-concept simulations that demonstrate the effectiveness of public data in private statistical estimation.\footnote{Our code is available at \url{https://github.com/alexbie98/1pub-priv-mean-est}. Experiments run within two minutes on a 2020 M1 MacBook Air.}
In Figure \ref{fig:1pub-priv-mean-est}, we show plots that evaluate 1-public-sample private mean estimation (the algorithm described in Section \ref{subsec:gaussians-mean}).

\begin{figure}[!ht]
\centering
\begin{subfigure}{0.3\textwidth}
    \includegraphics[trim={0.5cm 0.6cm 0.5cm 0.6cm}, width=1.0\textwidth]{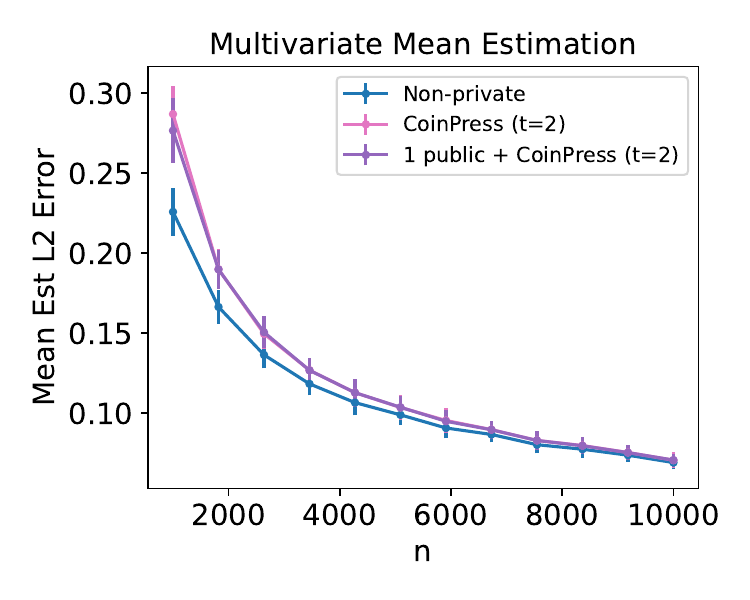}
    \caption{$k=10$}
\end{subfigure}
\begin{subfigure}{0.3\textwidth}
    \includegraphics[trim={0.5cm 0.6cm 0.5cm 0.6cm}, width=1.0\textwidth]{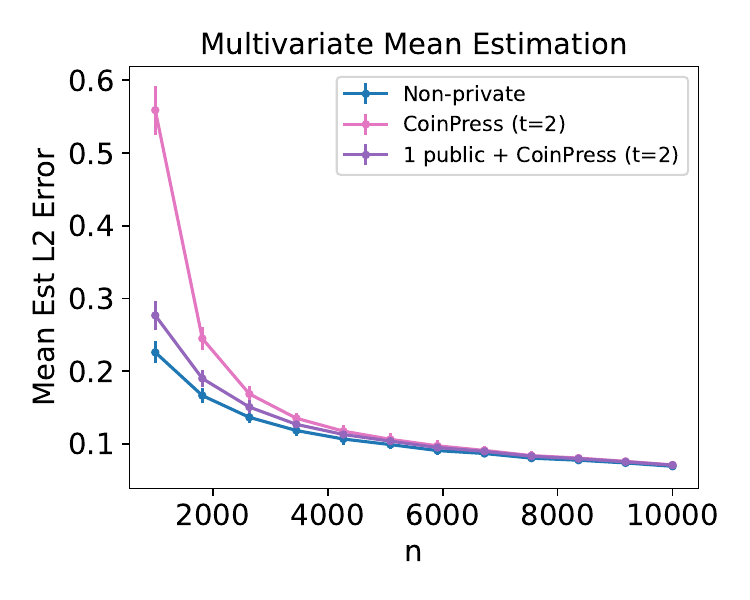}
    \caption{$k=100$}
\end{subfigure}
\begin{subfigure}{0.3\textwidth}
    \includegraphics[trim={0.5cm 0.6cm 0.5cm 0.6cm}, width=1.0\textwidth]{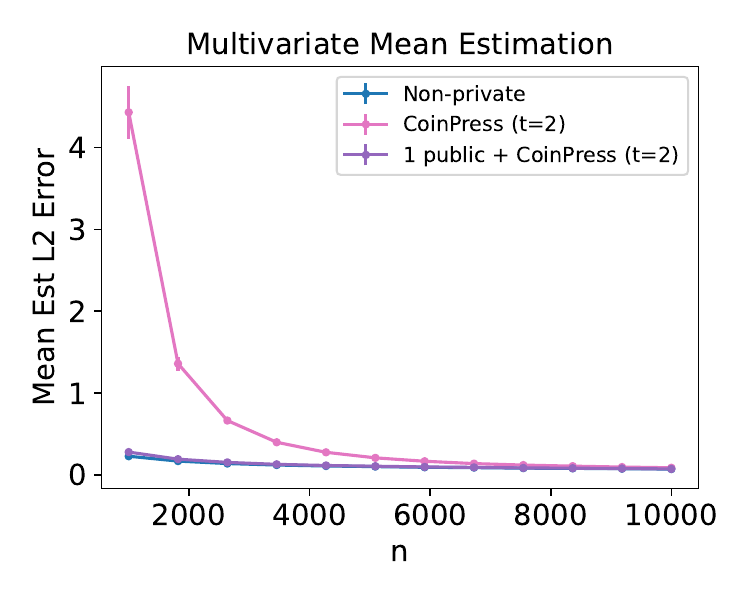}
    \caption{$k=1000$}
\end{subfigure}

\begin{subfigure}{0.35\textwidth}
    \includegraphics[trim={0.0cm 0.6cm 0.0cm 0.0cm}, width=1.0\textwidth]{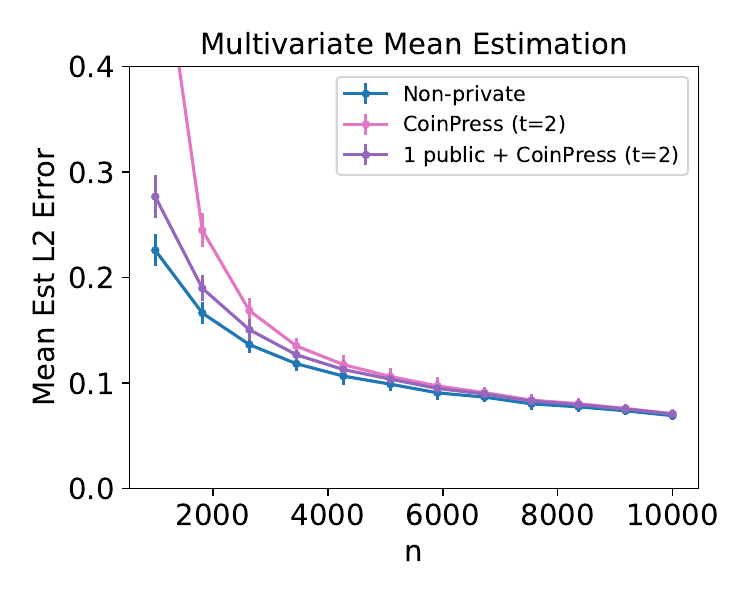}
    \caption{$k=100$ (zoomed in)}
\end{subfigure}
\begin{subfigure}{0.35\textwidth}
    \includegraphics[trim={0.0cm 0.6cm 0.0cm 0.0cm}, width=1.0\textwidth]{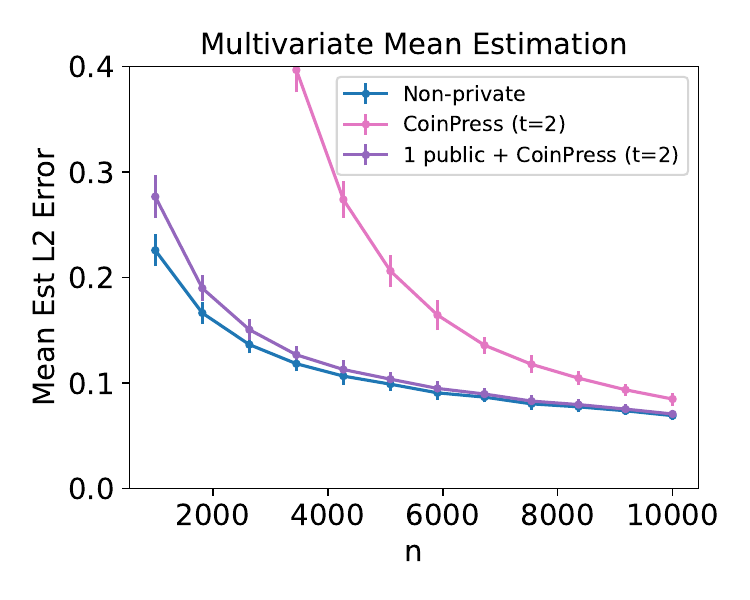}
    \caption{$k=1000$ (zoomed in)}
\end{subfigure}

\caption{Comparing the error of CoinPress \cite{BiswasDKU20} under its best setting, against CoinPress with 1 public sample for mean estimation of $\cN(k\cdot[1,...,1]^T, I_d)$ for $d=50$, targeting zCDP at $\rho=0.5$. Larger $k$ corresponds to weaker a priori bounds on the mean for CoinPress. For large $k$, a single public sample significantly improves results.}
\label{fig:1pub-priv-mean-est}
\end{figure}

We examine the effect of $1$ public sample on the performance of CoinPress \cite{BiswasDKU20} with its best parameter setting ($t=2$), in a case where the initial a priori bounds on the mean are weak. Concretely, we draw $n$ samples from a $d=50$ dimensional Gaussian $\cN(\mu_k, I_d)$, where $\mu_k = k \cdot [1,...,1]^T$ and correspondingly set our a priori bound $R = k\sqrt d$ for CoinPress. We show results for $k \in [10,100,1000]$, representing varying levels of strength in our a prior bounds on the mean.

We follow the evaluation protocol from \cite{BiswasDKU20}: we target zCDP with $\rho=0.5$, and at each sample size $n$ we run the estimator 100 times and report the $10\%$ trimmed mean of error from the ground truth (we additionally report the $10\%$ trimmed standard deviation as error bars). We also follow their practice of treating target failure probabilities $\beta_i$ of various steps of the algorithm as hyperparameters that can be tuned for the best empirical results. The new step we introduce (using the public sample to set $R$ based on $d, \beta$) uses $\beta = 0.01$, which is the same value used for all the $\beta_i$'s in CoinPress.

The numerical result demonstrates the promise of utilizing public data for private data analysis, and confirms the takeaway that very little public data can help greatly when a priori knowledge of the private data is weak. As is visible from the plots, the error of our public-private algorithm tracks the non-private algorithm closer when a priori bounds on the mean are weak $(k=1000$ case).

Note that these results are only meant to be a proof-of-concept simulation to demonstrate the promise of public data -- thorough tuning and evaluation of these algorithms (which is necessary to bring these algorithms to practice) is an important direction for future work.

\end{document}